\documentclass[twoside]{article}

\usepackage[accepted]{aistats2026}

\usepackage[linesnumbered,ruled,vlined]{algorithm2e}
\SetKw{KwBreak}{break}
\SetKwComment{Comment}{\# }{}

\usepackage[utf8]{inputenc} %
\usepackage[T1]{fontenc}    %
\usepackage{hyperref}       %
\usepackage{url}            %
\usepackage{booktabs}       %
\usepackage[round,authoryear]{natbib}  %
\usepackage[strings]{underscore} %

\usepackage{amsfonts}       %
\usepackage{stmaryrd}       %
\usepackage{nicefrac}       %
\usepackage{microtype}      %
\usepackage[table,dvipsnames]{xcolor}  %

\usepackage{amsmath}
\usepackage{amssymb}
\usepackage{textcomp} %
\usepackage{amsthm}
\usepackage{graphicx}
\usepackage{subcaption} %
\usepackage{caption}
\usepackage[capitalize,noabbrev]{cleveref} %
\usepackage{enumitem}
\usepackage{colortbl}
\usepackage{adjustbox}
\usepackage{listings} %
\usepackage{fontawesome} %
\usepackage{minitoc} %
\usepackage[toc,page,header]{appendix} %
\usepackage{transparent}    %
\usepackage{pifont} %
\usepackage{wrapfig}
\usepackage{multirow}
\usepackage{mathtools} %

\DeclareMathOperator*{\argmax}{arg\,max}

\makeatletter
\newcommand{\Scale}[2][1]{\scalebox{#1}{$\m@th#2$}}
\makeatother

\definecolor{tealblue}{HTML}{156082}

\usepackage{mdframed}
\newcommand{\ExtractorLLM}{f_{\theta}}
\newcommand{\PlannerLLM}{h_{\omega}} 
 
\newcommand{\FactMemory}{\mathcal{M}}
\newcommand{\FactSet}{\mathcal{F}}

\makeatletter 
\def\blfootnote{\xdef\@thefnmark{}\@footnotetext}
\makeatother

\definecolor{termcolor1}{RGB}{200, 0, 0} %
\definecolor{termcolor2}{RGB}{0, 150, 0} %
\definecolor{termcolor3}{RGB}{0, 0, 200} %

\newcommand{\R}{\mathbb{R}}
\newcommand{\E}{\mathbb{E}}
\newcommand{\Prob}{P} %
\newcommand{\powerset}{\mathcal{P}} %

\newtheorem{theorem}{Theorem}[section] %

\newtheorem{definition}[theorem]{Definition}

\newmdenv[
  linecolor=black,                %
  linewidth=0.5pt,                %
  roundcorner=0pt,                %
  innertopmargin=\baselineskip,     %
  innerbottommargin=\baselineskip,  %
  innerleftmargin=1em,            %
  innerrightmargin=1em,           %
  userdefinedwidth=0.9\linewidth, %
  font=\ttfamily\small,           %
  frametitlefont=\ttfamily\small\bfseries, %
  frametitlealignment=\flushleft, %
  middlelinewidth=0pt,            %
  splittopskip=\topskip,            %
  needspace=3\baselineskip,         %
  skipabove=\medskipamount,         %
  skipbelow=\medskipamount          %
]{promptboxstyle}

\newenvironment{promptbox}[1]
  {%
    \begin{center} %
    \begin{promptboxstyle}[frametitle={#1 Prompt:}]%
    \noindent\ignorespaces%
  }
  {%
    \end{promptboxstyle}%
    \end{center}%
  }

\newmdenv[
  linecolor=black,
  linewidth=0.5pt,
  backgroundcolor=gray!3,
  roundcorner=2pt,
  frametitlefont=\bfseries,
  frametitlealignment=\center,
  skipabove=\medskipamount,
  skipbelow=\medskipamount
]{theorybox}

\hypersetup{%
    colorlinks = true,
    citecolor  = blue,
    linkcolor  = blue,
    urlcolor   = blue
}

\begin{document}
\doparttoc
\faketableofcontents

\runningtitle{Fact-Augmented Lookahead Planning for LLM Agents}

\runningauthor{Samuel Holt, Max Ruiz Luyten, Thomas Pouplin, Mihaela van der Schaar}

\twocolumn[

\aistatstitle{Fact-Augmented Lookahead Planning for LLM Agents}

\aistatsauthor{
Samuel Holt\textsuperscript{*} \\ University of Cambridge 
\And
Max Ruiz Luyten\textsuperscript{*} \\ University of Cambridge
}
\vspace{0.1in}

\aistatsauthor{
Thomas Pouplin \\ University of Cambridge 
\And
Mihaela van der Schaar \\ University of Cambridge
}
\vspace{0.20in}
]

\begin{abstract}
Large Language Models (LLMs) are increasingly capable, but LLM agents still struggle to plan effectively in interactive, partially observable, long-horizon environments when search is unguided or recent history is insufficient. We introduce LWM-Planner, a fact-augmented lookahead planning framework that improves agent behavior purely through in-context learning. After each episode, the agent extracts task-critical atomic facts from its trajectories, validates candidates with a lightweight predictive-consistency filter (and optionally compresses them), and uses the resulting fact set to condition action proposal, single-step latent world-model simulation, and state-value estimation. Planning then proceeds via recursive, depth-limited lookahead over candidate trajectories conditioned on the accumulated facts and recent history, enabling online improvement without parameter updates. We provide abstraction-style motivation—treating facts as reducing state aliasing (proxy $\epsilon_{\mathrm{sim}}$) and fact-conditioned simulation as lowering one-step error (proxy $\delta_{\mathrm{model}}$)—without claiming formal guarantees. Empirically, on text FrozenLake variants, CrafterMini, and ALFWorld, the approach improves cumulative return over ReAct/Reflexion and search-only baselines, suggesting that additional test-time search is most useful when grounded by compact, experience-derived facts.
\end{abstract}

\begin{figure*}[!t]
\centering
\includegraphics[width=0.9\textwidth]{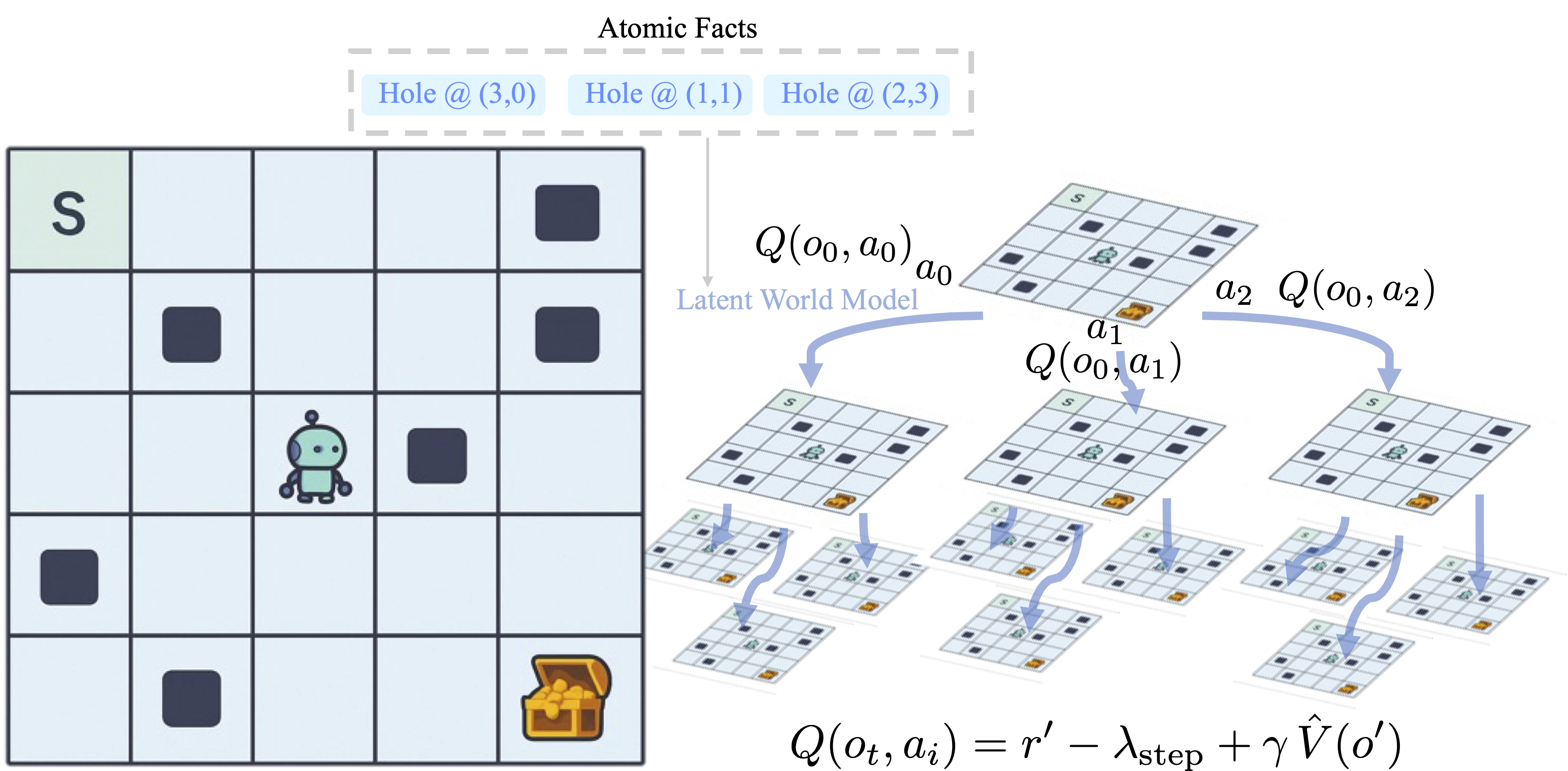}
\caption{%
  \textbf{LWM-Planner extracts compact facts from experience and uses them to ground lookahead planning at test time.} 
  Starting from the current environment observation $o_0$ (left panel), the agent conditions on previously extracted \texttt{Atomic Facts} (e.g., \texttt{Hole@(3,0)}) and performs a recursive lookahead search (right panel) to choose its next action. The search combines: (i) an LLM, acting as a \emph{Latent World Model}, to simulate action sequences ($a_i$) and predict subsequent latent states ($o'$) and immediate rewards ($r'$); (ii) an LLM-based state-value estimator that provides $\hat{V}(o')$ at the search frontier; and (iii) Q-value aggregation from the simulated outcomes and value estimates via
  \(
    Q(o_t, a_i) \;=\; r' \;-\; \lambda_{\mathrm{step}} \;+\; \gamma\,\hat{V}(o')\,,
  \)
  which guides action selection from $o_0$.%
}
\label{fig:fig1_overview}
\end{figure*}

\section{INTRODUCTION}
\label{sec:introduction}

Large Language Models (LLMs) have demonstrated remarkable potential in building autonomous agents for sequential decision-making in diverse settings, from text-based games \citep{yao2023react} to complex interactive environments \citep{zhou2024agentoccam}. A key insight underpinning their success is their vast pre-trained knowledge, which can be steered towards specific tasks. However, effectively harnessing this knowledge and enabling LLMs to learn from new experiences in-context remains a critical challenge for improving their accuracy and optimality in long-horizon tasks.

Many LLM agents rely on extensive few-shot examples \citep{shinn2023reflexion} or retrieve entire past trajectories \citep{kagaya2024rap} to inform their decisions. While effective to a degree, these approaches can lead to very long prompts or may not efficiently distill the most crucial pieces of information from past experiences. Model-based approaches \citep{hao2023reasoning, chae2024web} have emerged, but often involve learning separate, potentially shallow, predictive models or require environment interactions for each step of their lookahead, rather than leveraging the LLM's inherent simulation capabilities more deeply.

\blfootnote{*Equal contribution.}

The core idea of this paper is that LLMs possess a substantial amount of latent knowledge about world dynamics and task structures. To unlock better planning, we need to identify and provide the missing pieces of information—concise, critical insights derived from experience—that allow the LLM to more accurately simulate outcomes and evaluate states. We propose a novel LLM agent architecture that learns and utilizes ``atomic facts'' to augment its planning process. These facts are textual statements (e.g., ``object X is in receptacle\_Y'', ``action Z leads\_to\_failure\_condition'') extracted from the agent's interaction history at the end of each episode.

Our agent employs these atomic facts to inform a recursive, depth-limited lookahead search. The planning process involves three key LLM-driven components: An \textbf{action proposer} to suggest plausible next actions. A \textbf{latent world model} to simulate the next observation, reward, and termination status given an action. A \textbf{value estimator} to predict the long-term utility of states, especially at the leaves of the search tree.
All these components receive the current observation, recent interaction history, and the curated set of atomic facts as input, allowing the LLM to make more informed predictions and decisions. The agent learns online, purely in-context, as the set of facts evolves with experience, leading to improved policies without any LLM fine-tuning. This approach is inspired by Dyna-style architectures (model-learning + planning) \citep{sutton1990integrated}, where experience is used to refine a model (here, the fact-augmented LLM reasoning process) which is then used for planning.

\textbf{Contributions}: \textbf{\textcircled{\raisebox{-0.9pt}{1}}}\textbf{Algorithmic:} We present a simple fact-augmented planning mechanism for LLM agents: after each episode the agent distills a small set of atomic, task-specific facts and conditions \emph{action proposal}, \emph{single-step simulation}, and \emph{value estimation} within a depth-limited lookahead (Section~\ref{sec:method}). A lightweight \emph{predictive-consistency} filter and optional compression keep the fact set minimal and within the context window. 
\textbf{\textcircled{\raisebox{-0.9pt}{2}}} \textbf{Theoretical Framing:} We connect the design to fact-based state abstraction (Section~\ref{sec:theoretical_framework}): facts aim to reduce state aliasing (a proxy for $\epsilon_{\mathrm{sim}}$) while fact-conditioned simulation targets lower one-step prediction error (a proxy for $\delta_{\mathrm{model}}$), which together can reduce planning sub-optimality $\epsilon_{\mathrm{plan}}$. These links are provided as motivation rather than formal guarantees. 
\textbf{\textcircled{\raisebox{-0.9pt}{3}}} \textbf{Empirical:} Across text environments (TextFrozenLake, CrafterMini) and the ALFWorld evaluation suite, we compare against strong search and agent baselines (ReAct, Reflexion, Tree-of-Thought, RAP), reporting cumulative return. Targeted ablations indicate that grounding lookahead with learned facts is the dominant contributor to the gains, and we analyze computation–performance trade-offs.

Taken together, these results suggest that in partially observable, multi-step environments, compact facts learned from experience can make test-time lookahead materially more reliable.

Figure~\ref{fig:fig1_overview} gives a high-level overview of the resulting planning loop.

This work offers a step towards LLM agents that can more effectively learn from their interactions in-context, leading to more robust and accurate planning by systematically augmenting the LLM's reasoning with distilled, experience-grounded knowledge.
\section{THEORETICAL FRAMEWORK FOR FACT-BASED REINFORCEMENT LEARNING}
\label{sec:theoretical_framework}
We propose enabling agents to construct and reason over a \emph{fact-based world model}. Such a model relies on a compressed, symbolic representation of task-relevant information extracted as ``atomic facts'' from the environment's state. 
We do so only as design motivation: we define an ideal fact-based agent and derive bounds that motivate \cref{sec:method}; these are not guarantees for LWM-Planner, and we do not estimate or optimize $\epsilon_{\mathrm{sim}}$, $\delta_{\mathrm{model}}$, or $\epsilon_{\mathrm{plan}}$.

\subsection{Problem Formulation}
\label{sec:problem_formulation}

We model the agent's interaction with its environment as a \textbf{Markov Decision Process (MDP)}, specified by the tuple $\mathcal{G} = (\mathcal{S}, \mathcal{A}, \mathcal{T}, R, \gamma)$. Here, $\mathcal{S}$ is the set of environment states, which we assume are fully observable or derivable into a sufficient structured representation $s_t \in \mathcal{S}$ from raw observations $o_t$. $\mathcal{A}$ is a finite set of actions. $\mathcal{T}: \mathcal{S} \times \mathcal{A} \times \mathcal{S} \rightarrow [0,1]$ is the state transition probability function, $\mathcal{T}(s'|s,a) = \Prob(s_{t+1}=s'|s_t=s, a_t=a)$. $R: \mathcal{S} \times \mathcal{A} \times \mathcal{S} \rightarrow \R$ is the reward function, with expected reward $R(s,a) = \E_{s' \sim \mathcal{T}(\cdot|s,a)}[R(s,a,s')]$. Finally, $\gamma \in [0,1)$ is the discount factor. The agent's goal is to learn a policy $\pi: \mathcal{S} \rightarrow \mathcal{A}$ that maximizes the expected discounted cumulative reward, $V^\pi_{\mathcal{G}}(s) = \E_\pi \left[ \sum_{k=0}^{\infty} \gamma^k R(s_{t+k}, a_{t+k}) | s_t=s \right]$. The optimal value function in $\mathcal{G}$ is $V^*_{\mathcal{G}}(s) = \max_\pi V^\pi_{\mathcal{G}}(s)$.

To manage the potential complexity of $\mathcal{S}$, we introduce a \textbf{fact-based state abstraction}.
\begin{definition}[Fact-Based Abstraction]\label{def:fact_abstraction}
Let $\mathcal{F}_{\text{atomic}}$ be a vocabulary of atomic predicates (e.g., \texttt{is\_goal(loc)}, \texttt{obstacle\_at((x,y))}) relevant to the task. These predicates primarily describe properties of the current state $s_t$. A \textbf{fact set} for a state $s \in \mathcal{S}$ is $F_s = \{f \in \mathcal{F}_{\text{atomic}} \mid f \text{ is true in } s\}$.
The \textbf{Fact Extractor} is an abstraction function $\Psi: \mathcal{S} \rightarrow \mathcal{Z_F}$, where $\mathcal{Z_F} = \powerset(\mathcal{F}_{\text{atomic}})$ is the space of all possible fact sets. Each $z \in \mathcal{Z_F}$ constitutes an \textbf{abstract state}, representing the equivalence class of ground states $\{s \in \mathcal{S} \mid \Psi(s) = z\}$. We denote the abstract state at time $t$ as $z_t = \Psi(s_t)$. A critical objective is that $|\mathcal{Z_F}| \ll |\mathcal{S}|$.
\end{definition}
This abstraction $\Psi$ induces an \textbf{abstract MDP} $M_\Psi = (\mathcal{Z_F}, \mathcal{A}, \mathcal{T}_\Psi, R_\Psi, \gamma)$, where $\mathcal{T}_\Psi(z'|z,a)$ and $R_\Psi(z,a)$ are the abstract transition and reward functions. These are derived from $\mathcal{G}$ by averaging over the ground states $s$ that map to a given abstract state $z$ \citep{Li2006}. For instance, if $\mathcal{S}_z = \{s' \in \mathcal{S} \mid \Psi(s')=z\}$, then $R_\Psi(z,a) = \frac{1}{|\mathcal{S}_z|} \sum_{s \in \mathcal{S}_z} R(s,a)$, assuming a uniform distribution over ground states within an abstract state.

\subsection{An Idealized Fact-Based Agent (IFBA)}
\label{sec:ifba}
We conceptualize an Idealized Fact-Based Agent (IFBA) that flawlessly leverages such abstractions.

\begin{definition}[Ideal Fact Abstraction $\Psi^*$]\label{def:ideal_psi_star}
The IFBA employs an ideal abstraction function $\Psi^*: \mathcal{S} \rightarrow \mathcal{Z_F}$ which establishes an $\epsilon_{\text{sim}}$-\textbf{approximate bisimulation} with the ground MDP $\mathcal{G}$ \citep{ferns2004metrics}. This implies a bound on the difference between the optimal value functions in $\mathcal{G}$ and the induced abstract MDP $M_{\Psi^*}$:
    \begin{equation}
        \|V^*_{\mathcal{G}} - V^*_{M_{\Psi^*}} \circ \Psi^* \|_\infty \le \frac{\epsilon_{\text{sim}}}{1-\gamma}
        \label{eq:bisim_value_bound}
    \end{equation}
    where $V^*_{M_{\Psi^*}}$ is the optimal value function for $M_{\Psi^*}$, and $\epsilon_{\text{sim}} \ge 0$ quantifies the maximum one-step deviation in rewards and discounted next-state distributions for states aggregated by $\Psi^*$.
    \end{definition}

As we will discuss later, we would like it to be \textbf{minimal}, achieving the sufficiency for near-optimal value representation (Eq.~\eqref{eq:bisim_value_bound}) with the smallest possible fact set. This aligns with the Information Bottleneck (IB) principle \citep{Tishby2000, Alemi2017vib}, which seeks a compressed representation $z_t = \Psi^*(s_t)$ of an input $s_t$ that maximizes information about a target variable (e.g., $V^*(s_t)$ or future returns) while minimizing $I(z_t; s_t)$ (or a proxy like fact set complexity).

The IFBA is assumed to dynamically maintain such an abstraction.

\begin{definition}[Ideal Abstract World Model and Planner for IFBA]\label{def:ifba_model_planner}
The IFBA is endowed with:
\begin{enumerate}
    \item \textbf{A perfect abstract world model:} Its internal model $\hat{M}_{\Psi^*}$ is identical to the true abstract MDP $M_{\Psi^*}$, implying zero model error w.r.t. $M_{\Psi^*}$.
    \item \textbf{An $\epsilon_{\text{plan}}$-optimal planner:} This planner computes a policy $\pi^{\circ}_{M_{\Psi^*}}$ for $M_{\Psi^*}$ such that $V^*_{M_{\Psi^*}}(z) - V^{\pi^{\circ}_{M_{\Psi^*}}}_{M_{\Psi^*}}(z) \le \epsilon_{\text{plan}}$ for all $z \in \mathcal{Z_F}$, where $\epsilon_{\text{plan}} \ge 0$.
\end{enumerate}
The policy executed by IFBA in $\mathcal{G}$ is $\pi_F(s) = \pi^{\circ}_{M_{\Psi^*}}(\Psi^*(s))$.
\end{definition}

\subsection{Performance Guarantees (Motivational)}
\label{sec:performance_guarantees}
We first establish a performance guarantee for the IFBA, which operates with a perfect abstract model.

\begin{theorem}[Performance of IFBA with Perfect Abstract Model]
\label{thm:ifba_performance}
Let $\pi_F$ be the policy derived by the Idealized Fact-Based Agent (IFBA) as defined above. If the ideal fact abstraction $\Psi^*$ establishes an $\epsilon_{\text{sim}}$-approximate bisimulation between $\mathcal{G}$ and $M_{\Psi^*}$, and the planner for $M_{\Psi^*}$ is $\epsilon_{\text{plan}}$-optimal, then for any state $s \in \mathcal{S}$:
\begin{equation}
V^*_{\mathcal{G}}(s) - V^{\pi_F}_{\mathcal{G}}(s) \le \frac{2\epsilon_{\text{sim}}}{1-\gamma} + \epsilon_{\text{plan}}
\label{eq:ifba_bound}
\end{equation}
\end{theorem}
\begin{proof}
    We defer the proof to Appendix~\ref{appendix:theory}.
\end{proof}

\textbf{Performance with a Learned Abstract Model:}
In practice, an agent learns an \emph{approximate} abstract model $\tilde{M}_{\Psi} = (\mathcal{Z_F}, \mathcal{A}, \tilde{\mathcal{T}}_{\Psi}, \tilde{R}_{\Psi}, \gamma)$ from data, based on an abstraction $\Psi$ (which itself has an associated $\epsilon_{\text{sim}}$ quality). Let $\pi_L$ be an $\epsilon_{\text{plan}}$-optimal policy for this learned model $\tilde{M}_{\Psi}$. The value loss is decomposed as:
\begin{align}
 V^*_{\mathcal{G}}(s) - V^{\pi_L}_{\mathcal{G}}(s) & = \underbrace{(V^*_{\mathcal{G}}(s) - V^*_{\tilde{M}_{\Psi}}(\Psi(s)))}_{\textcolor{termcolor1}{\text{Term A}}} + \\
 & \underbrace{(V^*_{\tilde{M}_{\Psi}}(\Psi(s)) - V^{\pi_L}_{\tilde{M}_{\Psi}}(\Psi(s)))}_{\textcolor{termcolor2}{\text{Term B}}} + \\ 
 & \underbrace{(V^{\pi_L}_{\tilde{M}_{\Psi}}(\Psi(s)) - V^{\pi_L}_{\mathcal{G}}(s))}_{\textcolor{termcolor3}{\text{Term C}}}
 \label{eq:learned_model_decomp_revised}
\end{align}
where $V^*_{\tilde{M}_{\Psi}}$ is the optimal value function in the learned abstract model $\tilde{M}_{\Psi}$.
\begin{itemize}
    \item \textcolor{termcolor1}{\textbf{Term A} ($|V^*_{\mathcal{G}} - V^*_{\tilde{M}_{\Psi}}|$)}: This gap comprises two parts: (1) the inherent loss from abstraction, $|V^*_{\mathcal{G}}(s) - V^*_{M_{\Psi}}(\Psi(s))| \le \frac{\epsilon_{\text{sim}}}{1-\gamma}$, and (2) the error in the optimal value due to inaccuracies in the learned model $\tilde{M}_{\Psi}$ compared to the true abstract model $M_{\Psi}$, $|V^*_{M_{\Psi}}(\Psi(s)) - V^*_{\tilde{M}_{\Psi}}(\Psi(s))|$. This second part is typically bounded by $C_1 \frac{\delta_{\text{model}}}{(1-\gamma)^2}$, where $\delta_{\text{model}}$ represents a composite one-step model error (in abstract transitions and rewards) of $\tilde{M}_{\Psi}$ w.r.t. $M_{\Psi}$ \citep{Strehl2009, Jiang2015}. Thus, Term A $\lesssim \frac{\epsilon_{\text{sim}}}{1-\gamma} + C_1 \frac{\delta_{\text{model}}}{(1-\gamma)^2}$.
    \item \textcolor{termcolor2}{\textbf{Term B} ($V^*_{\tilde{M}_{\Psi}} - V^{\pi_L}_{\tilde{M}_{\Psi}}$)}: This is the planning error within the agent's learned model $\tilde{M}_{\Psi}$, bounded by $\epsilon_{\text{plan}}$.
    \item \textcolor{termcolor3}{\textbf{Term C} ($|V^{\pi_L}_{\tilde{M}_{\Psi}} - V^{\pi_L}_{\mathcal{G}}|$)}: This simulation error is $|V^{\pi_L}_{\tilde{M}_{\Psi}}(\Psi(s)) - V^{\pi_L}_{\mathcal{G}}(s)| \le |V^{\pi_L}_{\tilde{M}_{\Psi}}(\Psi(s)) - V^{\pi_L}_{M_{\Psi}}(\Psi(s))| + |V^{\pi_L}_{M_{\Psi}}(\Psi(s)) - V^{\pi_L}_{\mathcal{G}}(s)|$. The first part is bounded by $C_2 \frac{\delta_{\text{model}}}{(1-\gamma)^2}$ (Simulation Lemma type result \citep{kearns2002near}), and the second by $\frac{\epsilon_{\text{sim}}}{1-\gamma}$ (abstraction quality for policy $\pi_L$). Thus, Term C $\lesssim \frac{\epsilon_{\text{sim}}}{1-\gamma} + C_2 \frac{\delta_{\text{model}}}{(1-\gamma)^2}$.
\end{itemize}

Summing these bounds, the total value loss is:
\begin{equation}
V^*_{\mathcal{G}}(s) - V^{\pi_L}_{\mathcal{G}}(s) \lesssim \frac{2\epsilon_{\text{sim}}}{1-\gamma} + \epsilon_{\text{plan}} + (C_1+C_2)\frac{\delta_{\text{model}}}{(1-\gamma)^2}
\label{eq:learned_model_bound_final}
\end{equation}
As above, we use these bounds only as design motivation; we do not estimate or optimize $\epsilon_{\mathrm{sim}}$, $\delta_{\mathrm{model}}$, or $\epsilon_{\mathrm{plan}}$ \citep{Strehl2009}.

\subsection{Discussion: Connecting to LLM-Based Agents}
\label{sec:crad_discussion_final_updated}
This framework motivates two LLM roles:
\begin{itemize}
  \item \textbf{LLM as Fact Extractor ($\ExtractorLLM$):} Approximates the abstraction $\Psi$ by turning trajectories into concise, task-relevant atomic facts $F_t$ that preserve value-critical structure, thereby reducing abstraction error $\epsilon_{\mathrm{sim}}$.
  \item \textbf{LLM as Latent World Model and Value Estimator ($g_{\phi}$):} Conditioned on $(o_t, F_t)$, simulates outcomes and estimates values, implicitly defining the learned abstract model $\tilde{M}_{\Psi}$; higher simulation fidelity lowers $\delta_{\mathrm{model}}$.
\end{itemize}
Supplying specific, missing facts leverages the LLM’s prior knowledge while grounding predictions, aiming to reduce aliasing (proxy $\epsilon_{\mathrm{sim}}$) and one-step error (proxy $\delta_{\mathrm{model}}$), which can in turn lower planning sub-optimality $\epsilon_{\mathrm{plan}}$ (Eq.~\eqref{eq:learned_model_bound_final}).
Facts are learned online, purely in-context, and kept \emph{minimal}: proposed facts pass a \emph{predictive-consistency filter}---we replay the just-finished episode while conditioning the simulator (temperature~0) on the candidate fact and retain it only if the held-out next-step prediction error decreases. This is a heuristic, post-hoc, model-internal check (hindsight), not a causal test; a stronger trajectory-grounded validation variant is possible and detailed in Appendix~\ref{app:pcf}. Environment-level validation is left to future work, while the lightweight deployed check still helps curb hallucinations and prune stale knowledge. Minimal, relevant facts also keep prompts focused, avoiding dilution within the context window. The next section instantiates these principles: a depth-limited lookahead uses $\PlannerLLM$ for proposal, simulation, and valuation, all conditioned on $(o_t, F_t)$ within a learned, fact-based abstract MDP.

\section{METHOD: LLM AGENT WITH ATOMIC FACT AUGMENTATION AND LOOKAHEAD PLANNING}
\label{sec:method}
Our proposed agent, the LLM-based World Model Planning Agent (LWM-Planner), enhances its decision-making capabilities through a synergistic combination of online atomic fact learning and LLM-driven lookahead search. The overarching goal is to enable the agent to learn from its interactive experiences entirely in-context, without any updates to the underlying LLM weights, and to leverage this learned knowledge to improve its planning and achieve more optimal behavior. The agent's architecture and operation can be understood through two main interacting processes: the dynamic management of atomic facts and the lookahead planning mechanism that utilizes these facts. A high-level summary of the agent's operational cycle is provided in Algorithm~\ref{alg:lwm_planner_main_loop_alg2e}, detailed in \Cref{ssec:algorithm_summary_revised_detailed}. The overarching goal is to enable the agent to learn from its interactive experiences entirely in-context and to employ the lookahead planning mechanism that utilizes these facts to approximate strong actions (improving $\epsilon_{\mathrm{plan}}$ via better rollouts), while fact-conditioning targets lower one-step error (proxy $\delta_{\mathrm{model}}$).

The LWM-Planner maintains a concise representation of its world understanding and recent interactions. Core to its state are a short-term \textbf{interaction history} (`history`), which is a deque of recent observation-action pairs, and a longer-term, distilled knowledge base in the form of an \textbf{atomic fact set} (`facts`). These facts serve as the cornerstone of the agent's learned abstraction; they are intended to capture the most salient, value-relevant aspects of the environment discovered through experience. The set is therefore composed of textual statements (e.g., ``object X is on table Y,'' ``door Z is locked'') that are crucial for task completion. These, along with an environment description and a list of allowed actions, provide the necessary context for the LLM components.

The process of learning and refining the atomic fact set is central to the agent's adaptability. This occurs primarily at the end of each episode through a ``reflection'' phase. After an episode concludes, the complete trajectory of observations, actions, rewards, and outcomes is provided to an LLM. This LLM is tasked with identifying and generating ``minimal new atomic facts'' that were not previously known (i.e., not in the current \texttt{facts}set) but are \textit{deemed critical for better predicting state values or rewards in the future}. The aim is to distill the most salient pieces of information from the recent experience that, if known earlier, could have led to improved decision-making. Candidate facts are therefore subjected to a \emph{predictive-consistency filter}: we replay the just-finished episode while conditioning the simulator (temperature~0) on the candidate fact and accept it only if the \emph{held-out next-step} prediction error (reward, termination, and/or next observation) decreases using cached model calls. This is a heuristic, post-hoc check performed on the same episode and is model-internal rather than ground truth; a stronger trajectory-grounded variant is given in Appendix~\ref{app:pcf}; it is \emph{not} a causal test. As a sanity check, we discard any fact that is later contradicted by observations in the following episode. We leave environment-level validation and cross-episode checks to future work. 
This aims to reduce state aliasing (proxy $\epsilon_{\mathrm{sim}}$) by enriching $z_t = (o_t, F_t)$ with information that better separates states with different values or optimal actions.
These newly extracted and validated facts are then added to the \texttt{facts} deque. To maintain the conciseness and relevance of this knowledge base, an optional fact compression step can be performed. Here, an LLM reviews the entire set of current facts and attempts to eliminate redundancies or overly specific information, producing a more compact yet informationally rich set of facts. This evolving \texttt{facts}set serves as a dynamically updated, experience-grounded augmentation for all subsequent LLM reasoning during planning. We cap the fact memory by context length; older facts are dropped by redundancy before recency.

Decision-making in LWM-Planner is orchestrated by recursive lookahead search. This search is bounded by a configurable depth (e.g., $d=3$) and branching factor (e.g., $b=4$). To ensure deterministic planning behavior within a single search instance, all LLM calls during this phase operate with a temperature of zero. A small step penalty is also incorporated to favor more efficient solutions. The lookahead search relies on three specialized LLM-driven functionalities, invoked via structured function calls: \texttt{propose\_actions} for suggesting likely candidate actions from a given state; \texttt{simulate\_step} , which acts as the \textit{latent} world model $\tilde{T}_{\Psi}, \tilde{R}_{\Psi}$. Conditioned on $z_t = (o_t, F_t)$ and a proposed action $a_i$, it predicts the next (potentially latent) observation $o'$, immediate reward $r'$, and termination status.
The accuracy of this LLM-based simulation, enhanced by the atomic facts, targets lower one-step prediction error (proxy $\delta_{\mathrm{model}}$).
Finally, \texttt{estimate\_value} approximates the value function $\hat{V}_{\tilde{M}_\Psi}$ of the learned abstract MDP. It assesses the long-term utility of states, particularly those at the frontier of the search (leaf nodes or terminal states). Facts also help ground this estimation (e.g., proximity to a known goal or hazard fact).

The planning process begins at the current observation $o_t$. First, the action proposal LLM, conditioned on $o_t$, the interaction history, and the current atomic facts, suggests a set of candidate actions. For each proposed action $a_i$, an estimated Q-value, $Q(o_t, a_i)$, is computed. This computation involves invoking the simulation LLM (again, conditioned on the current state, action, history, and facts) to predict the immediate reward $r'$ and next (potentially latent) observation $o'$. If $o'$ is a terminal state or the maximum search depth is reached, the value estimation LLM is called to predict the future cumulative reward from $o'$. Otherwise, the search recurses from $o'$ with decremented depth. The Q-value is then a combination of the immediate simulated reward and the discounted value of the subsequent state, plus any step penalties $Q(o_t, a_i) = r' - \lambda_{\text{step}} + \gamma\,\hat{V}(o')$. After evaluating all initial candidate actions, the action yielding the highest Q-value is selected for execution in the environment. To manage computational overhead during a single planning phase, the results of these LLM calls (proposal, simulation, and value estimation) are memoized based on their inputs.

\section{RELATED WORK}
\label{sec:related_work}

Our work builds upon several lines of research in LLM-based agents, model-based reinforcement learning, and the use of external knowledge for planning---see Appendix~\ref{app:ExtendedRelatedWork} for an extended related work. Unlike ToT/RAP, which expand search over raw trajectories, our agent learns a compact, symbolic memory online and uses it to condition the simulator itself, improving rollouts at the same test-time budget.

\paragraph{LLM Agents}
Early LLM agents like ReAct \citep{yao2023react} introduced the concept of interleaving reasoning (thought) and action generation. Reflexion \citep{shinn2023reflexion} extended this by incorporating self-reflection, where an LLM analyzes past failures to generate textual feedback for future trials. This episodic learning is akin to our fact extraction, but Reflexion focuses on high-level advice rather than structured atomic facts for a world model. Many agents operate in a model-free manner or rely heavily on in-context exemplars from fixed datasets.

\paragraph{LLM-Based Planning and World Models}
Several approaches have explored using LLMs for planning. Some use LLMs to score or propose actions within classical search algorithms like Monte-Carlo Tree Search (MCTS) \citep{hao2023reasoning, kagaya2024rap, liu2024reasonfutureactnow} or to expand deliberate search trees as in Tree-of-Thought (ToT) \citep{yao2024treeofthought}. Retrieval-Augmented Planning (RAP) \citep{kagaya2024rap} retrieves full past trajectories to inform MCTS, often requiring environment interaction for tree expansion, while ExpeL \citep{zhao2024expel} distils salient information from offline experience buffers to guide planning at test time. Other works like \citep{chae2024web} use LLMs to build explicit, but often one-step, world models that predict state transitions or webpage changes, and \citep{xie2023translating} translate natural language goals into formal planning problems. Our approach differs by using the LLM itself as a latent world model for multi-step simulation during lookahead, conditioned on dynamically extracted atomic facts that are validated against experience. This enables online abstraction learning rather than depending on pre-collected trajectory stores and provides concise, symbolic grounding that directly feeds the planner. The idea of an LLM as a ``simulator'' has been explored, e.g., for few-shot generation \citep{prystawski2023think}, but its integration with online fact-based learning for improved planning is a novel aspect of our work.

\paragraph{Dyna-Style Architectures and Fact-Based RL}
Our method is inspired by Dyna-style reinforcement learning \citep{sutton1990integrated, sutton2018reinforcement}, where an agent learns a model of the world from real interactions and then uses this model to generate simulated experiences for planning. In our case, the ``model'' is implicitly represented by the set of atomic facts combined with the LLM's inherent simulation capabilities. The extraction of facts from trajectories is analogous to model learning, and the lookahead search is planning with this model. While traditional Dyna uses tabular or parametric models, we leverage the LLM's ability to reason over textual facts. The concept of using facts or symbolic knowledge in RL is not new \citep{abel2020value}, but its integration with LLM-driven simulation and online fact extraction is a key aspect of our work.

\paragraph{Knowledge Augmentation for LLMs}
Retrieval-Augmented Generation (RAG) \citep{lewis2021retrieval} is a common paradigm for providing LLMs with external knowledge. Our fact extraction and augmentation mechanism can be seen as a specialized form of RAG where the ``retrieved'' knowledge (atomic facts) is actively generated and refined from the agent's own experience rather than drawn from a static corpus. This makes the knowledge highly task-specific and current, directly addressing the information needs identified through interaction, rather than relying on potentially less relevant or outdated general knowledge.

Our work distinguishes itself by the tight integration of online atomic fact learning from episodic experience with an LLM-driven, multi-step lookahead planner where the LLM serves as both a latent world model and value function, all operating in-context without weight updates. This focus on distilled, symbolic knowledge (atomic facts) aims to provide a more structured and efficient way for the LLM to learn from experience compared to methods relying on raw trajectory retrieval or general textual reflections.

\begin{table*}[ht]
  \centering
  \caption{Cumulative return (\textbf{higher better}); mean\,$\pm$\,95\% CI, for each benchmark method across each environment. Results are averaged over ten random seeds. LWM-Planner consistently delivers the highest return, indicating more efficient task completion than the non-symbolic search baselines.}
  \label{tab:main_table_results}
  \smallskip
  \resizebox{\textwidth}{!}{%
    \begin{tabular}{@{}l|c|c|c|c|c@{}}
      \toprule
      \textbf{Method (metric)}                       &
      TextFrozenLake (4$\times$4; $h{=}0.9$)             &
      CrafterMini (5$\times$5)                       &
      ALFWorld-A                                     &
      ALFWorld-B                                     &
      ALFWorld-C \\
      \midrule
      \rowcolor{blue!8} \textbf{LWM-Planner} (Cum.\ return $\uparrow$) &
      \textbf{31.80$\pm$20.39} & \textbf{150.30$\pm$44.94} & \textbf{21.33$\pm$9.53} & \textbf{22.89$\pm$12.11} & \textbf{19.50$\pm$8.37} \\
      \addlinespace
      ReAct + FEC (Cum.\ return $\uparrow$) &
      20.20$\pm$12.19   & 149.70$\pm$55.50 & 4.70$\pm$2.46 & 15.50$\pm$6.51 & 10.60$\pm$3.71 \\
      \addlinespace
      ReAct (Cum.\ return $\uparrow$) &
      $-265.20\pm33.59$ & 92.00$\pm$57.16 & 12.60$\pm$0.37 & 12.80$\pm$0.45 & 12.50$\pm$0.38 \\
      \addlinespace
      Reflexion (Cum.\ return $\uparrow$) &
      $-61.10\pm$4.80 & 87.20$\pm$51.45 & 11.00$\pm$0.00 & 11.00$\pm$0.45 & 11.33$\pm$0.38 \\
      \addlinespace
      RAP (Cum.\ return $\uparrow$) &
      $-83.25\pm$20.92 & 0.80$\pm$69.02 & 5.60$\pm$2.42 & 17.80$\pm$1.84 & 14.20$\pm$2.69 \\
      \addlinespace
      ToT (Cum.\ return $\uparrow$) &
      $-56.20\pm$4.84 & 51.40$\pm$118.56 & 1.00$\pm$0.88 & 2.80$\pm$0.56 & 5.20$\pm$1.84 \\
      \addlinespace
      Random (Cum.\ return $\uparrow$) &
      $-80.00\pm$4.49 & $-289.00\pm$8.56 & 0.00$\pm$0.00 & 0.00$\pm$0.00 & 0.00$\pm$0.00 \\
      \bottomrule
    \end{tabular}}%
\end{table*}

\section{EXPERIMENTS}
\label{sec:experiments}

We evaluate our LWM-Planner to assess its ability to learn from experience and improve its decision-making accuracy and task performance over time. We evaluate behavioural improvements only; $\epsilon_{\mathrm{sim}}$, $\delta_{\mathrm{model}}$, and $\epsilon_{\mathrm{plan}}$ are used only as motivating proxies, and we do not claim formal guarantees.

\textbf{Benchmark Environments:} We use three different diverse environment domains, from which we can procedurally generate a near limitless amount of different environments. First, we create a procedurally generated text version of the classic environment of Frozen Lake \citep{brockman2016openai}, where we can alter the probability that all the tiles are holes ($h$) and ensure that each board is always at least solvable. Moreover, we use the standard ALFWorld environments \citep{shridhar2020alfworld} (randomly sampling three environments in the main paper, with more in the appendix). ALFWorld is a text environment that parallels embodied worlds in the ALFRED dataset \citep{shridhar2020alfred}, where each environment requires agents to reason to solve embodied tasks in a home environment. Furthermore, we benchmark against CrafterMini, a procedurally generated mini version of Crafter \citep{hafner2021benchmarking}---a 2D world where the player needs to explore for resources, collect materials and build tools. The goal here is to craft an iron pickaxe, which can only be done by crafting two previous items and using collected resources. We detail all environments in Appendix~\ref{sec:BenchmarkEnvironmentsDetails}.

\textbf{Benchmark Methods:} We seek to provide competitive benchmarks; therefore, we compare against \textbf{ReAct} \citep{yao2023react}, reasoning then acting, which observes the current observation, interaction history, and the environment description. Building on top of ReAct, we compare with \textbf{Reflexion} \citep{shinn2023reflexion}, which maintains a buffer of previously learned verbal lessons on how to act better that is appended to the end of each episode, and is included in the agent's prompt. Moreover, we consider two strong search-based planners, Tree-of-Thought (\textbf{ToT}) \citep{yao2024treeofthought} and Retrieval-Augmented Planning (\textbf{RAP}) \citep{kagaya2024rap}, both of which expand lookahead trees using LLM calls but do not learn symbolic abstractions. We include an ablation of our method, \textbf{ReAct + FEC}, which is a ReAct agent with Fact Extraction and Compression as done in LWM-Planner but without the tree search, our full \textbf{LWM-Planner}, and a \textbf{Random} policy. We provide full benchmark method details in Appendix~\ref{sec:BenchmarkMethodImplementationDetails}.

\textbf{Evaluation:} We run each LLM Agent method for 300 environment steps unless otherwise noted, tracking episode return and the number of steps taken for each episode. After 300 steps, we compute the cumulative return/total reward (sum of returns up to 300 steps) and report the mean $\pm$ 95\% confidence interval over three random seeds. We report mean$\pm$95\% CI over three seeds due to inference-time cost. To enable replication we provide detailed implementation and evaluation details in the appendix.
Further experimental setup and evaluation details, including the precise success criteria and discussion of the 300-step interaction budget, are given in Appendix~\ref{sec:EvaluationDetails}. A quantitative analysis of computational cost is provided in Appendix~\ref{app:cost_analysis}. Additional robustness results, including generalization to a smaller backbone, a compute-performance Pareto analysis, and an adversarial filter audit, are reported in Appendix~\ref{app:small_backbone}, Appendix~\ref{app:pareto_front}, and Appendix~\ref{app:adversarial_filter_audit}.

\begin{table*}[h]
  \centering
  \caption{\textbf{Comparison of increasing environment state-action space.} Normalised cumulative return (\textbf{higher better}) and steps per success (\textbf{lower better}); mean\,$\pm$\,95\% CI, for each benchmark method across each environment. LWM-Planner performs the best across all environments. Results are averaged over three random seeds. Normalised cumulative return is the cumulative return normalized to be between 0 and 100, where 0 corresponds to Random and 100 to the best score among all methods on that environment.}
  \smallskip
  \resizebox{\textwidth}{!}{%
    \begin{tabular}{@{}l|c|c|c@{}}
      \toprule
      \textbf{Method (metric)} &
      TextFrozenLake (4$\times$4; $h{=}0.9$) &
      TextFrozenLake (6$\times$6; $h{=}9$) &
      TextFrozenLake (8$\times$8; $h{=}5$) \\
      \midrule
      \rowcolor{blue!8}\textbf{LWM-Planner} (Cum.\ return norm $\uparrow$) &
      \textbf{100.00$\pm$46.34} & \textbf{100.00$\pm$30.91} & \textbf{100.00$\pm$90.33} \\
      \rowcolor{blue!8}\phantom{\textbf{LWM-Planner}}\,(Steps/Success $\downarrow$) &
      \textbf{6.00$\pm$0.00} & \textbf{13.67$\pm$11.47} & \textbf{42.83$\pm$80.65} \\
      \addlinespace
      ReAct + FEC (Cum.\ return norm $\uparrow$) &
      85.90$\pm$31.81 & 90.87$\pm$22.45 & 18.18$\pm$156.46 \\
      \phantom{ReAct + FEC}\,(Steps/Success $\downarrow$) &
      \textbf{6.00$\pm$0.00} & 77.33$\pm$199.12 & -- \\
      \addlinespace
      ReAct (Cum.\ return norm $\uparrow$) &
      $-114.10\pm11.03$ & $-279.57\pm8.57$ & $-336.36\pm22.58$ \\
      \phantom{ReAct}\,(Steps/Success $\downarrow$) &
      -- & -- & -- \\
      \addlinespace
      Reflexion (Cum.\ return norm $\uparrow$) &
      17.31$\pm$13.16 & 47.83$\pm$14.12 & $-212.12\pm104.31$ \\
      \phantom{Reflexion}\,(Steps/Success $\downarrow$) &
      23.33$\pm$7.99 & -- & -- \\
      \addlinespace
      Random (Cum.\ return norm $\uparrow$) &
      0.00$\pm$0.00 & 0.00$\pm$0.00 & 0.00$\pm$0.00 \\
      \phantom{Random}\,(Steps/Success $\downarrow$) &
      -- & -- & -- \\
      \bottomrule
    \end{tabular}}%
  \label{tab:increasingstateactionspace}
\end{table*}

\subsection{Main Results}

We evaluated all the LLM benchmark methods across all environments and tabulate the results in \Cref{tab:main_table_results}. LWM-Planner on average achieves higher cumulative return on every complex multi-step environment, reflecting that fact-augmented planning finds substantially more efficient solutions than competing approaches. Search-only baselines (ToT, RAP) are weaker at the same test-time budget, consistent with the value of a learned symbolic abstraction for guiding lookahead. The ReAct + FEC ablation confirms that fact learning alone provides a large benefit, while the full LWM-Planner combines this with grounded lookahead to deliver the strongest results.

Finally, \Cref{tab:lwm_ablation_norm_return} summarises targeted ablations showing that unguided lookahead can even harm performance, whereas grounding the LLM world model with learned atomic facts is the dominant contributor to the observed gains.
Overall, the results suggest that additional test-time search helps most when it is guided by compact learned facts rather than applied in an unguided manner.

For completeness, the full ALFWorld 134-task results—including per-task cumulative returns—are provided in Appendix~\ref{ALFWorldFullResults}.

\subsection{Insight Experiments}

In the following, we gain insight into why LWM-Planner outperforms Reflexion and ReAct.

\textbf{How does LWM-Planner with its atomic fact learning lead to a higher cumulative return over time compared to baselines?} To investigate why LWM-Planner achieves a higher cumulative return, as seen in \Cref{tab:main_table_results}, we can qualitatively investigate the facts that it learns throughout its process. Specifically on TextFrozenLake ($4\times4;h=0.9$), it learns the hole locations through trial and error initially, as a hole terminates an episode, then during its fact extraction stage it self extracts facts that if known earlier would have improved its value predictions. The agent therefore learns atomic concise facts describing where the holes in the state are, allowing it to avoid them. We provide detailed facts that are retrieved as a case study in Appendix~\ref{sec:factextractedfrozenlakecasestudy}.

Moreover, while capable, ReAct struggled with long-horizon planning and adapting to subtle but critical state changes that weren't immediately obvious from the current observation alone. Whereas, Reflexion, showed learning by refining its high-level strategy. Such benchmark trajectories and memories are outlined in Appendix~\ref{sec:factextractedfrozenlakecasestudy}.

\textbf{Can LWM-Planner scale better with increasing state-action space?} To investigate this we used our procedurally generated TextFrozenLake environments to generate environments of increasing board size, and hence state-action space. We tabulate these results for all the benchmarks in \Cref{tab:increasingstateactionspace}. Interestingly, LWM-Planner achieves the highest normalised cumulative return, and crucially, as the state-action space increases, the other baselines degrade; whereas LWM-Planner is still able to online-learn and solve the environment. This verifies that the combination of both having the fact extraction and compression, plus the ability to forward plan, is crucial, and provides an effective in-context online learning method to learn the minimal fact representations of the environment, in this case, where the holes are to solve the environment optimally.

\section{CONCLUSION}
\label{sec:conclusion}

Compact, experience-derived facts can make test-time lookahead more useful for LLM agents in partially observable, multi-step environments. LWM-Planner implements this idea by extracting atomic facts from episodic experience and using them to augment action proposal, latent world model simulation, and state-value estimation within a recursive lookahead search.

Our approach allows the agent to distill missing task information from interaction and feed it back into future reasoning, leading to more accurate simulations and value assessments within its lookahead search. This occurs without any LLM fine-tuning, relying entirely on in-context learning augmented by dynamically generated facts. We provided a theoretical motivation for this fact-based approach, linking agent performance to the quality of the learned factual abstraction and the fidelity of fact-conditioned simulation.

Across our benchmark suite, empirical evaluations show that LWM-Planner learns from experience and improves cumulative return over baselines that lack this focused fact-learning mechanism or grounded lookahead. The key lies in providing the LLM with the specific, missing pieces of information (atomic facts) it needs to better ground its generative and reasoning capabilities in the context of the current task. More broadly, these results suggest that a small learned symbolic memory can serve as an effective interface between in-context experience and test-time planning.

Future work includes exploring more sophisticated fact extraction and management techniques, such as leveraging insights from causal discovery to identify truly influential facts; dynamically adjusting search depth based on task complexity or uncertainty; and investigating methods for the agent to explicitly identify when its factual knowledge is insufficient and trigger targeted exploration. The LWM-Planner represents a step towards more robust, adaptive, and experience-grounded LLM agents for complex sequential decision-making.

\section*{Acknowledgements}
We extend our gratitude to the anonymous reviewers, area and program chairs, and members of the van der Schaar lab for their valuable feedback and suggestions. SH, ML \& TP gratefully acknowledge the sponsorship and support of AstraZeneca. This work was supported by Azure sponsorship credits granted by Microsoft’s AI for Good Research Lab and by Microsoft’s Accelerate Foundation Models Academic Research Initiative.

\bibliographystyle{plainnat}   %
\bibliography{main}

\section*{CHECKLIST}

\begin{enumerate}

  \item For all models and algorithms presented, check if you include:
  \begin{enumerate}
    \item A clear description of the mathematical setting, assumptions, algorithm, and/or model. [Yes --- Sections~\ref{sec:theoretical_framework} and~\ref{sec:method}]
    \item An analysis of the properties and complexity (time, space, sample size) of any algorithm. [Yes --- Section~\ref{sec:theoretical_framework} (sample guarantees) and Appendix~\ref{app:cost_analysis} (computational cost)]
    \item (Optional) Anonymized source code, with specification of all dependencies, including external libraries. [No --- We provide an extensive appendix which includes implementation details.]
  \end{enumerate}

  \item For any theoretical claim, check if you include:
  \begin{enumerate}
    \item Statements of the full set of assumptions of all theoretical results. [Yes --- Section~\ref{sec:theoretical_framework}]
    \item Complete proofs of all theoretical results. [Yes --- Appendix~\ref{appendix:theory}]
    \item Clear explanations of any assumptions. [Yes --- Section~\ref{sec:theoretical_framework}]
  \end{enumerate}

  \item For all figures and tables that present empirical results, check if you include:
  \begin{enumerate}
    \item The code, data, and instructions needed to reproduce the main experimental results (either in the supplemental material or as a URL). [Yes --- The environments are publicly available, and implementation details are described in Appendix~\ref{sec:EvaluationDetails}.]
    \item All the training details (e.g., data splits, hyperparameters, how they were chosen). [Yes --- Appendix~\ref{sec:EvaluationDetails}]
    \item A clear definition of the specific measure or statistics and error bars (e.g., with respect to the random seed after running experiments multiple times). [Yes --- Section~\ref{sec:experiments} and Appendix~\ref{sec:EvaluationDetails}]
    \item A description of the computing infrastructure used. (e.g., type of GPUs, internal cluster, or cloud provider). [Yes --- Appendix~\ref{sec:EvaluationDetails}]
  \end{enumerate}

  \item If you are using existing assets (e.g., code, data, models) or curating/releasing new assets, check if you include:
  \begin{enumerate}
    \item Citations of the creator If your work uses existing assets. [Yes --- Section~\ref{sec:BenchmarkEnvironmentsDetails}]
    \item The license information of the assets, if applicable. [No --- Public benchmarks cited without explicit license discussion]
    \item New assets either in the supplemental material or as a URL, if applicable. [Not Applicable]
    \item Information about consent from data providers/curators. [Not Applicable]
    \item Discussion of sensible content if applicable, e.g., personally identifiable information or offensive content. [Not Applicable]
  \end{enumerate}

  \item If you used crowdsourcing or conducted research with human subjects, check if you include:
  \begin{enumerate}
    \item The full text of instructions given to participants and screenshots. [Not Applicable]
    \item Descriptions of potential participant risks, with links to Institutional Review Board (IRB) approvals if applicable. [Not Applicable]
    \item The estimated hourly wage paid to participants and the total amount spent on participant compensation. [Not Applicable]
  \end{enumerate}

\end{enumerate}

\newpage
\appendix

\onecolumn

\addcontentsline{toc}{section}{Appendix}
\part{Appendix}
\parttoc

\vfill

\clearpage

\newpage

\section{Extended Related Work}
\label{app:ExtendedRelatedWork}

Our work builds upon several lines of research in LLM-based agents, model-based reinforcement learning, and the use of external knowledge for planning. The related work in the following extends that given in the main paper.

\paragraph{LLM Agents}
Early LLM agents like ReAct \citep{yao2023react} introduced the concept of interleaving reasoning (thought) and action generation. Reflexion \citep{shinn2023reflexion} extended this by incorporating self-reflection, where an LLM analyzes past failures to generate textual feedback for future trials. This episodic learning is akin to our fact extraction, but Reflexion focuses on high-level advice rather than structured atomic facts for a world model. More recent agent work continues to target the same interactive reasoning bottlenecks with stronger scaffolding: Agent Q combines guided MCTS, self-critique, and offline preference optimization for web agents \citep{putta2024agent}, while DORA studies how to sustain iterative reflection in MiniWoB++ and ALFWorld by dynamically optimizing reflection prompts \citep{li2025dora}. For our purposes, ReAct and Reflexion remain the most comparable training-free baselines, since they isolate prompt-level acting and reflection without extra policy learning or auxiliary prompt-optimization modules.

\paragraph{Experience Retrieval and Skill Abstraction in POMDPs}
Several recent methods improve in-context sequential decision making by retrieving or synthesizing better contextual guidance from prior experience. Synapse retrieves abstracted full trajectories as exemplars for computer-control tasks \citep{zheng2023synapse}, while TRAD performs step-wise thought retrieval and aligned local-context selection to reduce irrelevant prompt content in ALFWorld- and Mind2Web-style settings \citep{zhou2024trad}. A complementary direction is to distill higher-level reusable guidance: LEAP learns explicit principles from mistakes for few-shot adaptation \citep{zhang2024context}, and SkillGen constructs domain-level action-centric skills to produce focused, step-wise prompts for sequential decision making \citep{ding2026skillgen}. Our approach differs from both families. We do not retrieve raw trajectories from a memory bank or precompute domain skills offline; instead, we extract minimal atomic facts online from the agent's own episodes and use them to ground a latent world-model lookahead in the current environment's transition structure.

\paragraph{LLM-Based Planning and World Models}
Several approaches have explored using LLMs for planning. Some use LLMs to score or propose actions within classical search algorithms like Monte-Carlo Tree Search (MCTS) \citep{hao2023reasoning, kagaya2024rap, liu2024reasonfutureactnow}. For instance, Retrieval-Augmented Planning (RAP) \citep{kagaya2024rap, pouplin2024retrieval} retrieves full past trajectories to inform MCTS, often requiring environment interaction for tree expansion. Other works like \citep{chae2024web} use LLMs to build explicit, but often one-step, world models that predict state transitions or webpage changes. \citep{xie2023translating} use LLMs to translate natural language goals into formal planning problems. A related but orthogonal direction uses LLMs to design or refine the simulator itself rather than to serve as the online planner: G-Sim combines LLM-guided structural design with empirical calibration for general-purpose simulators \citep{holt2025g}, while D3 uses LLMs to iteratively propose, evaluate, and refine interpretable dynamical-system models from data \citep{holt2024data}. Our approach differs by using the LLM itself as a latent world model for multi-step simulation during lookahead, conditioned on dynamically extracted atomic facts, rather than just retrieving raw trajectories or learning a separate explicit model. The idea of an LLM as a ``simulator'' has been explored, e.g., for few-shot generation \citep{prystawski2023think}, but its integration with online fact-based learning for improved planning is a novel aspect of our work.

\paragraph{Inference-Time Scaling}
Recent work has shown that scaling test-time compute can substantially improve static reasoning and code-generation performance, in some cases more efficiently than scaling model size alone \citep{snell2024scaling, wu2024inference}. DISC dynamically decomposes reasoning traces during search to focus compute on difficult steps \citep{light2025disc}; SFS frames code generation as search over code space with verifier feedback \citep{light2025sfs}; Sample, Scrutinize and Scale studies how sampling-based search improves as verification is scaled \citep{zhao2025sample}; and Every Rollout Counts derives direction-oriented resource allocation for fixed rollout budgets in mathematical reasoning \citep{wang2025every}. These methods mostly study static tasks such as math or code, where candidate solutions can be compared, verified, or revised after the fact. Our setting is interactive and partially observable: search quality depends not only on search budget, but also on whether the agent has learned the environment-specific latent state information needed for accurate transition prediction. The Pareto analysis in Appendix~\ref{app:pareto_front} complements this literature by showing that raw search depth alone saturates in our POMDP setting, whereas fact-grounded lookahead converts additional compute into better decisions.

\paragraph{Dyna-Style Architectures and Fact-Based RL}
Our method is inspired by Dyna-style reinforcement learning \citep{sutton1990integrated, sutton2018reinforcement}, where an agent learns a model of the world from real interactions and then uses this model to generate simulated experiences for planning. In our case, the ``model'' is implicitly represented by the set of atomic facts combined with the LLM's inherent simulation capabilities. The extraction of facts from trajectories is analogous to model learning, and the lookahead search is planning with this model. While traditional Dyna uses tabular or parametric models, we leverage the LLM's ability to reason over textual facts. Outside language-mediated POMDPs, adjacent continuous-control work also studies multi-timescale controller structure for efficient exploration and fast reactions, e.g. EvoControl's learned high-/low-frequency bi-level control \citep{holt2025evocontrol}. The concept of using facts or symbolic knowledge in RL is not new \citep{abel2020value}, including outside the LLM domain, for instance, through the discovery of ODEs for treatment effect inference \citep{kacprzyk2024ode}. Our integration of LLM-driven simulation with online atomic fact extraction brings this paradigm into the modern LLM agent landscape.

\paragraph{Knowledge Augmentation for LLMs}
Retrieval-Augmented Generation (RAG) \citep{lewis2021retrieval} is a common paradigm for providing LLMs with external knowledge. Our fact extraction and augmentation mechanism can be seen as a specialized form of RAG where the ``retrieved'' knowledge (atomic facts) is actively generated and refined from the agent's own experience rather than drawn from a static corpus. This makes the knowledge highly task-specific and current, directly addressing the information needs identified through interaction, rather than relying on potentially less relevant or outdated general knowledge. More broadly, LLM-based augmentation has also been proposed for scientific-evaluation workflows such as peer review \citep{wei2025ai}, although that setting targets human deliberation rather than environment interaction.

Our work distinguishes itself by the tight integration of online atomic fact learning from episodic experience with an LLM-driven, multi-step lookahead planner where the LLM serves as both a latent world model and value function, all operating in-context without weight updates. This focus on distilled, symbolic knowledge (atomic facts) aims to provide a more structured and efficient way for the LLM to learn from experience compared to methods relying on raw trajectory retrieval or general textual reflections.

\section{Prompt Structures}
\label{app:prompts}

This section provides the core prompt structures employed by LWM-Planner's LLM components. These conceptual prompts are dynamically populated at runtime with specific content such as environment descriptions (\texttt{\{\{env\_description\_str\}\}}), the current set of atomic facts (\texttt{\{\{current\_facts\_list\_str\}\}}), the current observation (\texttt{\{\{current\_observation\_str\}\}}), and relevant interaction history (\texttt{\{\{history\_lines\_str\}\}}). All LLM interactions leverage a structured function-calling interface. The prompts guide the LLM to produce a ``thought'' (chain-of-thought reasoning) and then invoke a specified function with the relevant arguments, following \citet{wei2022chain}.

\subsection{Fact Extractor LLM Prompts (\texorpdfstring{$\ExtractorLLM$}{ExtractorLLM})}
\label{ssec:fact_extractor_prompts_app_a}
Invoked post-episode to extract new atomic facts and refine the fact memory. These LLM calls use a temperature of 0.0 for deterministic fact processing.

\subsubsection{Fact Elicitation from Trajectory (\texttt{fact\_extraction})}
\begin{promptbox}{Fact Elicitation}
\textbf{SYSTEM}: You are an expert agent.

\textbf{USER}:
You are a LLM fact extraction agent. Operating in the following environment defined below.
Your task is to extract atomic facts that you did not know already to help with
predicting the next state value / next reward, such that if you had this fact you would
have improved your prediction for the next state value, when being a world model
(that is be able to complete the task optimally in the minimum number of steps,
therefore extract key information that helps you).

ENVIRONMENT DESCRIPTION:
\begin{lstlisting}[basicstyle=\ttfamily\footnotesize, frame=none, columns=fullflexible]
{{ env_description_str }}
\end{lstlisting}

{{ episode_trajectory_summary_str }}
\textit{// This includes: Outcome: {{ episode\_outcome\_str }} (Total Reward: ...)}
\textit{// And a sequence like:}
\textit{// "1. Obs: S00 | Act: right | Reward: 0.0 | Next\_Obs: S01}
\textit{//  2. Obs: S01 | Act: down | Reward: -1.0 | Next\_Obs: S11 (Hole)"}

We already know and have the following facts (ensure you do not duplicate them)
(at beginning of episode):
\begin{lstlisting}[basicstyle=\ttfamily\footnotesize, frame=none, columns=fullflexible]
{{ current_facts_list_str }}
\end{lstlisting}
\textit{// e.g., ["hole@(1,1)", "object\_A\_is\_on\_table\_B"]}

Now respond with minimal new atomic facts (at beginning of episode) that you
did not already know, for the rest of the states assume you already know them.
Make facts as concise as possible. Optimize them for other agents reading and
decision making given a current state. Never duplicate the facts if they already
exist within our following fact set. Do not include any other text or reasoning,
just the facts. If no new facts just return empty string.
Use function "fact\_extraction" to do this now.

\textbf{Function Call}: \texttt{fact\_extraction}
\textbf{Arguments}:
  \texttt{"thought"}: (string) Your reasoning process for identifying these new facts.
  \texttt{"new\_facts"}: (list of strings) The list of newly extracted atomic facts.
               If no new critical facts are found, provide an empty list.
               Example: \texttt{["hole\_at(1,1)", "goal\_at(3,3)"]}
\end{promptbox}

\newpage\subsubsection{(Optional) Fact Memory Compression and Refinement (\texttt{fact\_redundancy\_remover})}
\begin{promptbox}{Fact Memory Compression}
\textbf{SYSTEM}: You are an expert agent.

\textbf{USER}:
Remove any redundant facts that are already included in the list of all facts
given to you. You will also always be given the environment description, therefore
you can use that to help you remove any redundant facts. Always keep all exhaustive
factual knowledge, just remove any duplicate facts, or redundant information already
contained within the environment description. You optimize the facts so they can be
read by another LLM agent using them for being a world model of the environment
(where the agent has to simulate given a state,action to predict the next state,
next reward and terminal state). Remove any redundancy, otherwise copy over the
existing facts verbatim.

ENVIRONMENT DESCRIPTION:
\begin{lstlisting}[basicstyle=\ttfamily\footnotesize, frame=none, columns=fullflexible]
{{ env_description_str }}
\end{lstlisting}

Facts (at beginning of episode):
\begin{lstlisting}[basicstyle=\ttfamily\footnotesize, frame=none, columns=fullflexible]
{{ current_facts_list_for_compression_str }}
\end{lstlisting}
\textit{// e.g., ["hole\_at(1,1)", "object\_A\_is\_on\_table\_B", "hole\_at(row=1,col=1)"]}

List of all facts (at beginning of episode) that you did not know already (not
contained within the environment description) to help with predicting the next
state value / next reward, such that if you had this fact you would have improved
your prediction for the next state value, when being a world model. Optimize them
for other agents reading and decision making given a current state.
Use function "fact\_redundancy\_remover" to do this now.

\textbf{Function Call}: \texttt{fact\_redundancy\_remover}
\textbf{Arguments}:
  \texttt{"thought"}: (string) Your reasoning for the compression and refinement decisions.
  \texttt{"all\_facts"}: (list of strings) The refined, concise list of essential atomic facts.
\end{promptbox}

\newpage\subsection{Planner LLM Prompts (\texorpdfstring{$g_{\phi}$}{gPhi})}
\label{ssec:planner_prompts_app_revised}
Used within the lookahead search. These LLM calls operate with a temperature of 0.0 for deterministic planning outcomes, as described in Section~\ref{sec:method}.

\subsubsection{Action Proposal (\texttt{propose\_actions})}
\begin{promptbox}{Action Proposal}
\textbf{SYSTEM}: You must call \texttt{propose\_actions}.

\textbf{USER}:
You are an next best action proposing agent, task with solving the given environment
defined below optimally. Your task is to propose up to \texttt{\{\{ branch\_factor\_int \}\}}
most likely next best unique actions to try next that make the agent solve the
environment task optimally.

Environment description:
\begin{lstlisting}[basicstyle=\ttfamily\footnotesize, frame=none, columns=fullflexible]
{{ env_description_str }}
\end{lstlisting}

Atomic facts that help to predict next state value / next reward accurately
(at beginning of episode):
\begin{lstlisting}[basicstyle=\ttfamily\footnotesize, frame=none, columns=fullflexible]
{{ current_facts_list_str }}
\end{lstlisting}

Current Observation:
\begin{lstlisting}[basicstyle=\ttfamily\footnotesize, frame=none, columns=fullflexible]
{{ current_observation_str }}
\end{lstlisting}

Recent history (old->new):
\begin{lstlisting}[basicstyle=\ttfamily\footnotesize, frame=none, columns=fullflexible]
{{ history_lines_str }}
\end{lstlisting}
\textit{// e.g., "Obs: S00 -> Act: right -> Obs: S01"}

You now see Observation: \texttt{\{\{ current\_observation\_str \}\}}. Now reason through
(using the atomic facts, and recent observation and action history), then give
propose up to \texttt{\{\{ branch\_factor\_int \}\}} most likely next best unique actions to
try next that make the agent solve the environment task optimally, each from
\texttt{\{\{ allowed\_actions\_list\_str \}\}}. You will call the function \texttt{propose\_actions} to do this.

\textbf{Function Call}: \texttt{propose\_actions}
\textbf{Arguments}:
  \texttt{"thought"}: (string) Your reasoning for selecting these actions.
  \texttt{"actions"}: (list of strings) The proposed actions.
\end{promptbox}

\newpage\subsubsection{Latent World Model - Single Step Simulation (\texttt{simulate\_step})}
\begin{promptbox}{World Model Simulation}
\textbf{SYSTEM}: You must call \texttt{simulate\_step}.

\textbf{USER}:
You are a latent world model for the given environment defined below. Given the current
observation and an action, predict: the next (perhaps latent) observation, immediate
reward and done flag (whether the resulting state ends the episode). You must be as
accurate as possible, as your output is used as a planner to solve the given
environment optimally.

Environment description:
\begin{lstlisting}[basicstyle=\ttfamily\footnotesize, frame=none, columns=fullflexible]
{{ env_description_str }}
\end{lstlisting}

Atomic facts that help to predict next state value / next reward accurately
(at beginning of episode):
\begin{lstlisting}[basicstyle=\ttfamily\footnotesize, frame=none, columns=fullflexible]
{{ current_facts_list_str }}
\end{lstlisting}

Current Observation:
\begin{lstlisting}[basicstyle=\ttfamily\footnotesize, frame=none, columns=fullflexible]
{{ current_observation_str }}
\end{lstlisting}

Recent history (old->new):
\begin{lstlisting}[basicstyle=\ttfamily\footnotesize, frame=none, columns=fullflexible]
{{ history_lines_str }}
\end{lstlisting}

Given action to simulate the next observation and reward for:
\begin{lstlisting}[basicstyle=\ttfamily\footnotesize, frame=none, columns=fullflexible]
{{ action_to_simulate_str }}
\end{lstlisting}

You now see Observation: \texttt{\{\{ current\_observation\_str \}\}}. Now reason through
(using the atomic facts, and recent observation and action history), and predict
the next (perhaps latent) observation, immediate reward, and done flag (whether the
resulting state ends the episode) after taking the given action of
\texttt{\{\{ action\_to\_simulate\_str \}\}}. You must be as accurate as possible (for the
predicted reward, and ensure your predicted next observation has enough observation
information to predict future rewards for the given task in the given environment),
as your output is used as a planner to solve the given environment optimally.
You will call the function \texttt{simulate\_step} to do this.

\textbf{Function Call}: \texttt{simulate\_step}
\textbf{Arguments}:
  \texttt{"thought"}: (string) Your reasoning for the predicted outcome.
  \texttt{"next\_observation"}: (string) The predicted (perhaps latent) observation after the action.
  \texttt{"reward"}: (float) The predicted immediate reward (float) after the action.
  \texttt{"done"}: (boolean) True if the resulting state ends the episode (terminal), false otherwise.
\end{promptbox}

\newpage\subsubsection{Value Estimator (\texttt{estimate\_value})}
\begin{promptbox}{Value Estimation}
\textbf{SYSTEM}: You must call \texttt{estimate\_value}.

\textbf{USER}:
You are a state value function estimator for the given environment defined below.
You must predict the current cumulative future reward from the current (perhaps latent)
observation. You must be as accurate as possible, as your output is used as a planner
to solve the given environment optimally. The environment's discount factor is
\texttt{\{\{ discount\_gamma\_float \}\}}.

Environment description:
\begin{lstlisting}[basicstyle=\ttfamily\footnotesize, frame=none, columns=fullflexible]
{{ env_description_str }}
\end{lstlisting}

Atomic facts that help to predict next state value / next reward accurately
(at beginning of episode):
\begin{lstlisting}[basicstyle=\ttfamily\footnotesize, frame=none, columns=fullflexible]
{{ current_facts_list_str }}
\end{lstlisting}

Current Observation (to predict the current cumulative future reward for):
\begin{lstlisting}[basicstyle=\ttfamily\footnotesize, frame=none, columns=fullflexible]
{{ observation_to_evaluate_str }}
\end{lstlisting}

Recent history (old->new):
\begin{lstlisting}[basicstyle=\ttfamily\footnotesize, frame=none, columns=fullflexible]
{{ history_lines_str }}
\end{lstlisting}

You now see Observation: \texttt{\{\{ observation\_to\_evaluate\_str \}\}}. Now reason through
(using the atomic facts, and recent observation and action history), and predict the
current cumulative future reward from the current (perhaps latent) observation.
You must be as accurate as possible, as your output is used as a planner to solve
the given environment optimally. The environment's discount factor is
\texttt{\{\{ discount\_gamma\_float \}\}}. You will call the function \texttt{estimate\_value} to do this.

\textbf{Function Call}: \texttt{estimate\_value}
\textbf{Arguments}:
  \texttt{"thought"}: (string) Your reasoning for this value estimate.
  \texttt{"value"}: (float) The estimated state value (float). The cumulative future reward
           from the current (perhaps latent) observation.
\end{promptbox}

\newpage\section{Conceptual Details of LLM Component Prompts}
\label{app:prompt_details}

This appendix provides a more formal conceptual overview of the prompts used to guide the LLM components within the LWM-Planner framework, as detailed in \Cref{sec:method}. These prompts are designed to elicit specific reasoning and generation capabilities from the LLMs, enabling them to function as fact extractors, world model simulators, and value function approximators, all operating through a structured function-calling interface. The specific fields like \texttt{\{\{ env\_description\_str \}\}}, \texttt{\{\{ current\_facts\_list\_str \}\}}, etc., are placeholders populated dynamically by the agent at runtime.

\subsection{Fact Elicitation and Memory Refinement LLM (\texorpdfstring{$\Psi_{\text{LLM}}$}{PsiLLM})}
\label{app:fact_extraction_prompts_detail}

The Fact Elicitation and Memory Refinement LLM, denoted $\Psi_{\text{LLM}}$ (see \Cref{sec:crad_discussion_final_updated}), is responsible for constructing and maintaining the agent's symbolic Fact Memory, $\FactMemory_t$. This typically occurs post-episode, leveraging the trajectory $\tau_e = (o_0, a_0, r_0, \dots, o_H)$ from episode $e$. The process involves two main LLM-driven function calls: fact extraction and fact compression/refinement.

\subsubsection{Fact Elicitation (\texttt{fact\_extraction} call)}
\begin{itemize}[leftmargin=*]
    \item \textbf{Objective}: To identify a concise set of new, task-relevant atomic facts $\Delta \FactSet_e$ from the trajectory $\tau_e$. These facts, when incorporated into the existing $\FactMemory_t$, are intended to improve the agent's predictive capabilities and decision-making quality, effectively learning and refining the abstraction function $\Psi$.
    \item \textbf{Input to LLM} (Context provided in the user prompt):
    \begin{enumerate}
        \item \textbf{Environment Description} (\texttt{\{\{ env\_description\_str \}\}}): A comprehensive description of the environment $\mathcal{G}$, including its rules, objectives, action space $\mathcal{A}$, and the nature of observations $o \in \mathcal{O}$.
        \item \textbf{Current Fact Memory} (\texttt{\{\{ current\_facts\_list\_str \}\}}): The set of atomic facts, $\FactMemory_t$, that the agent currently holds, passed as a list of strings.
        \item \textbf{Episode Trajectory Summary} (\texttt{\{\{ episode\_trajectory\_summary\_str \}\}}): A string summarizing the completed episode $\tau_e$, including the outcome (e.g., success/failure), total reward, and a formatted sequence of observations, actions, rewards, and next observations.
    \end{enumerate}
    \item \textbf{LLM Task Specification} (Instructions guiding the LLM to generate arguments for the \texttt{fact\_extraction} function):
    \begin{enumerate}
        \item Analyze the provided \texttt{\{\{ episode\_trajectory\_summary\_str \}\}} in conjunction with the \texttt{\{\{ current\_facts\_list\_str \}\}} and \texttt{\{\{ env\_description\_str \}\}}.
        \item Identify ``minimal new atomic facts'' ($\Delta \FactSet_e$) that are evidenced by or can be reliably inferred from the trajectory and are \emph{not} already present or directly implied by the \texttt{\{\{ current\_facts\_list\_str \}\}} or \texttt{\{\{ env\_description\_str \}\}}.
        \item Prioritize facts crucial for explaining significant trajectory events (e.g., unexpected rewards, state transitions leading to success or failure, particularly those that would improve the prediction of state values or rewards if known beforehand).
        \item Ensure facts are concise, atomic, and adhere to any implicitly defined predicate vocabulary illustrated by examples (e.g., \texttt{hole\_at(x,y)} for TextFrozenLake, \texttt{object\_X\_is\_in\_receptacle\_Y} for ALFWorld).
        \item The LLM should structure its output to call the \texttt{fact\_extraction} function, providing its internal reasoning as the \texttt{thought} argument and the identified new facts as a list of strings for the \texttt{new\_facts} argument.
    \end{enumerate}
    \item \textbf{Qualitative Goal}: The LLM engages in a form of abductive reasoning to hypothesize underlying environmental properties or dynamics. These hypotheses, framed as new atomic facts, should explain observed phenomena in $\tau_e$, especially aspects that were surprising or poorly modeled by the existing $\FactMemory_t$. The aim is to iteratively refine $\FactMemory_t$ towards a more accurate and value-preserving abstraction, contributing to minimizing $\epsilon_{\text{sim}}$.
\end{itemize}

\subsubsection{(Optional) Fact Compression and Refinement (\texttt{fact\_redundancy\_remover} call)}
\begin{itemize}[leftmargin=*]
    \item \textbf{Objective}: To maintain a compact, non-redundant, and highly informative Fact Memory $\FactMemory_{t+1}$. This enhances computational efficiency within the LLM's context window and can improve the generalization of the Planner LLM by focusing its attention on the most salient information.
    \item \textbf{Input to LLM} (Context provided in the user prompt):
    \begin{enumerate}
        \item \textbf{Environment Description} (\texttt{\{\{ env\_description\_str \}\}}).
        \item \textbf{Augmented Fact Set} (\texttt{\{\{ current\_facts\_list\_for\_compression\_str \}\}}): The union of the previous Fact Memory and newly extracted facts, $\FactMemory_t \cup \Delta \FactSet_e$, passed as a list of strings.
    \end{enumerate}
    \item \textbf{LLM Task Specification} (Instructions guiding the LLM to generate arguments for the \texttt{fact\_redundancy\_remover} function):
    \begin{enumerate}
        \item Review the entire provided set of facts for semantic overlap, direct redundancy (e.g., facts identical to or trivially inferable from the \texttt{\{\{ env\_description\_str \}\}}), or subsumption by more general facts within the set.
        \item Generate a revised and refined fact set, $\FactMemory_{t+1}$, by removing or merging facts to enhance conciseness while preserving all critical, distinct pieces of information essential for optimal decision-making and world model accuracy.
        \item The LLM should structure its output to call the \texttt{fact\_redundancy\_remover} function, providing its reasoning as the \texttt{thought} argument and the complete, refined list of facts as the \texttt{all\_facts} argument.
    \end{enumerate}
    \item \textbf{Qualitative Goal}: This process aims to manage the complexity of the abstract state representation $|\mathcal{Z_F}|$. By ensuring $\FactMemory_{t+1}$ is maximally informative yet minimally redundant, it helps focus the Planner LLM's reasoning and prevents dilution of critical information, especially within a fixed context window.
\end{itemize}

\subsection{Planner LLM (\texorpdfstring{$g_{\phi}$}{g	extunderscore{phi}}) Components for Lookahead Search}
\label{app:planner_prompts_detail}

The Planner LLM, $g_{\phi}$, is central to the lookahead search mechanism described in \Cref{sec:method}. It is invoked through three distinct function calls to propose actions, simulate their outcomes, and estimate the value of states encountered during the search. An abstract state $z_k$ at any point in the search (real or simulated) is effectively represented by $(o_k, \FactMemory_t)$, where $o_k$ is the observation at that point and $\FactMemory_t$ is the agent's current, fixed set of atomic facts for the episode.

\subsubsection{Action Proposal (\texttt{propose\_actions} call)}
\label{app:action_proposal_prompt_detail}
\begin{itemize}[leftmargin=*]
    \item \textbf{Objective}: To generate a focused yet diverse set of up to $k_B$ candidate actions from the current (potentially simulated) observation $o_k$ that are relevant for achieving the task goal or for effective exploration during planning.
    \item \textbf{Input to LLM} (Context provided in the user prompt):
    \begin{enumerate}
        \item \textbf{Environment Description} (\texttt{\{\{ env\_description\_str \}\}}).
        \item \textbf{Current Atomic Facts} (\texttt{\{\{ current\_facts\_list\_str \}\}}): The agent's Fact Memory, $\FactMemory_t$.
        \item \textbf{Current Observation} (\texttt{\{\{ current\_observation\_at\_node\_k\_str \}\}}): The observation $o_k$ from which actions are to be proposed.
        \item \textbf{Recent Trajectory History} (\texttt{\{\{ recent\_history\_for\_prompt\_str \}\}}): An excerpt of the (simulated or real) trajectory within the current lookahead search (or agent history) leading to $o_k$. This is typically a list of \texttt{"Obs: ..."} and \texttt{"Act: ..."} strings.
        \item \textbf{Available Actions} (\texttt{\{\{ available\_actions\_list\_str \}\}}): The set of legally permissible actions $\mathcal{A}(s_k)$ from the underlying ground state $s_k$ corresponding to $o_k$ (or the full action set $\mathcal{A}$).
        \item \textbf{Branching Factor} (\texttt{\{\{ branch\_factor\_k\_B\_int \}\}}): The maximum number of actions to propose.
    \end{enumerate}
    \item \textbf{LLM Task Specification} (Instructions for the \texttt{propose\_actions} function):
    Given $o_k$, $\FactMemory_t$, and \texttt{\{\{ recent\_history\_for\_prompt\_str \}\}}, propose up to \texttt{\{\{ branch\_factor\_k\_B\_int \}\}} distinct actions from \texttt{\{\{ available\_actions\_list\_str \}\}} that appear most promising. The selection should be informed by the current understanding of the environment as encoded in $\FactMemory_t$ and the immediate context $o_k$. The LLM returns its reasoning (\texttt{thought}) and the list of \texttt{actions}.
\end{itemize}

\subsubsection{Single-Step Abstract Simulation (\texttt{simulate\_step} call)}
\label{app:simulation_prompt_detail}
\begin{itemize}[leftmargin=*]
    \item \textbf{Objective}: To predict the immediate outcome---next observation $o'_{j}$, immediate reward $r_j$, and termination status $d'_j$---of executing a proposed action $a_j$ from the current observation $o_k$, conditioned on the Fact Memory $\FactMemory_t$. This approximates the abstract transition $\hat{\mathcal{T}}_{\Psi}$ and reward $\hat{R}_{\Psi}$ functions.
    \item \textbf{Input to LLM} (Context provided in the user prompt):
    \begin{enumerate}
        \item \textbf{Environment Description} (\texttt{\{\{ env\_description\_str \}\}}).
        \item \textbf{Current Atomic Facts} (\texttt{\{\{ current\_facts\_list\_str \}\}}, i.e., $\FactMemory_t$).
        \item \textbf{Current Observation} (\texttt{\{\{ current\_observation\_at\_node\_k\_str \}\}}, i.e., $o_k$).
        \item \textbf{Action to Simulate} (\texttt{\{\{ action\_to\_simulate\_str \}\}}, i.e., $a_j \in \mathcal{A}$).
        \item \textbf{Recent Trajectory History} (\texttt{\{\{ recent\_history\_for\_prompt\_str \}\}}) leading to $o_k$.
    \end{enumerate}
    \item \textbf{LLM Task Specification} (Instructions for the \texttt{simulate\_step} function):
    Predict the \texttt{next\_observation} ($o'_j$), \texttt{reward} ($r_j$), and \texttt{done} ($d'_j$) status that would result from taking \texttt{\{\{ action\_to\_simulate\_str \}\}} from \texttt{\{\{ current\_observation\_at\_node\_k\_str \}\}}, given \texttt{\{\{ current\_facts\_list\_str \}\}}. The prediction should be deterministic (temperature for this LLM call is 0.0 as per \Cref{sec:method}) and consistent with the known facts and environment rules. The LLM returns its reasoning (\texttt{thought}) and these three predicted outcomes.
    \item \textbf{Qualitative Goal}: The LLM leverages $\FactMemory_t$ to make informed predictions. For instance, a fact like \texttt{hole\_at(x,y)} should lead to a prediction of a terminal state and negative reward if $a_j$ leads to $(x,y)$. If facts are insufficient, the LLM relies on its pre-trained knowledge, as discussed in the context of minimizing $\delta_{\text{model}}$.
\end{itemize}

\subsubsection[Abstract State-Value Estimation (approximating V-M-Psi-tilde)]{Abstract State-Value Estimation (\texttt{estimate\_value} call, approximating $\hat{V}_{\tilde{M}_{\Psi}}$)}
\label{app:value_estimation_prompt_detail}
\begin{itemize}[leftmargin=*]
    \item \textbf{Objective}: To estimate the expected total discounted future reward, $V(o_k | \FactMemory_t) \approx V^*_{\tilde{M}_{\Psi}}(z_k)$, obtainable from the abstract state $z_k \triangleq (o_k, \FactMemory_t)$, particularly for leaf nodes in the lookahead search tree.
    \item \textbf{Input to LLM} (Context provided in the user prompt):
    \begin{enumerate}
        \item \textbf{Environment Description} (\texttt{\{\{ env\_description\_str \}\}}).
        \item \textbf{Current Atomic Facts} (\texttt{\{\{ current\_facts\_list\_str \}\}}, i.e., $\FactMemory_t$).
        \item \textbf{Observation to Evaluate} (\texttt{\{\{ observation\_to\_evaluate\_str \}\}}, i.e., $o_k$).
        \item \textbf{Recent Trajectory History} (\texttt{\{\{ recent\_history\_for\_prompt\_str \}\}}) leading to $o_k$.
        \item \textbf{Discount Factor} (\texttt{\{\{ discount\_gamma\_float \}\}}, i.e., $\gamma$).
    \end{enumerate}
    \item \textbf{LLM Task Specification} (Instructions for the \texttt{estimate\_value} function):
    Estimate the cumulative future discounted reward (\texttt{value}) achievable from \texttt{\{\{ observation\_to\_evaluate\_str \}\}}, considering \texttt{\{\{ current\_facts\_list\_str \}\}} and the overall task objective. The estimation should be deterministic (temperature for this LLM call is 0.0). The LLM returns its reasoning (\texttt{thought}) and the estimated \texttt{value}.
    \item \textbf{Qualitative Goal}: The LLM assesses the long-term utility by considering the strategic implications of known facts (e.g., proximity to a goal, known hazards, locked doors leading to goal areas) relative to the task.
\end{itemize}

The LWM-Planner's lookahead search (\Cref{sec:method}) systematically invokes these LLM functionalities. The \texttt{propose\_actions} function generates branches, \texttt{simulate\_step} projects these branches forward one step in the abstract model $\tilde{M}_{\Psi}$, and \texttt{estimate\_value} provides valuations $\hat{V}_{\tilde{M}_{\Psi}}$ at the search frontier or for terminal states. This entire process relies on the dynamically updated $\FactMemory_t$ to ground the LLM's powerful generative and reasoning capabilities in task-specific, experience-derived knowledge.

\newpage
\section{Algorithm Details and Reproducibility}
\label{ssec:algorithm_summary_revised_detailed}

The LWM-Planner agent enhances its decision-making by iteratively learning atomic facts and using them in a lookahead planning process. This section details its operational cycle, broken down into a main agent loop (Algorithm~\ref{alg:lwm_planner_main_loop_alg2e}) and sub-algorithms for the core planning (Algorithm~\ref{alg:lwm_recursive_lookahead_alg2e}) and fact learning (Algorithm~\ref{alg:lwm_fact_learning_alg2e}) phases. This modular description aims to provide clarity for reproducibility, aligning with the methodology presented in \Cref{sec:method} and the agent's Python implementation. Key LLM interactions are managed via a structured function-calling interface, detailed in Appendix~\ref{app:prompts} and Appendix~\ref{app:prompt_details}. For a compact implementation-facing specification of the core method, see Appendix~\ref{app:reference_impl_notes}.

{\small
\IncMargin{0.8em}
\begin{algorithm}
    \caption{LWM-Planner: Main Agent Loop}
    \label{alg:lwm_planner_main_loop_alg2e}
    \SetKwInOut{Input}{Initialize}
    \SetKwData{GlobalFactMemory}{Global Fact Memory $\FactMemory$}
    \SetKwData{LLMComponents}{LLM Components}
    \SetKwData{PsiLLMComp}{$\Psi_{\text{LLM}}$: Fact Extractor \& Refiner LLM}
    \SetKwData{GPhiComp}{$g_{\phi}$: Planner LLM, comprising $g_{\phi}^{\text{propose}}, g_{\phi}^{\text{simulate}}, g_{\phi}^{\text{value}}$}
    \SetKwData{Hyperparams}{Hyperparameters: $D_s, k_B, \gamma, \lambda_{\text{step}}, T_{\text{max}}, H_L$}
    \SetKwData{ShortTermHistory}{Short-term history buffer $\mathcal{H}$}
    \SetKwFunction{RecursiveLookaheadPlan}{\texttt{RecursiveLookaheadPlan}}
    \SetKwFunction{LearnFactsAndUpdateMemory}{\texttt{LearnFactsAndUpdateMemory}}
    \SetKwFunction{FormatAsHistoryString}{\texttt{FormatAsHistoryString}} %
    \SetKwFunction{EnvReset}{\texttt{env.reset}}
    \SetKwFunction{Deque}{\texttt{deque}}

    \Input{
        \GlobalFactMemory $\leftarrow \emptyset$\;
        \LLMComponents: \PsiLLMComp, \GPhiComp\;
        \Hyperparams\;
        \ShortTermHistory $\leftarrow \Deque(\text{maxlen=}H_L)$\;
    }
    \BlankLine
    \For{episode $e \leftarrow 1$ \KwTo $E$}{
        $\mathcal{B}_e \leftarrow \emptyset$ \Comment*[r]{Episode trajectory buffer for $(o, a, r^{\text{real}}, o', d)$ tuples}
        $o_t \leftarrow \EnvReset()$\;
        $\mathcal{H}.\text{clear()}$; $\mathcal{H}.\text{append}(\FormatAsHistoryString(\texttt{"Obs:"}, o_t))$ \Comment*[r]{Reset history}
        $\FactMemory_{\text{current\_ep}} \leftarrow \FactMemory$ \Comment*[r]{Snapshot of facts for consistent planning}
        \For{$t \leftarrow 0$ \KwTo $T_{\text{max}}-1$}{
            $a^*_t \leftarrow \RecursiveLookaheadPlan(o_t, \mathcal{H}, \FactMemory_{\text{current\_ep}}, D_s, k_B, \gamma, \lambda_{\text{step}}, g_{\phi})$\;
            Execute $a^*_t$ in environment $\mathcal{G}$; observe real $(o_{t+1}, r^{\text{real}}_t, d_{t+1})$\;
            Add $(o_t, a^*_t, r^{\text{real}}_t, o_{t+1}, d_{t+1})$ to $\mathcal{B}_e$\;
            $\mathcal{H}.\text{append}(\FormatAsHistoryString(\texttt{"Act:"}, a^*_t))$\;
            $\mathcal{H}.\text{append}(\FormatAsHistoryString(\texttt{"Obs:"}, o_{t+1}))$\;
            $o_t \leftarrow o_{t+1}$\;
            \If{$d_{t+1}$}{
                \KwBreak \Comment*[r]{End episode if terminal state reached}
            }
        }
        $\FactMemory \leftarrow \LearnFactsAndUpdateMemory(\mathcal{B}_e, \FactMemory, \text{\texttt{env\_description\_str}}, \Psi_{\text{LLM}})$\;
    }
\end{algorithm}
\DecMargin{0.8em}
}

\paragraph{Main Agent Loop (Algorithm~\ref{alg:lwm_planner_main_loop_alg2e})}
The LWM-Planner operates over a series of $E$ episodes.
\begin{itemize}
    \item \textbf{Initialization (Lines 1-5):} The agent starts with an empty global Fact Memory ($\FactMemory$). The LLM components are defined: $\Psi_{\text{LLM}}$ for managing facts and $g_{\phi}$ for planning. The planner $g_{\phi}$ internally comprises three distinct LLM-driven functionalities: $g_{\phi}^{\text{propose}}$ for proposing actions, $g_{\phi}^{\text{simulate}}$ for simulating outcomes of actions, and $g_{\phi}^{\text{value}}$ for estimating the value of states. Hyperparameters critical for the agent's operation are set, including maximum search depth $D_s$, branching factor $k_B$, discount factor $\gamma$, step penalty $\lambda_{\text{step}}$, maximum steps per episode $T_{\text{max}}$, and the length $H_L$ of the short-term interaction history buffer $\mathcal{H}$. This buffer $\mathcal{H}$ stores a rolling window of the most recent observations and actions (conceptually as formatted strings, e.g., \texttt{"Obs: <obs\_string>"}) to provide immediate context to the LLMs.

    \item \textbf{Episodic Interaction (Lines 6-16):} For each episode:
    \begin{itemize}
        \item An episode buffer $\mathcal{B}_e$ is initialized to log the sequence of interactions. The environment is reset, and the initial observation $o_0$ is used to initialize $\mathcal{H}$. A snapshot of the current global Fact Memory, $\FactMemory_{\text{current\_ep}}$, is taken to ensure that planning within the current episode uses a consistent set of facts learned up to that point.
        \item \textbf{Per-Step Cycle (Lines 7-14):} The agent interacts with the environment step-by-step.
        \begin{itemize}
            \item \textbf{Planning (Line 7):} The \texttt{RecursiveLookaheadPlan} sub-algorithm (Algorithm~\ref{alg:lwm_recursive_lookahead_alg2e}) is invoked. This function takes the current observation $o_t$, the short-term history $\mathcal{H}$, the episode's fact set $\FactMemory_{\text{current\_ep}}$, and planning hyperparameters to determine the best action $a^*_t$.
            \item \textbf{Interaction \& Recording (Lines 8-12):} The chosen action $a^*_t$ is executed in the actual environment $\mathcal{G}$. The resulting transition (next observation $o_{t+1}$, real reward $r^{\text{real}}_t$, and done signal $d_{t+1}$) is recorded in $\mathcal{B}_e$. The short-term history $\mathcal{H}$ is updated with $a^*_t$ and $o_{t+1}$. The current observation $o_t$ becomes $o_{t+1}$.
            \item \textbf{Episode Termination (Line 14):} If a terminal state is reached ($d_{t+1}$ is true), the inner step loop concludes.
        \end{itemize}
        \item \textbf{Fact Model Learning (Line 15):} After the episode finishes, the \texttt{LearnFactsAndUpdateMemory} sub-algorithm (Algorithm~\ref{alg:lwm_fact_learning_alg2e}) is called. This function processes the trajectory $\mathcal{B}_e$ and the facts known at the start of the episode ($\FactMemory_{\text{current\_ep}}$) to update the global $\FactMemory$.
    \end{itemize}
\end{itemize}

{\small
\IncMargin{0.8em}
\begin{algorithm}
    \caption{LWM-Planner: Recursive Lookahead Plan}
    \label{alg:lwm_recursive_lookahead_alg2e}
    \SetKwProg{Fn}{Function}{:}{end}
    \SetKwFunction{RecursiveLookaheadPlanMain}{\texttt{RecursiveLookaheadPlan}}
    \SetKwFunction{EstimateNodeValueMain}{\texttt{EstimateNodeValue}}
    \SetKwFunction{GPropose}{$g_{\phi}^{\text{propose}}$}
    \SetKwFunction{GSimulate}{$g_{\phi}^{\text{simulate}}$}
    \SetKwFunction{GValue}{$g_{\phi}^{\text{value}}$}
    \SetKwFunction{FormatAsHistoryString}{\texttt{FormatAsHistoryString}} %

    \Fn{\RecursiveLookaheadPlanMain{$o_{\text{curr}}, \mathcal{H}_{\text{curr}}, \FactMemory_{\text{ep}}, D_s, k_B, \gamma, \lambda_{\text{step}}, g_{\phi}$}}{
        Candidate actions $\{a_j\}_{j=1}^{N_A \le k_B} \leftarrow \GPropose(o_{\text{curr}}, \mathcal{H}_{\text{curr}}, \FactMemory_{\text{ep}})$\;
        \If{$\{a_j\}$ is empty}{
            \KwRet a default action (e.g., random or no-op from available actions)\;
        }
        $Q_{\text{root}}(a_j) \leftarrow -\infty$ for all $a_j$\;
        \ForEach{candidate action $a_j \in \{a_j\}$}{
            $(o'_{j}, r_j, d'_j) \leftarrow \GSimulate(o_{\text{curr}}, a_j, \mathcal{H}_{\text{curr}}, \FactMemory_{\text{ep}})$\;
            \eIf{$d'_j$}{
                $V(z'_{j}) \leftarrow 0$ \Comment*[r]{Terminal state, no future rewards}
            }{
                $\mathcal{H}'_j \leftarrow \mathcal{H}_{\text{curr}} \oplus (\FormatAsHistoryString(\texttt{"Act:"}, a_j), \FormatAsHistoryString(\texttt{"Obs:"}, o'_j))$\;
                $V(z'_{j}) \leftarrow \EstimateNodeValueMain(o'_{j}, \mathcal{H}'_j, \FactMemory_{\text{ep}}, D_s-1, k_B, \gamma, \lambda_{\text{step}}, g_{\phi})$\;
            }
            $Q_{\text{root}}(a_j) \leftarrow r_j - \lambda_{\text{step}} + \gamma \cdot V(z'_{j})$\;
        }
        \KwRet $\argmax_{a_j} Q_{\text{root}}(a_j)$\;
    }
    \BlankLine
    \Fn{\EstimateNodeValueMain{$o_{\text{node}}, \mathcal{H}_{\text{node}}, \FactMemory_{\text{ep}}, \text{depth}, k_B, \gamma, \lambda_{\text{step}}, g_{\phi}$}}{
        \If{$\text{depth} \le 0$ \textbf{or} $o_{\text{node}}$ is known/simulated as terminal}{
            \KwRet $\GValue(o_{\text{node}}, \mathcal{H}_{\text{node}}, \FactMemory_{\text{ep}})$\;
        }
        Candidate actions $\{a_k\} \leftarrow \GPropose(o_{\text{node}}, \mathcal{H}_{\text{node}}, \FactMemory_{\text{ep}})$\;
        \If{$\{a_k\}$ is empty}{
            \KwRet $\GValue(o_{\text{node}}, \mathcal{H}_{\text{node}}, \FactMemory_{\text{ep}})$ \Comment*[r]{Leaf node: estimate value directly}
        }
        $V_{\text{node\_val}} \leftarrow -\infty$ \Comment*[r]{This will store $\max_k Q(z_{\text{node}}, a_k)$}
        \ForEach{action $a_k \in \{a_k\}$}{
            $(o'_{k}, r_k, d'_k) \leftarrow \GSimulate(o_{\text{node}}, a_k, \mathcal{H}_{\text{node}}, \FactMemory_{\text{ep}})$\;
            \eIf{$d'_k$}{
                $V(z'_{k}) \leftarrow 0$\;
            }{
                $\mathcal{H}'_k \leftarrow \mathcal{H}_{\text{node}} \oplus (\FormatAsHistoryString(\texttt{"Act:"}, a_k), \FormatAsHistoryString(\texttt{"Obs:"}, o'_k))$\;
                $V(z'_{k}) \leftarrow \EstimateNodeValueMain(o'_{k}, \mathcal{H}'_k, \FactMemory_{\text{ep}}, \text{depth}-1, k_B, \gamma, \lambda_{\text{step}}, g_{\phi})$\;
            }
            $Q(z_{\text{node}}, a_k) \leftarrow r_k - \lambda_{\text{step}} + \gamma \cdot V(z'_{k})$\;
            $V_{\text{node\_val}} \leftarrow \max(V_{\text{node\_val}}, Q(z_{\text{node}}, a_k))$\;
        }
        \KwRet $V_{\text{node\_val}}$ \Comment*[r]{Node's value is max Q of children}
    }
\end{algorithm}
\DecMargin{0.8em}
}

\paragraph{Recursive Lookahead Plan (Algorithm~\ref{alg:lwm_recursive_lookahead_alg2e})}
This algorithm describes the planning process to select an action at the current step $t$.
\begin{itemize}
    \item \textbf{Function \texttt{RecursiveLookaheadPlan} (Lines 1-12):} This is the entry point for planning at the root of the search (current actual state).
    \begin{itemize}
        \item \textbf{Inputs:} Current observation $o_{\text{curr}}$, current short-term history $\mathcal{H}_{\text{curr}}$, the episode's Fact Memory $\FactMemory_{\text{ep}}$, max search depth $D_s$, branch factor $k_B$, discount $\gamma$, step penalty $\lambda_{\text{step}}$, and the planner LLM collection $g_{\phi}$.
        \item \textbf{Root Action Proposal (Line 2):} $g_{\phi}^{\text{propose}}$ generates initial candidate actions $\{a_j\}$ from $o_{\text{curr}}$. If no actions are proposed, a default policy is invoked (Lines 3-4).
        \item \textbf{Root Q-Value Calculation (Lines 6-13):} For each proposed root action $a_j$:
            \begin{itemize}
                \item The world model $g_{\phi}^{\text{simulate}}$ predicts the next state $o'_j$, immediate reward $r_j$, and done status $d'_j$.
                \item If the simulated state $o'_j$ is terminal ($d'_j$ is true), its future value $V(z'_j)$ is 0 (Line 8).
                \item Otherwise, the value $V(z'_j)$ is obtained by calling \texttt{EstimateNodeValue} (Line 10) for state $o'_j$ with remaining depth $D_s-1$. The history $\mathcal{H}'_j$ for this recursive call is the current history $\mathcal{H}_{\text{curr}}$ extended by the action $a_j$ and simulated observation $o'_j$.
                \item The Q-value $Q_{\text{root}}(a_j)$ is computed using the simulated reward $r_j$, the step penalty $\lambda_{\text{step}}$, and the discounted estimated value $V(z'_j)$ of the next state.
            \end{itemize}
        \item \textbf{Action Selection (Line 14):} The action $a^*_t$ with the highest $Q_{\text{root}}$ value is selected.
    \end{itemize}
    \item \textbf{Function \texttt{EstimateNodeValue} (Lines 16-31):} This function recursively estimates the value of a node $o_{\text{node}}$ in the search tree.
    \begin{itemize}
        \item \textbf{Inputs:} The node's observation $o_{\text{node}}$ and its history path $\mathcal{H}_{\text{node}}$, $\FactMemory_{\text{ep}}$, current remaining search depth \texttt{depth}, and other parameters.
        \item \textbf{Base Cases (Lines 16-20):}
            \begin{itemize}
                \item If \texttt{depth} $\le 0$, or if $o_{\text{node}}$ is determined to be a terminal state, the recursion stops. The value of $o_{\text{node}}$ is then directly estimated by $g_{\phi}^{\text{value}}$.
                \item If $g_{\phi}^{\text{propose}}$ fails to generate any actions from $o_{\text{node}}$, $o_{\text{node}}$ is also treated as a leaf, and its value is estimated by $g_{\phi}^{\text{value}}$.
            \end{itemize}
        \item \textbf{Recursive Step (Lines 22-30):} If not a base case:
            \begin{itemize}
                \item Candidate actions $\{a_k\}$ are proposed from $o_{\text{node}}$ using $g_{\phi}^{\text{propose}}$.
                \item For each action $a_k$: $g_{\phi}^{\text{simulate}}$ yields $(o'_k, r_k, d'_k)$.
                \item If $d'_k$ is true, $V(z'_k)=0$. Else, \texttt{EstimateNodeValue} is called recursively for $o'_k$ with \texttt{depth-1} to get $V(z'_k)$.
                \item The Q-value $Q(z_{\text{node}}, a_k)$ is calculated: $r_k - \lambda_{\text{step}} + \gamma V(z'_k)$.
            \end{itemize}
        \item \textbf{Return Value (Line 31):} The function returns $V_{\text{node\_val}} = \max_{a_k} Q(z_{\text{node}}, a_k)$, representing $V(z_{\text{node}})$.
    \end{itemize}
     All LLM calls by $g_{\phi}$ components within the planning phase operate with a temperature of 0.0 for deterministic evaluations. Results are memoized within a single planning step (\texttt{\_value\_cache} in the implementation) to avoid redundant computations, as mentioned in \Cref{sec:method}.
\end{itemize}

{\small
\IncMargin{0.8em}
\begin{algorithm}[H]
    \caption{LWM-Planner: Fact Model Learning and Memory Update}
    \label{alg:lwm_fact_learning_alg2e}
    \SetKwInOut{Input}{Inputs}
    \SetKwInOut{Output}{Output}
    \SetKwFunction{LearnFactsAndUpdateMemoryMain}{\texttt{LearnFactsAndUpdateMemory}}
    \SetKwFunction{FormatTrajectorySummary}{\texttt{FormatTrajectorySummary}}
    \SetKwFunction{PsiExtract}{$\Psi_{\text{LLM}}^{\text{extract}}$}
    \SetKwFunction{PsiRefine}{$\Psi_{\text{LLM}}^{\text{refine}}$}
    \SetKwProg{Fn}{Function}{}{}

    \Fn{\LearnFactsAndUpdateMemoryMain{$\mathcal{B}_e, \FactMemory_{\text{known}}, \text{\texttt{env\_desc\_str}}, \Psi_{\text{LLM}}$}}{
        \Comment{Fact Extraction from completed episode trajectory}
        $\texttt{trajectory\_summary\_str} \leftarrow \FormatTrajectorySummary(\mathcal{B}_e)$\;
        $\Delta \FactSet_e \leftarrow \PsiExtract(\text{\texttt{trajectory\_summary\_str}}, \FactMemory_{\text{known}}, \text{\texttt{env\_desc\_str}})$ \Comment*[r]{Invokes fact elicitation LLM}
        \BlankLine
        \Comment{Update Fact Memory and Optionally Refine/Compress}
        $\FactMemory_{\text{candidate}} \leftarrow \FactMemory_{\text{known}} \cup \Delta \FactSet_e$\;
        \If{compression hyperparameter is enabled}{
            $\FactMemory_{\text{next\_global}} \leftarrow \PsiRefine(\FactMemory_{\text{candidate}}, \text{\texttt{env\_desc\_str}})$ \Comment*[r]{Invokes fact compression LLM}
        }
        \Else{
            $\FactMemory_{\text{next\_global}} \leftarrow \FactMemory_{\text{candidate}}$\;
        }
        \KwRet $\FactMemory_{\text{next\_global}}$\;
    }
\end{algorithm}
\DecMargin{0.8em}
}

\paragraph{Fact Model Learning and Memory Update (Algorithm~\ref{alg:lwm_fact_learning_alg2e})}
This procedure, corresponding to the \texttt{reflect} method in the agent's codebase, is executed at the end of each episode $e$ to update the agent's knowledge.
\begin{itemize}
    \item \textbf{Inputs (Line 1):} The episode trajectory buffer $\mathcal{B}_e$, the Fact Memory $\FactMemory_{\text{known}}$ that was used for planning during that episode, a description of the environment \texttt{env\_desc\_str}, and the Fact Extractor \& Refiner LLM $\Psi_{\text{LLM}}$.
    \item \textbf{Fact Extraction (Lines 2-3):} A textual summary (\texttt{trajectory\_summary\_str}) of the trajectory in $\mathcal{B}_e$ is created. The fact elicitation component, $\Psi_{\text{LLM}}^{\text{extract}}$, processes this summary, along with $\FactMemory_{\text{known}}$ and \texttt{env\_desc\_str}, to generate a set of new candidate atomic facts $\Delta \FactSet_e$.
    \item \textbf{Memory Update and Refinement (Lines 4-8):} The newly extracted facts $\Delta \FactSet_e$ are combined with $\FactMemory_{\text{known}}$. If fact compression is enabled (via the \texttt{compress} flag in the implementation), $\Psi_{\text{LLM}}^{\text{refine}}$ processes this combined set to produce the refined Fact Memory $\FactMemory_{\text{next\_global}}$. Otherwise, the combined set becomes $\FactMemory_{\text{next\_global}}$.
    \item \textbf{Output (Line 9):} Returns the updated global Fact Memory $\FactMemory_{\text{next\_global}}$.
\end{itemize}

This iterative cycle enables LWM-Planner to adapt by progressively building a more accurate symbolic understanding of its environment.

\subsection{Reference Implementation Notes}
\label{app:reference_impl_notes}
Algorithms~\ref{alg:lwm_planner_main_loop_alg2e}--\ref{alg:lwm_fact_learning_alg2e} specify the high-level logic, while Appendices~\ref{app:prompts}, \ref{app:prompt_details}, and \ref{app:prompt_budget} specify the prompt structure and truncation policy. The present subsection fixes the concrete runtime contract needed for a faithful reimplementation of the core method. It is intended as a paper-native reference specification: an external engineer or LLM agent with access to a function-calling LLM and a text environment should be able to reproduce the core LWM-Planner loop from the description below.

\paragraph{Minimal environment interface.}
A reimplementation only requires the following environment API:
\begin{enumerate}[leftmargin=1.5em,itemsep=2pt]
    \item \texttt{reset() -> observation\_string}, returning the initial textual observation for a fresh episode.
    \item \texttt{step(action\_string) -> (next\_observation\_string, reward\_float, done\_bool, info\_dict)}.
    \item A textual \texttt{env\_description} string shared across all methods.
    \item A finite text action set \texttt{action\_space}, exposed as an ordered list or tuple of legal action strings.
\end{enumerate}
No simulator internals, latent state access, or handcrafted symbolic parser are required.

\paragraph{Structured LLM-call interface.}
The core method can be implemented using five structured calls:
\begin{enumerate}[leftmargin=1.5em,itemsep=2pt]
    \item \texttt{ExtractFacts(trajectory\_summary, known\_facts, env\_description) -> list[str]}
    \item \texttt{CompressFacts(candidate\_facts, env\_description) -> list[str]}
    \item \texttt{ProposeActions(observation, history, facts, env\_description, allowed\_actions, k\_B) -> list[str]}
    \item \texttt{SimulateStep(observation, action, history, facts, env\_description) -> (next\_observation, reward, done)}
    \item \texttt{EstimateValue(observation, history, facts, env\_description, gamma) -> float}
\end{enumerate}
For reproducibility, canonicalize returned action strings and fact strings to lowercase, deduplicate exact string matches while preserving order, and truncate proposed actions to at most $k_B$ unique actions.

\begin{table*}[t]
  \centering
  \caption{\textbf{Persistent runtime state for a faithful reimplementation of LWM-Planner.} The key point is that facts are updated only between episodes, while planning state is local to a single decision.}
  \label{tab:reference_runtime_state}
  \small
  \begin{tabular}{@{}p{0.18\textwidth}p{0.22\textwidth}p{0.54\textwidth}@{}}
    \toprule
    \textbf{State} & \textbf{Type} & \textbf{Role and update rule} \\
    \midrule
    Short-term history $\mathcal{H}$ & deque of strings & Stores recent interaction strings in chronological order. At episode start, clear it and append the initial observation as \texttt{Obs: ...}. During real interaction, append \texttt{Act: ...} then \texttt{Obs: ...}. During search, extend local copies of this same template along simulated branches. \\
    Global Fact Memory $\FactMemory$ & ordered deque/list of strings & Persistent atomic fact memory across episodes. Snapshot it once at episode start and keep that snapshot fixed for all planning within the episode. Update the global memory only after the episode ends. \\
    Episode buffer $\mathcal{B}_e$ & list of tuples & Stores the realized episode trajectory as $(o_t, a_t, r_t, o_{t+1}, d_{t+1})$. This buffer is consumed by the fact-extraction stage at episode end, then cleared on the next reset. \\
    Per-decision cache $\mathcal{C}$ & dictionary & Memoizes \texttt{propose}, \texttt{simulate}, and \texttt{value} calls within one real action selection. Clear it before every new real decision. A faithful key is \texttt{(call\_type, observation, action if any, history\_tuple)}. \\
    Terminal set $\mathcal{T}$ & set of strings & Optional cache of observations already classified as terminal by the simulator. If a simulated observation is in $\mathcal{T}$, return future value $0$ without further expansion. \\
    \bottomrule
  \end{tabular}
\end{table*}

\paragraph{Reference execution protocol.}
The following ordering matches the core method used throughout the paper:
\begin{enumerate}[leftmargin=1.5em,itemsep=2pt]
    \item \textbf{Episode reset.} Call \texttt{env.reset()}, clear the episode buffer, clear the short-term history, append the initial observation string, and snapshot the current global Fact Memory for the whole episode.
    \item \textbf{Planning state reset.} Before each real action selection, clear the per-decision cache. Keep the episode-level fact snapshot fixed; do not update facts mid-episode.
    \item \textbf{Action proposal.} Query \texttt{ProposeActions} on the current observation, recent history, episode fact snapshot, environment description, and legal action set. Deduplicate exact action matches and keep the first $k_B$ unique actions returned by the model.
    \item \textbf{Single-step simulation.} For each proposed action, call \texttt{SimulateStep}. Extend the local branch history with \texttt{Act: ...} followed by the predicted \texttt{Obs: ...}. If the simulator predicts \texttt{done=True}, define the future value of that child to be $0$.
    \item \textbf{Leaf valuation.} If the remaining search depth is zero, or if an internal node produces no candidate actions, call \texttt{EstimateValue} on that node rather than expanding further.
    \item \textbf{Recursive backup.} Compute action values using
    \[
    Q(o_t, a_i) = r' - \lambda_{\mathrm{step}} + \gamma \hat{V}(o').
    \]
    At each internal node, return the maximum child Q-value. At the root, choose the first action among the argmax set to keep planning deterministic.
    \item \textbf{Real environment transition.} Execute the chosen action in the real environment, append $(o_t, a_t, r_t, o_{t+1}, d_{t+1})$ to the episode buffer, then append the realized \texttt{Act: ...} and \texttt{Obs: ...} strings to the real history deque.
    \item \textbf{End-of-episode fact update.} When the episode terminates, construct a trajectory summary containing the outcome, total reward, and formatted step-by-step transition list. Call \texttt{ExtractFacts} with this summary, the previous Fact Memory, and the environment description. Merge the new facts with the old memory; if compression is enabled, run \texttt{CompressFacts} once on the merged set. The resulting fact memory is then used only from the next episode onward.
\end{enumerate}
If \texttt{ProposeActions} returns no action at the root, a robust reimplementation should execute a fixed legal fallback action (for example, the first legal action in the environment's ordering) rather than leaving the behavior undefined. In our reported runs this failure mode was rare, but the guard is useful for portability.

\begin{table*}[t]
  \centering
  \caption{\textbf{Concrete defaults for the reported runs.} These values pin down the implementation choices that most strongly affect paper-faithful reproduction. Smaller search depths appear only in ablations and appendix-only robustness experiments.}
  \label{tab:reference_defaults}
  \small
  \begin{tabular}{@{}p{0.16\textwidth}p{0.30\textwidth}p{0.48\textwidth}@{}}
    \toprule
    \textbf{Method} & \textbf{Core defaults} & \textbf{Notes} \\
    \midrule
    ReAct & History cap $H_L{=}51$ items; ReAct thought/action temperature $0.3$; max output tokens $8512$ & No persistent memory beyond the short-term history buffer. Uses the shared environment header and action set. \\
    Reflexion & ReAct defaults; lesson buffer length $5$; post-episode lesson synthesis & Lessons are short natural-language summaries appended between episodes. No lookahead search. \\
    ReAct + FEC & ReAct defaults; fact buffer length $200$; fact extraction and compression temperature $0.0$; compression enabled & Uses the same atomic-fact memory as LWM-Planner, but acts greedily without tree search. \\
    LWM-Planner & Fact buffer length $200$; planning depth $d{=}3$ in the main tables; branch factor $k_B{=}4$; planning discount $\gamma{=}0.99$; step penalty magnitude $\lambda_{\mathrm{step}}{=}0.02$; proposal/simulation/value temperatures $0.0$; fact extraction/compression temperature $0.0$; max output tokens $8512$ & Facts are frozen within an episode and refreshed only between episodes. Planning calls share the same environment header and use local branch histories with per-decision memoization. \\
    \bottomrule
  \end{tabular}
\end{table*}

\paragraph{Evaluation defaults.}
Unless otherwise stated, each run uses a fixed interaction budget of $300$ environment steps and reports cumulative return over that budget. Results are averaged over three random seeds in most experiments due to inference-time cost; the main comparison table uses ten seeds where indicated in the caption. Appendix~\ref{app:baseline_parity_text2} specifies the remaining prompt-parity constraints shared across methods.

\newpage
\section{Theory} \label{appendix:theory}

This section provides the theoretical underpinnings of our fact-based reinforcement learning approach, focusing on the performance guarantees of an idealized agent.

\subsection{Theorem: Performance of Idealized Fact-Based Agent (IFBA) with Perfect Abstract Model}

Let $\pi_F$ be the policy derived by the Idealized Fact-Based Agent (IFBA), as defined in Section~\ref{sec:ifba} of the main text. If the ideal fact abstraction $\Psi^*$ establishes an $\epsilon_{\text{sim}}$-approximate bisimulation (Definition~\ref{def:ideal_psi_star}) between the ground MDP $\mathcal{G}$ and the induced abstract MDP $M_{\Psi^*}$, and the planner for $M_{\Psi^*}$ is $\epsilon_{\text{plan}}$-optimal (Definition~\ref{def:ifba_model_planner}), then for any state $s \in \mathcal{S}$:
\begin{equation}
V^*_{\mathcal{G}}(s) - V^{\pi_F}_{\mathcal{G}}(s) \le \frac{2\epsilon_{\text{sim}}}{1-\gamma} + \epsilon_{\text{plan}}
\label{eq:ifba_bound_appendix}
\end{equation}

\subsubsection{Proof}\label{sec:thm_proof}

The value loss of the policy $\pi_F$ in the ground MDP $\mathcal{G}$ can be decomposed. Recall that $\pi_F(s') = \pi^{\circ}_{M_{\Psi^*}}(\Psi^*(s'))$, where $\pi^{\circ}_{M_{\Psi^*}}$ is the $\epsilon_{\text{plan}}$-optimal policy found by the planner operating in the abstract MDP $M_{\Psi^*}$.

We express the total value loss as follows:
\begin{align*}
V^*_{\mathcal{G}}(s) - V^{\pi_F}_{\mathcal{G}}(s) &= \left( V^*_{\mathcal{G}}(s) - V^*_{M_{\Psi^*}}(\Psi^*(s)) \right) \\
&\quad + \left( V^*_{M_{\Psi^*}}(\Psi^*(s)) - V^{\pi^{\circ}_{M_{\Psi^*}}}_{M_{\Psi^*}}(\Psi^*(s)) \right) \\
&\quad + \left( V^{\pi^{\circ}_{M_{\Psi^*}}}_{M_{\Psi^*}}(\Psi^*(s)) - V^{\pi_F}_{\mathcal{G}}(s) \right)
\end{align*}
Let's analyze each term:

\begin{enumerate}
    \item \textbf{Term I: \textcolor{termcolor1}{Abstraction Error on the Optimal Value Function}} \\
    This term, $\left( V^*_{\mathcal{G}}(s) - V^*_{M_{\Psi^*}}(\Psi^*(s)) \right)$, represents the difference between the optimal value function in the ground MDP and the optimal value function in the abstract MDP, mapped back to the ground state via $\Psi^*$.
    By Definition~\ref{def:ideal_psi_star} (specifically, Equation~\eqref{eq:bisim_value_bound} in the main text), the property of $\epsilon_{\text{sim}}$-approximate bisimulation directly bounds this difference:
    $$
    |V^*_{\mathcal{G}}(s) - V^*_{M_{\Psi^*}}(\Psi^*(s))| \le \frac{\epsilon_{\text{sim}}}{1-\gamma}
    $$

    \item \textbf{Term II: \textcolor{termcolor2}{Planning Error in the Abstract MDP}} \\
    This term, $\left( V^*_{M_{\Psi^*}}(\Psi^*(s)) - V^{\pi^{\circ}_{M_{\Psi^*}}}_{M_{\Psi^*}}(\Psi^*(s)) \right)$, quantifies the sub-optimality of the policy $\pi^{\circ}_{M_{\Psi^*}}$ computed by the planner within the abstract model $M_{\Psi^*}$.
    By Definition~\ref{def:ifba_model_planner}, the planner is $\epsilon_{\text{plan}}$-optimal, which means:
    $$
    V^*_{M_{\Psi^*}}(\Psi^*(s)) - V^{\pi^{\circ}_{M_{\Psi^*}}}_{M_{\Psi^*}}(\Psi^*(s)) \le \epsilon_{\text{plan}}
    $$
    Since value functions are non-negative (assuming non-negative rewards or appropriate initialization), this term is bounded by $\epsilon_{\text{plan}}$.

    \item \textbf{Term III: \textcolor{termcolor3}{Abstraction Error on the Planned Policy's Value}} \\
    This term, $\left( V^{\pi^{\circ}_{M_{\Psi^*}}}_{M_{\Psi^*}}(\Psi^*(s)) - V^{\pi_F}_{\mathcal{G}}(s) \right)$, captures the difference between the value of the planned policy $\pi^{\circ}_{M_{\Psi^*}}$ when evaluated in the abstract MDP $M_{\Psi^*}$ versus its value when executed in the ground MDP $\mathcal{G}$ (which is $V^{\pi_F}_{\mathcal{G}}(s)$ by definition of $\pi_F$).
    A key property of $\epsilon_{\text{sim}}$-approximate bisimulation is that it not only bounds the difference in optimal value functions but also the difference in value functions for any fixed policy that respects the abstraction $\Psi^*$. Since $\pi_F(s') = \pi^{\circ}_{M_{\Psi^*}}(\Psi^*(s'))$, it respects the abstraction. Therefore, we have \citep{Ravindran2004thesis}:
    $$
    |V^{\pi^{\circ}_{M_{\Psi^*}}}_{M_{\Psi^*}}(\Psi^*(s)) - V^{\pi_F}_{\mathcal{G}}(s)| \le \frac{\epsilon_{\text{sim}}}{1-\gamma}
    $$
\end{enumerate}

\textbf{Combining the Bounds:}

Using the triangle inequality ($X - Z = (X - Y) + (Y - Z) \implies |X - Z| \le |X - Y| + |Y - Z|$), we can sum the absolute bounds of these terms. More directly, since Term II is already an upper bound on the difference (not an absolute value), we have:
\begin{align*}
V^*_{\mathcal{G}}(s) - V^{\pi_F}_{\mathcal{G}}(s) &\le \left| V^*_{\mathcal{G}}(s) - V^*_{M_{\Psi^*}}(\Psi^*(s)) \right| \\
&\quad + \left( V^*_{M_{\Psi^*}}(\Psi^*(s)) - V^{\pi^{\circ}_{M_{\Psi^*}}}_{M_{\Psi^*}}(\Psi^*(s)) \right) \\
&\quad + \left| V^{\pi^{\circ}_{M_{\Psi^*}}}_{M_{\Psi^*}}(\Psi^*(s)) - V^{\pi_F}_{\mathcal{G}}(s) \right| \\
&\le \frac{\epsilon_{\text{sim}}}{1-\gamma} + \epsilon_{\text{plan}} + \frac{\epsilon_{\text{sim}}}{1-\gamma} \\
&= \frac{2\epsilon_{\text{sim}}}{1-\gamma} + \epsilon_{\text{plan}}
\end{align*}

This completes the proof, establishing the performance bound for the Idealized Fact-Based Agent operating with a perfect abstract model $M_{\Psi^*}$ derived from an $\epsilon_{\text{sim}}$-approximate bisimulation $\Psi^*$, and an $\epsilon_{\text{plan}}$-optimal planner. \qed

\subsection{Further Theoretical Considerations}
\label{app:further_theory}

The theoretical framework presented in Section~\ref{sec:theoretical_framework} and the proof in Appendix~\ref{sec:thm_proof} rely on the quality of the fact-based abstraction $\Psi$ and the learned abstract model $\tilde{M}_{\Psi}$. Here, we elaborate on two guiding principles relevant to achieving a good abstraction and, consequently, robust agent performance: bisimulation for state aggregation and the Information Bottleneck principle for fact minimality and relevance.

\subsubsection{Approximate State Abstraction and Bisimulation}
The concept of $\epsilon_{\text{sim}}$-approximate bisimulation (Definition~\ref{def:ideal_psi_star}) provides a formal grounding for why a good fact-based abstraction can lead to provably near-optimal policies \citep{ferns2004metrics}. A bisimulation groups states of the ground MDP $\mathcal{G}$ that are \emph{behaviorally equivalent}. In an exact bisimulation, states $s_1, s_2$ are in the same abstract state $z = \Psi(s_1) = \Psi(s_2)$ if, for any action $a \in \mathcal{A}$:
\begin{itemize}
    \item They yield the same immediate reward: $R(s_1, a) = R(s_2, a)$.
    \item They transition to the same distribution over next abstract states: $\sum_{s'_1 \in \mathcal{S}_{z'}} \mathcal{T}(s'_1|s_1,a) = \sum_{s'_2 \in \mathcal{S}_{z'}} \mathcal{T}(s'_2|s_2,a)$ for all $z' \in \mathcal{Z_F}$.
\end{itemize}
The $\epsilon_{\text{sim}}$ term in an approximate bisimulation relaxes these conditions, allowing for small differences in one-step rewards and transition probabilities for states mapped to the same abstract state $z$. The crucial implication, as shown in Equation~\eqref{eq:bisim_value_bound}, is that the value function $V^*_{M_{\Psi^*}}$ in the abstract MDP $M_{\Psi^*}$ (induced by an $\epsilon_{\text{sim}}$-approximate bisimulation $\Psi^*$) is close to the optimal value function $V^*_{\mathcal{G}}$ in the ground MDP, with the error bounded by $\frac{\epsilon_{\text{sim}}}{1-\gamma}$.

In LWM-Planner, the atomic fact set $F_s$ generated by the Fact Extractor $\ExtractorLLM$ for a ground state $s$ forms the basis of the abstract state $z = \Psi(s) = (o, F_s)$ (where $o$ is the direct observation). The goal of learning ``critical'' and ``minimal'' facts is to construct an abstraction $\Psi$ that minimizes $\epsilon_{\text{sim}}$. That is, the facts should capture enough information to ensure that states $s_1, s_2$ mapped to the same $(o, F_s)$ have similar optimal values and similar optimal actions. If the extracted facts fail to distinguish between ground states that are behaviorally different (e.g., one leads to a high reward and another to a low reward with the same action), $\epsilon_{\text{sim}}$ will be large, and the performance guarantee in Equation~\eqref{eq:learned_model_bound_final} degrades. The online learning of facts is a process of refining $\Psi$ to better approximate a bisimulation relevant to the task.

\subsubsection{Information Bottleneck Principle and Fact Relevance}
The desire for ``minimal, yet impactful, units of knowledge'' (atomic facts) aligns with the \textbf{Information Bottleneck (IB)} principle \citep{Tishby2000, Alemi2017vib}. The IB principle seeks to find a compressed representation (bottleneck) $Z$ of a source variable $X$ that is maximally informative about a target relevance variable $Y$, while minimizing the mutual information $I(X;Z)$ (i.e., $Z$ should be a minimal sufficient statistic of $X$ for predicting $Y$).

In our context:
\begin{itemize}
    \item The source variable $X$ is the ground state $s_t$ (or the full trajectory history leading to $s_t$).
    \item The compressed representation $Z$ is the abstract state $z_t = \Psi(s_t)$, primarily defined by the set of atomic facts $F_{s_t}$ derived from $s_t$ or its history.
    \item The target relevance variable $Y$ can be conceptualized as the optimal action-value function $Q^*_{\mathcal{G}}(s_t, a)$, the optimal value function $V^*_{\mathcal{G}}(s_t)$, or even future rewards and observations.
\end{itemize}
The LWM-Planner's fact extraction process aims to learn a $\Psi$ (instantiated by $\ExtractorLLM$) that distills $F_{s_t}$ such that it is highly informative about $V^*_{\mathcal{G}}(s_t)$ (to minimize $\epsilon_{\text{sim}}$) and future transitions/rewards (to minimize $\delta_{\text{model}}$ when $g_{\phi}$ uses these facts). The ``atomicity'' and ``minimality'' criteria for facts directly serve the goal of compressing $s_t$ into $F_{s_t}$ without losing task-critical information. A smaller, more relevant set of facts:
\begin{itemize}
    \item \textbf{Reduces computational burden}: Shorter prompts for the LLM components.
    \item \textbf{Focuses LLM reasoning}: Prevents dilution of critical information within the LLM's limited context window, allowing the LLM to better utilize its pre-trained knowledge by grounding it on the most salient task-specific evidence.
    \item \textbf{Improves sample efficiency for model learning}: A smaller abstract state space $|\mathcal{Z_F}|$ generally means that learning the dynamics $\tilde{\mathcal{T}}_{\Psi}$ and rewards $\tilde{R}_{\Psi}$ of the abstract MDP $\tilde{M}_{\Psi}$ (implicitly done by $g_{\phi}$) requires fewer samples (interactions) to achieve a low $\delta_{\text{model}}$ \citep{Strehl2009}.
\end{itemize}
The optional fact compression step in LWM-Planner explicitly attempts to enforce this minimality by removing redundant or less informative facts, further aligning with the IB objective of finding a maximally compressed yet sufficient representation. The quality of this compression directly impacts the complexity and effectiveness of the learned abstract model and, consequently, the overall planning performance.

By striving for fact-based abstractions that approximate bisimulations and adhere to information bottleneck principles, LWM-Planner aims to construct an internal representation that is both robust for planning and efficient for learning. The interplay between the quality of facts ($\epsilon_{\text{sim}}$) and the LLM's ability to simulate and value using these facts ($\delta_{\text{model}}$) is central to achieving near-optimal performance, as formalized in Equation~\eqref{eq:learned_model_bound_final}.

\subsubsection{Future Theoretical Directions and Open Questions}
\label{app:future_theoretical_directions}

While the current theoretical framework provides initial motivation, several avenues for future theoretical work could further solidify and extend the understanding of LWM-Planner and similar fact-based LLM agents.

\paragraph{Causality in Fact Extraction}
A significant direction is the integration of \textbf{causal reasoning} into the fact extraction process \citep{Pearl_2009}. Currently, atomic facts are evaluated in terms of correlation\footnote{LLMs probably implicitly already perform some form of causal reasoning}, capturing observed relationships (e.g., ``action Z leads\_to\_failure\_condition''). However, facts that represent causal relationships (e.g., ``action Z \emph{causes} failure condition \emph{if} precondition P holds'') would offer more robust and generalizable knowledge.
\begin{itemize}
    \item \textbf{Improved Generalization}: Causal facts are more likely to hold true under slight variations in the environment or task, potentially leading to a smaller $\delta_{\text{model}}$ when the agent encounters novel situations that share underlying causal structures.
    \item \textbf{Intervention-based Learning}: Future agents could be designed to perform specific interventions (exploratory actions) aimed at discovering causal links, rather than passively observing correlations. This could lead to more sample-efficient learning of highly impactful facts.
    \item \textbf{Counterfactual Reasoning}: A fact base enriched with causal understanding could allow the Planner LLM ($g_{\phi}$) to engage in more sophisticated counterfactual reasoning during lookahead search (e.g., ``what would have happened if I had chosen action B instead of A, given these causal rules?'').
\end{itemize}
Developing methods for $\ExtractorLLM$ to reliably infer or validate causal statements from observational and interventional data within interactive environments is a challenging but promising research area.

\paragraph{Formal Analysis of LLM-driven Components}
The current framework treats the LLM components ($g_{\phi}^{\text{propose}}$, $g_{\phi}^{\text{simulate}}$, $g_{\phi}^{\text{value}}$, and $\Psi_{\text{LLM}}$) largely as oracles with certain performance characteristics (e.g., LLM simulation accuracy contributing to $\delta_{\text{model}}$). Future work could delve into:
\begin{itemize}
    \item \textbf{Characterizing LLM Errors}: More formally characterizing the types of errors LLMs make in simulation and value estimation, and how these errors propagate through the lookahead search to affect $\epsilon_{\text{plan}}$.
    \item \textbf{Impact of Prompt Engineering}: Theoretically analyzing the sensitivity of $\epsilon_{\text{sim}}$ and $\delta_{\text{model}}$ to the quality and structure of prompts, including the presentation of atomic facts.
    \item \textbf{Convergence of Fact Memory}: Investigating conditions under which the learned Fact Memory $\FactMemory_t$ converges to a ``sufficient'' or ``minimal'' set of facts that guarantees a certain level of performance (e.g., bounding $\epsilon_{\text{sim}}$ below a desired threshold).
\end{itemize}

\paragraph{Fact-Based Abstractions in Partially Observable MDPs (POMDPs)}
The current theoretical analysis assumes an underlying MDP where states $s_t$ are fully observable or can be derived into a sufficient structured representation. Many real-world scenarios are better modeled as POMDPs \citep{KAELBLING199899}, where the agent receives observations $o_t$ that are incomplete or noisy manifestations of the true underlying state $s_t$.
\begin{itemize}
    \item \textbf{Belief State Abstraction}: Future work could explore how atomic facts can be used to form abstractions not directly over $s_t$, but over belief states $b(s_t) = P(s_t | \text{history})$. Atomic facts could represent properties of the environment that are inferred to be true with high probability based on the observation history.
    \item \textbf{Information-Gathering Facts}: The agent might learn facts not just about the environment's dynamics, but about which actions are most informative for reducing uncertainty about critical, unobserved aspects of the state, guiding exploration more effectively.
\end{itemize}

Addressing these theoretical questions will be crucial for advancing the capabilities and understanding of LLM agents that learn and plan from interaction by building and reasoning over symbolic knowledge representations like atomic facts.

\newpage\section{Benchmark Environments Details}
\label{sec:BenchmarkEnvironmentsDetails}

We evaluate our LWM-Planner agent and baseline methods on three distinct, procedurally generated, text-based environments. Each environment is designed to test different aspects of an agent's learning and planning capabilities, ranging from grid-world navigation with sparse rewards to complex instruction following and multi-step crafting tasks.

\subsection{TextFrozenLake}

This environment is a procedurally generated text-based version of the classic FrozenLake problem \citep{brockman2016openai}.

\begin{itemize}
    \item \textbf{Objective}: The agent must navigate from a starting position (S) at coordinates (0,0) to a goal position (G) at ($N-1, N-1$) on an $N \times N$ grid. The grid also contains ice surfaces (.) and holes (H). Reaching the goal yields a reward of +1.0, falling into a hole yields -1.0, and all other steps yield 0.0. An episode terminates upon reaching the goal, falling into a hole, or exceeding a maximum step limit.
    \item \textbf{Procedural Generation}:
    \begin{itemize}
        \item The grid size $N$ (e.g., $4 \times 4$, $6 \times 6$, $8 \times 8$) and hole density $h$ are configurable parameters at initialization.
        \item A key feature is the guarantee of at least one solvable path from start to goal. This is achieved by first constructing a "Manhattan corridor" or a zig-zag safe path near the diagonal to connect (0,0) and ($N-1, N-1$). Holes are then sampled on the remaining cells based on the \texttt{hole\_density} parameter.
        \item The environment can be seeded for deterministic board generation and agent starting position.
    \end{itemize}
    \item \textbf{Observation Space}: The agent receives a textual observation describing its current state, e.g., "You are at (0, 1) on ice.". This provides local information (current coordinates and terrain type of the square it stands on) without revealing the global map layout.
    \item \textbf{Action Space}: The agent has four discrete actions: "up", "down", "left", "right". Actions that would move the agent off the grid boundaries result in the agent remaining in its current position but on the edge cell.
    \item \textbf{Rewards}: As described, +1.0 for goal, -1.0 for a hole, 0.0 otherwise.
    \item \textbf{Episode Termination}: An episode ends if the agent reaches the goal (G), falls into a hole (H), or if the \texttt{\_step\_count} reaches \texttt{max\_steps}. The \texttt{max\_steps} is set to $8 \times (N-1)$, which is four times the optimal path length in an empty grid.
    \item \textbf{\texttt{env\_description}}: The environment provides a detailed textual description string that includes the grid size, start/goal locations, reward structure, maximum steps, and hole density, explicitly stating that a path to the goal is guaranteed.
\end{itemize}

\subsection{ALFWorld}

We utilize the standard ALFWorld (Action Learning From World) benchmark \citep{shridhar2020alfworld}, which involves text-based agents performing tasks in simulated household environments based on the ALFRED dataset \citep{shridhar2020alfred}. The \texttt{AlfWorldEnv} class in our codebase acts as a thin adapter around the official ALFWorld text environment.

\begin{itemize}
    \item \textbf{Objective}: The agent is given a high-level natural language instruction (e.g., "put some spraybottle on toilet") and must navigate the environment, interact with objects and receptacles, and manipulate objects to satisfy the goal condition.
    \item \textbf{Environment Structure}: ALFWorld environments are simulated indoor scenes (kitchens, bedrooms, bathrooms) containing various receptacles (e.g., cabinets, drawers, countertops, sinks) and objects (e.g., spraybottle, bowl, desklamp). Objects can be picked up, placed in/on receptacles, and sometimes manipulated (e.g., heated, cleaned, sliced, though these more complex interactions are often simplified or yield generic feedback in some wrapper implementations).
    \item \textbf{Observation Space}: The agent receives rich textual observations describing its current location, visible objects and receptacles, its inventory, and feedback from its previous action. The initial observation also includes the specific task goal.
    \item \textbf{Action Space}: The environment defines a set of canonical action templates such as "look", "inventory", "go to (receptacle)", "open (receptacle)", "close (receptacle)", "take (object) from (receptacle)", "move (object) to (receptacle)", "examine (something)", "use (object)", etc..
    \item \textbf{Rewards}: A reward of +1.0 is typically given upon successful completion of the task goal. All other steps yield a reward of 0.0.
    \item \textbf{Episode Termination}: An episode ends if the agent successfully completes the task or if the maximum number of steps (\texttt{max\_steps}, typically configured from a YAML file, e.g., \texttt{base\_config.yaml}) is reached.
    \item \textbf{Task Variability}: ALFWorld offers a diverse set of tasks (identified by \texttt{task\_id}) categorized into different types (e.g., pick \& place, heat \& place, clean \& place). For our experiments, we randomly sample tasks from the "eval\_out\_of\_distribution" split, as specified in the \texttt{AlfWorldEnv} constructor. Each \texttt{task\_id} corresponds to a unique environment configuration and goal.
    \item \textbf{\texttt{env\_description}}: The \texttt{AlfWorldEnv} provides a comprehensive \texttt{env\_description} string that outlines the nature of the environment, receptacles, objects, task structure, reward system, maximum steps, the full list of admissible actions, and examples of interaction, along with advice for the agent.
\end{itemize}

For some ablation studies or simpler scenarios, a minimal version, \texttt{AlfMiniEnv}, is also available. It features a single, deterministically generated room with a fixed set of receptacle and object types (e.g., "drawer", "shelf", "vase", "keychain") and a canonical goal like "put some vase in safe 1". This version allows for more controlled experimentation by simplifying the state and action space while retaining the core object interaction mechanics. It also features deterministic resets to a blueprint state or a newly seeded state.

\subsection{CrafterMini}

CrafterMini is a procedurally generated, text-only, miniaturized version of the Crafter environment \citep{hafner2021benchmarking}, designed to test planning for resource gathering and multi-step crafting.

\begin{itemize}
    \item \textbf{Objective}: The primary goal is to craft an "iron\_pickaxe". This requires a sequence of sub-goals: collecting raw materials (wood, stone, iron) and crafting intermediate tools (wood\_pickaxe, stone\_pickaxe).
    \item \textbf{Environment Structure}: The world is a $N \times N$ grid (default $5 \times 5$) with a toroidal (wrap-around) topology. Each tile can be grass, tree, stone, iron, or water.
    \item \textbf{Procedural Generation}:
    \begin{itemize}
        \item The grid size $N$ and \texttt{max\_steps} are configurable. The world is seeded for deterministic generation.
        \item The grid is randomly populated with tiles, ensuring that at least one of each crucial resource (tree, stone, iron) is present, making the game solvable.
        \item The \texttt{reset()} method can either restore an initial "blueprint" of the world or generate a new world if a new seed is provided.
    \end{itemize}
    \item \textbf{Observation Space}: The observation is a textual string describing the agent's current tile type and coordinates, the terrain in the four cardinal directions, the agent's inventory (e.g., "wood=3, stone=1"), and a list of tools already crafted.
    \item \textbf{Action Space}: Actions are represented by integers with corresponding names:
    \begin{itemize}
        \item 0-3: Movement (north, south, east, west).
        \item 4: "collect" - Gathers a resource from the current tile if it's a resource tile (tree, stone, iron). Collecting turns the tile to grass.
        \item 5: "craft\_wood\_pickaxe" (requires 3 wood).
        \item 6: "craft\_stone\_pickaxe" (requires 1 wood, 3 stone).
        \item 7: "craft\_iron\_pickaxe" (requires 1 stone\_pickaxe, 3 iron).
    \end{itemize}
    Crafting actions are only considered available if the recipe can be satisfied by the current inventory and already crafted tools (e.g., a stone\_pickaxe is consumed to make an iron\_pickaxe). The \texttt{RobustCrafterMiniEnv} variant can also parse textual synonyms for these actions.
    \item \textbf{Rewards}:
    \begin{itemize}
        \item Each step incurs a -1 reward (step cost).
        \item Crafting a wood\_pickaxe gives +10 reward.
        \item Crafting a stone\_pickaxe gives +20 reward.
        \item Crafting an iron\_pickaxe gives +50 reward.
    \end{itemize}
    \item \textbf{Episode Termination}: The episode ends immediately if an "iron\_pickaxe" is successfully crafted or if the number of steps reaches \texttt{max\_steps} (default $4 \times N^2$).
    \item \textbf{\texttt{env\_description}}: A detailed textual description outlines the grid, observation format, action list with integer mappings, crafting recipes, rewards, and termination conditions.
\end{itemize}

These three environments provide a diverse set of challenges for evaluating the LWM-Planner's ability to learn from textual interactions, build effective world models through atomic facts, and plan over extended horizons.

\newpage\section{Benchmark Method Implementation Details}
\label{sec:BenchmarkMethodImplementationDetails}

This section provides a detailed overview of the high-level logic for each benchmark method evaluated in our experiments. All LLM-based agents (LWM-Planner, ReAct, Reflexion, and ReAct + FEC) utilize a frozen LLM (model \texttt{gpt-4o}). LLM interactions are performed via API calls using a common internal utility that supports structured function calling; no LLM weights are updated during experiments. The default temperature for LLM calls in planning components (simulation, value estimation) and fact processing is 0.0. For ReAct-style thought generation, a temperature of 0.3 is typically used. The maximum token output for LLM responses is configured to 8512 tokens. A compact implementation-facing specification of the core runtime contract and reported defaults is given in Appendix~\ref{app:reference_impl_notes}.

\subsection{Random Agent}
The \textbf{Random} agent serves as a basic non-learning baseline.
\begin{itemize}
    \item \textbf{Policy}: At each step, the agent selects an action uniformly at random from the set of allowed actions provided for the specific environment.
    \item \textbf{State}: It maintains a short-term memory buffer of observation-action pairs for logging consistency, though this history does not influence its action selection.
\end{itemize}

\subsection{ReAct Agent}
The \textbf{ReAct} agent implements the Reason-Act prompting paradigm \citep{yao2023react}.
\begin{itemize}
    \item \textbf{Core Mechanism}: The agent prompts an LLM to first generate a ``Thought'' (internal reasoning) and then an ``Action'' to take in the environment.
    \item \textbf{LLM Interaction}:
    \begin{itemize}
        \item A prompt is constructed using a template specific to the ReAct style. This template incorporates a description of the environment, the current observation, the recent interaction history (formatted as a sequence of observations and actions), and the list of allowed actions.
        \item The LLM is expected to provide its output in a structured format that distinguishes the thought process from the chosen action.
    \end{itemize}
    \item \textbf{State}: The agent maintains a short-term memory buffer (a deque of observation-action strings) of a configurable length (e.g., 51 interactions). It also stores the most recent thought generated by the LLM.
    \item \textbf{Hyperparameters}: Key parameters include the length of the interaction history, LLM model, temperature, and maximum token limits.
\end{itemize}

\subsection{Reflexion Agent}
The \textbf{Reflexion} agent extends the ReAct agent by incorporating a self-reflection mechanism to learn from past experiences \citep{shinn2023reflexion}.
\begin{itemize}
    \item \textbf{Core Mechanism}: In addition to ReAct's thought-action cycle, after each episode, the Reflexion agent analyzes its trajectory to generate a textual "lesson".
    \item \textbf{Lesson Generation and Usage}:
    \begin{itemize}
        \item Lessons are stored in a memory buffer (a deque) of a configurable maximum length (e.g., 5 lessons).
        \item The ReAct prompt template is augmented to include these learned lessons, providing additional context for future decisions.
        \item A post-episode reflection process involves constructing a summary of the completed trajectory (including observations, actions, rewards, and overall outcome) and prompting an LLM to generate a concise, actionable lesson (typically $\le 20$ words and prefixed accordingly).
    \end{itemize}
    \item \textbf{State}: In addition to the ReAct state, it maintains the buffer of lessons, a record of the current episode's trajectory (observations, actions, rewards, next observations), and the cumulative reward for the current episode.
    \item \textbf{Hyperparameters}: Includes ReAct parameters plus the lesson buffer length and a reward threshold to determine episode success for reflection purposes (e.g., 0.99).
\end{itemize}

\subsection{ReAct + FEC Agent (Ablation)}
This agent is an ablation of our full LWM-Planner. It combines the ReAct decision-making process with the Fact Extraction and Compression (FEC) mechanism, but without lookahead search.
\begin{itemize}
    \item \textbf{Fact Mechanism}:
    \begin{itemize}
        \item Maintains a memory buffer (a deque) of atomic facts with a configurable maximum length (e.g., 200 facts).
        \item The ReAct prompt template is augmented to include these atomic facts.
        \item \textbf{Fact Extraction}: After each episode, a dedicated LLM-driven process analyzes the trajectory summary, environment description, and existing facts to identify minimal new atomic facts critical for improving predictions. The LLM is guided to output these new facts in a structured manner.
        \item \textbf{Fact Compression}: If enabled, another LLM-driven process reviews the complete set of current facts (newly extracted plus existing) along with the environment description. It aims to produce a concise, refined set of facts by removing redundancies or information trivially inferable from the environment description, again using a structured output format.
    \end{itemize}
    \item \textbf{Decision-Making}: Employs the standard ReAct thought-action cycle, but the LLM's reasoning is informed by the dynamically updated set of atomic facts included in its prompt.
    \item \textbf{Hyperparameters}: Includes ReAct parameters plus the fact buffer length and a flag to enable/disable fact compression.
\end{itemize}

\subsection{LWM-Planner (Our Method)}
The LWM-Planner is our proposed agent that integrates online atomic fact learning with a recursive lookahead search, where LLMs serve as key planning components.
\begin{itemize}
    \item \textbf{Base Functionality}: It incorporates the same fact extraction and compression mechanisms (Fact Extractor $\Psi_{\text{LLM}}$) as the ReAct + FEC agent. These facts are stored in a dynamically updated memory buffer (a deque with a capacity of, e.g., 200 facts) and are used to augment the reasoning of all LLM components.
    \item \textbf{Core Planning Mechanism}: Instead of a direct ReAct step, LWM-Planner performs a depth-limited recursive lookahead search to select actions (Planner $g_{\phi}$).
    \begin{itemize}
        \item \textbf{Action Proposal ($g_{\phi}^{\text{propose}}$)}: At each node in the search, an LLM module proposes a set of plausible actions (up to a configurable branching factor, e.g., 4). This proposal is conditioned on the current (potentially simulated) observation, the history of interactions within the simulation, and the accumulated atomic facts.
        \item \textbf{Latent World Model Simulation ($g_{\phi}^{\text{simulate}}$)}: For each proposed action, an LLM module predicts the next (latent) observation, immediate reward, and termination status. This simulation is conditioned on the current state, the action being simulated, the simulation history, and the atomic facts. This LLM interaction operates with a temperature of 0.0 for deterministic outcomes.
        \item \textbf{Value Estimation ($g_{\phi}^{\text{value}}$)}: At the leaves of the search tree (determined by a configurable search depth, e.g., 3, or upon reaching a terminal state), an LLM module estimates the discounted future cumulative reward (value) from that state. This estimation is also conditioned on the state's observation, simulation history, and atomic facts, and uses a temperature of 0.0.
    \end{itemize}
    \item \textbf{Q-Value Computation}: The Q-value for an action $a_i$ from observation $o_t$ is computed as $Q(o_t, a_i) = r' - \lambda_{\text{step}} + \gamma \hat{V}(o')$, where $r'$ and $o'$ are from the simulation, $\lambda_{\text{step}}$ is a step penalty (e.g., -0.01), and $\gamma$ is the discount factor (e.g., 0.99). $\hat{V}(o')$ is the estimated value of the next state, derived either from further recursion or direct estimation at a leaf node.
    \item \textbf{State and Caching}: The agent maintains a short-term interaction history (e.g., last 51 interactions) and the set of learned atomic facts. Within a single planning phase (for one action selection), results of LLM calls for action proposal, simulation, and value estimation are memoized to avoid redundant computations. A set of known terminal observations is also maintained across steps.
    \item \textbf{Hyperparameters}: Inherits fact-related parameters. Key planning parameters include search depth, branching factor, discount factor for planning, and step penalty. Asynchronous execution of LLM calls for different search branches is supported to improve computational efficiency. The MCTS-based extension of this agent introduces further parameters like the number of simulations and an exploration constant (UCT).
\end{itemize}

All LLM-based agents are initialized with a textual description of the environment and the set of allowed actions. Their prompts are dynamically constructed to include the current observation, relevant interaction history, and any learned knowledge (lessons or facts) appropriate for the specific agent architecture.

\newpage\section{Evaluation Details}
\label{sec:EvaluationDetails}

The experimental evaluation of our proposed LWM-Planner and baseline methods follows a structured procedure to ensure fair comparison and robust assessment of performance. Key aspects of our evaluation protocol are detailed below:

\begin{itemize}
    \item \textbf{Run Duration}: Each agent method is run for a total of 300 environment steps per evaluation trial, unless specified otherwise in a particular experiment. Within these 300 steps, an agent may complete multiple episodes depending on task complexity and its efficiency.

    \item \textbf{Metrics Tracked}: We focus on the following quantitative metrics to assess agent performance:
    \begin{itemize}
        \item \textbf{Cumulative Return}: This is the primary metric and is defined as the sum of all rewards obtained by the agent over the entire 300-step run. It reflects the agent's overall ability to accumulate reward within a fixed interaction budget.
        \item \textbf{Steps per Success}: For environments that have a clear binary success condition (e.g., reaching the goal tile in TextFrozenLake, or successfully completing the assigned task in ALFWorld), we record the number of environment steps taken within an episode to achieve that success. This metric is typically reported as an average over all successful episodes completed by the agent during its run. If an agent fails to achieve success in an episode or across the entire 300-step run, it may not contribute to this average, or its contribution might be noted as not applicable (e.g., marked as '--' in results tables).
    \end{itemize}

    \item \textbf{Replication and Statistical Significance}: To account for inherent stochasticity in agent learning (if applicable) and environment generation (for procedurally generated tasks), results for each method on each environment are averaged over multiple independent runs, each initialized with a different random seed. The number of seeds is typically 3 or 10, as specified in the respective table captions in Section~\ref{sec:experiments}. We report the mean of the metrics and the 95\% confidence interval (CI) to provide a measure of statistical significance and variability.

    \item \textbf{Final Success Rate}: For ALFWorld we additionally report the percentage of tasks solved at least once (FSR). This measure is complementary to cumulative return and enables direct comparison to prior work that reports success rates only.

    \item \textbf{Computational Cost}: We record the total number of tokens processed, the average latency per environment step, and an estimated monetary cost assuming the public pricing of the deployed API; detailed figures are provided in Appendix~\ref{app:cost_analysis}.
    \item \textbf{Type of Compute Used}: All experiments were run on a single workstation equipped with an Intel Core~i9 CPU and an NVIDIA RTX~3090 GPU (24~GB VRAM), together with hosted LLM API calls configured as detailed in Section~\ref{sec:EvaluationDetails}.
\end{itemize}

\newpage\section{Case Study: LWM-Planner on TextFrozenLake (4x4)}
\label{sec:factextractedfrozenlakecasestudy}

This case study details the learning process of the LWM-Planner agent in a procedurally generated $4 \times 4$ TextFrozenLake environment. The specific instance, corresponding to \texttt{grid\_4\_h\_9\_s\_0} from our experiments (with a 90\% chance any non-start/goal tile is a hole, and solvability ensured), features a high density of holes. Holes (H) terminate the episode with a negative reward. The agent's objective is to navigate from the starting position (S) at \texttt{(0,0)} to the goal (G) at \texttt{(3,3)}. Ice tiles (.) are safe to traverse.

The initial state of the grid is (S represents the Start/Agent):
\begin{verbatim}
S . H H
H . . H
H H . .
H H H G
\end{verbatim}
The agent never observes this map, this is just for illustration. Instead the agent can only learn where the holes are by falling down each hole at least once, motivating the need for a persistent memory.

\subsection*{Initial Exploration and Learning from Failures (Episodes 0-3)}

In its initial episodes, LWM-Planner explores the environment and primarily learns by encountering hazards. The atomic fact extraction mechanism is crucial during this phase for building a rudimentary map of dangers.

\begin{itemize}
    \item \textbf{Episode 0:} The agent's first action is to move `down` from \texttt{(0,0)}.
    \begin{itemize}
        \item Trajectory Snippet: \texttt{Obs: You are at (0,0) on start. | Act: down | R: -1.0 | Next: You are at (1,0) on hole.}
        \item Outcome: FAILURE.
        \item \textbf{Fact Extracted:} \texttt{(1,0) is a hole.}
        \item \textbf{Accumulated Facts (after Ep. 0):} \texttt{['(1,0) is a hole.']}
    \end{itemize}
    This first fact immediately informs the agent about a critical environmental feature.

    \item \textbf{Episode 1:} Aware of the hole at \texttt{(1,0)}, the agent attempts a different path. It moves `right` from \texttt{(0,0)} to \texttt{(0,1)}, then `down` to \texttt{(1,1)}, and then `down` again.
    \begin{itemize}
        \item Trajectory Snippet: \texttt{...Obs: You are at (1,1) on ice. | Act: down | R: -1.0 | Next: You are at (2,1) on hole.}
        \item Outcome: FAILURE.
        \item \textbf{Fact Extracted:} \texttt{(2,1) is a hole.}
        \item \textbf{Accumulated Facts (after Ep. 1):} \texttt{['(2,1) is a hole.', '(1,0) is a hole.']}
    \end{itemize}

    \item \textbf{Episode 2:} The agent continues to explore. From \texttt{(0,1)}, it moves `right`.
    \begin{itemize}
        \item Trajectory Snippet: \texttt{...Obs: You are at (0,1) on ice. | Act: right | R: -1.0 | Next: You are at (0,2) on hole.}
        \item Outcome: FAILURE.
        \item \textbf{Fact Extracted:} \texttt{(0,2) is a hole.}
        \item \textbf{Accumulated Facts (after Ep. 2):} \texttt{['(0,2) is a hole.', '(2,1) is a hole.', '(1,0) is a hole.']}
    \end{itemize}

    \item \textbf{Episode 3:} Another failed attempt reveals another hole. From \texttt{(1,2)}, it moves `right`.
    \begin{itemize}
        \item Trajectory Snippet: \texttt{...Obs: You are at (1,2) on ice. | Act: right | R: -1.0 | Next: You are at (1,3) on hole.}
        \item Outcome: FAILURE.
        \item \textbf{Fact Extracted:} \texttt{(1,3) is a hole.}
        \item \textbf{Accumulated Facts (after Ep. 3):} \texttt{['(1,3) is a hole.', '(0,2) is a hole.', '(2,1) is a hole.', '(1,0) is a hole.']}
    \end{itemize}
\end{itemize}
After these initial four failures, the agent has learned the locations of four distinct holes. This knowledge is critical for subsequent planning using lookahead search, as these facts help the LLM-based simulator predict negative outcomes.

\subsection*{First Success and Learning the Safe Path (Episode 4)}

Equipped with knowledge of several hazards, LWM-Planner's lookahead search can now better evaluate potential paths, biasing away from known holes.

\begin{itemize}
    \item \textbf{Episode 4:} The agent successfully navigates to the goal.
    \begin{itemize}
        \item Full Trajectory: \texttt{(0,0) -> right -> (0,1) -> down -> (1,1) -> right -> (1,2) -> down -> (2,2) -> right -> (2,3) -> down -> (3,3)}
        \item Steps: 6 (Optimal for this grid)
        \item Outcome: SUCCESS (Reward: +1.0)
        \item \textbf{Facts Extracted:} A set of facts confirming the nature of the traversed safe tiles and the goal location:
        \begin{itemize}
            \item \texttt{(0,1) is ice.}
            \item \texttt{(1,1) is ice.}
            \item \texttt{(1,2) is ice.}
            \item \texttt{(2,2) is ice.}
            \item \texttt{(2,3) is ice.}
            \item \texttt{(3,3) is the goal.}
        \end{itemize}
        \item \textbf{Accumulated Facts (Snapshot after Ep. 4 includes):} \texttt{['(0,1) is ice.', '(1,1) is ice.', ..., '(3,3) is the goal.', '(1,3) is a hole.', '(0,2) is a hole.', '(2,1) is a hole.', '(1,0) is a hole.']}
    \end{itemize}
\end{itemize}
This successful episode significantly expands the agent's knowledge base, not just with more hazards, but with positive confirmation of safe (ice) tiles and the goal's location. This richer set of facts allows the LLM-driven world model and value estimator to make more accurate predictions during lookahead.

\subsection*{Refinement and Consistent Optimal Performance (Episodes 5 onwards)}

Even after the first success, the agent continues to refine its understanding of the environment.

\begin{itemize}
    \item \textbf{Episode 5:} The agent explores an alternative move from a state on the previously successful path (\texttt{(2,2)}) and encounters another hole.
    \begin{itemize}
        \item Trajectory Snippet: \texttt{...Obs: You are at (2,2) on ice. | Act: down | R: -1.0 | Next: You are at (3,2) on hole.}
        \item Outcome: FAILURE.
        \item \textbf{Fact Extracted:} \texttt{(3,2) is a hole.}
        \item This further completes the agent's map of hazards, particularly those adjacent to the known safe path.
    \end{itemize}

    \item \textbf{Episode 6:} The agent again reaches the goal in 6 steps, following the optimal path.
    \begin{itemize}
        \item Outcome: SUCCESS.
        \item \textbf{Facts Extracted:} \texttt{['(0,0) is the start.', '(0,0) is ice.']} (Identifying properties of the start tile based on the successful trajectory.)
    \end{itemize}

    \item \textbf{Episode 7:} The agent achieves another 6-step success. The LLM's reflection process identifies additional facts based on the episode's context and existing knowledge.
    \begin{itemize}
        \item Outcome: SUCCESS.
        \item \textbf{New Facts Extracted Include:} \texttt{'(0,3) is a hole.'}, \texttt{'(3,0) is a hole.'}, \texttt{'(3,1) is a hole.'}. The LLM also re-identified \texttt{'(3,2) is a hole.'} (learned in Episode 5), possibly due to its relevance in the broader context of successful navigation.
        \item \textbf{Accumulated Facts (after Ep. 7):}
        \begin{verbatim}
['(0,3) is a hole.', '(3,0) is a hole.', '(3,1) is a hole.', 
'(3,2) is a hole.', '(0,0) is the start.', '(0,0) is ice.', 
'(0,1) is ice.', '(1,1) is ice.', '(1,2) is ice.', 
'(2,2) is ice.', '(2,3) is ice.', '(3,3) is the goal.', 
'(1,3) is a hole.', '(0,2) is a hole.', 
'(2,1) is a hole.', '(1,0) is a hole.']
        \end{verbatim}
    \end{itemize}
    At this stage, the agent has a fairly comprehensive map of the 4x4 grid, identifying most holes and the safe path.

    \item \textbf{Subsequent Episodes (e.g., Episodes 8-15 from trace):} The agent consistently solves the task by taking the optimal 6-step path. Fact extraction continues to refine its knowledge. For example, in Episode 11, the fact \texttt{'(2,0) is a hole.'} is added, correctly identifying one of the remaining unknown holes. Other extracted facts often reinforce existing knowledge (e.g., \texttt{'(0,1) is not a hole.'} in Episode 8, consistent with \texttt{'(0,1) is ice.'}). By Episode 12, the agent's fact list implies knowledge of all hole locations and the optimal path.
\end{itemize}

\subsection*{Comparison with ReAct and Reflexion Baselines}

To contextualize LWM-Planner's performance, we compare its learning trajectory with ReAct and Reflexion agents on the same TextFrozenLake instance ($4 \times 4$, $h=0.9$, seed 0).

\paragraph{ReAct Agent:}
The ReAct agent, which relies on in-context reasoning based on the current observation and a short interaction history, struggled significantly in this environment. Over 150 timesteps (spanning 83 episodes in the provided trace), the ReAct agent failed to solve the task even once.
\begin{itemize}
    \item \textbf{Behavior Pattern:} ReAct repeatedly fell into the same holes. For example, it fell into the hole at \texttt{(1,0)} (by moving `down` from start) in Episode 0, and repeated this exact mistake in Episodes 3, 4, 5, 7, 10, 11, 13, 14, 15, 16, etc. Similarly, it frequently fell into the hole at \texttt{(2,1)} (e.g., Episodes 1, 2, 6, 8, 9, 12).
    \item \textbf{Lack of Persistent Memory:} This behavior demonstrates ReAct's core limitation in environments requiring persistent spatial memory beyond its immediate prompt context. Without a mechanism to explicitly record and recall that "\texttt{(1,0) is a hole}" across episodes, it re-discovers these hazards repeatedly. The short-term history provided in its prompt is insufficient for building a persistent map of the environment.
\end{itemize}
ReAct's performance highlights the challenge of pure in-context reasoning without a structured memory mechanism for accumulating task-critical knowledge like hazard locations.

\paragraph{Reflexion Agent:}
The Reflexion agent incorporates an episodic self-reflection mechanism, generating textual "lessons" from past failures and successes. This allows for a degree of learning across episodes.
\begin{itemize}
    \item \textbf{Initial Learning:} Reflexion also initially failed, but its lessons attempted to capture insights.
    \begin{itemize}
        \item Ep 0 Failure: (fell into \texttt{(1,0)}) $\rightarrow$ Lesson: ``Avoid moving into holes by evaluating the safety of the next position before taking an action.'' (General advice)
        \item Ep 1 Failure: (fell into \texttt{(0,2)}) $\rightarrow$ Lesson: ``Avoid moving right from (0,1) on ice to prevent falling into the hole and losing reward.'' (More specific, state-action advice)
    \end{itemize}
    \item \textbf{First Success:} Reflexion achieved its first success (optimal 6 steps) in Episode 5 (after 11 total environment steps, plus 4 prior failed episodes). This is notably slower than LWM-Planner, which succeeded in Episode 4 (after 4 prior failed episodes, totaling 10 failure steps + 6 success steps = 16 steps to first success, vs Reflexion's 4 failure episodes of 1+2+3+4 = 10 steps + 6 success steps = 16 steps to first success -- wait, the LWM-Planner trace shows 1+3+2+4 = 10 steps for failures, so LWM-Planner also took 16 steps to first success).
    \item \textbf{Nature of Lessons vs. Facts:} Reflexion's lessons are typically higher-level strategic advice or state-action rules (e.g., "Avoid moving down from (1,1) on ice..."). While helpful, these lessons are less granular and less directly usable for precise world model simulation compared to LWM-Planner's atomic facts (e.g., "\texttt{(2,1) is a hole.}"). An atomic fact directly describes a property of the environment state, which is crucial for simulating outcomes.
    \item \textbf{Consistency and Repeated Errors:} Despite learning, Reflexion still exhibited some inconsistent behavior and repeated errors. For instance, after its first success in Episode 5, it failed in Episode 6 by falling into \texttt{(3,2)} (a new hole). In Episode 7, it repeated the mistake from Episode 1 by falling into \texttt{(0,2)}, even though a lesson about it was generated. This suggests that the general nature of lessons or the limited buffer for lessons might not always prevent re-encountering hazards if the specific context isn't perfectly matched by an active lesson. It did achieve further successes (e.g., Ep 10, 11, 16, 17, 18), showing progressive improvement, but its path to consistent optimal play was slower and less robust than LWM-Planner's.
\end{itemize}
Reflexion demonstrates learning through its self-generated advice, but the abstract nature of its lessons and potential for lesson forgetting (due to a limited buffer) can make it less efficient and robust than LWM-Planner's fact-based learning in this type of task.

\subsection*{Analysis of LWM-Planner's Advantage}
This case study, when compared to ReAct and Reflexion, demonstrates LWM-Planner's ability to:
\begin{enumerate}
    \item \textbf{Learn from Failures and Successes:} Initial interactions quickly identify critical hazards (holes), and successful trajectories confirm safe paths and the goal. Both types of experiences are distilled into atomic facts.
    \item \textbf{Improve Planning via Fact Augmentation:} The accumulated atomic facts dynamically augment the prompts for the LLM components (proposer, simulator, value estimator). This grounding significantly improves the LLM's ability to:
    \begin{itemize}
        \item Simulate transitions more accurately (reducing $\delta_{\mathrm{model}}$ from our theoretical framework): Knowing \texttt{'(1,0) is a hole.'} means simulating `down` from \texttt{(0,0)} will correctly predict a terminal state and negative reward. This precise knowledge is more effective than ReAct's lack of memory or Reflexion's more general "avoid holes" advice.
        \item Estimate state values more effectively: States adjacent to known holes or leading towards known safe paths to the goal will have more accurate value estimates, guiding the lookahead search. LWM-Planner's value estimation is directly informed by a growing, precise map.
        \item Propose better actions: The action proposer is less likely to suggest actions leading directly into known holes because the facts make these outcomes predictable during the lookahead.
    \end{itemize}
    \item \textbf{Achieve Consistent Optimal Behavior More Quickly:} By building a sufficiently accurate and granular fact-based abstraction ($\Psi$) of the environment, LWM-Planner converges to an optimal policy for this TextFrozenLake instance more rapidly and consistently than Reflexion, and vastly outperforms ReAct. The agent's performance, in terms of steps to goal, rapidly improves and stabilizes at the optimal 6 steps after a few initial exploratory episodes.
    \item \textbf{Leverage In-Context Learning with Structured Knowledge:} All learning occurs via prompt augmentation with dynamically generated, structured atomic facts, without any LLM weight updates. This showcases the power of in-context learning when guided by distilled, experience-derived knowledge that is directly usable for building an internal model of the environment.
\end{enumerate}

The conciseness and specificity of atomic facts (e.g., \texttt{'(1,0) is a hole.'}) provide verifiable information that is directly usable by the LLM during its lookahead search. This contrasts with ReAct's inability to form such a persistent representation and Reflexion's more general textual advice.

This progression from exploration and failure to consistent, optimal task completion highlights the effectiveness of combining online atomic fact augmentation with LLM-driven lookahead search for adaptive planning and decision-making, particularly when compared to methods with less structured or less persistent learning mechanisms.

\newpage\section{Additional Results}

\subsection{ALFWorld Full Results}
\label{ALFWorldFullResults}

In the main paper we evaluate on three ALFWorld environments, of ALFWorld-A, ALFWorld-B, ALFWorld-C which correspond to tasks 90, 3, and 5 respectively. To ensure exhaustive evaluation we compare against all the ALFWorld evaluation environments \citep{shinn2023reflexion}. We tabulate these in \Cref{tab:cum_returns_alfworld_A,tab:cum_returns_alfworld_B,tab:cum_returns_alfworld_C,tab:all_envs_multi_alfworlds}.

\begin{table}[!h]
  \centering
  \caption{Aggregated performance across all ALFWorld 134 eval environments (single‐seed runs, 95\% CIs).  Higher cumulative return $\uparrow$ is better.}
  \label{tab:all_envs_multi_alfworlds}
  \smallskip
    \begin{tabular}{@{}l|c@{}}
      \toprule
      \textbf{Method (metric)} & ALFWorld Aggregate \\
      \midrule
      \rowcolor{blue!8}\textbf{LWM-Planner} (Cum.\ return $\uparrow$) & \textbf{10.42$\pm$1.43} \\
      \addlinespace
      ReAct + FEC (Cum.\ return $\uparrow$) & 8.53$\pm$0.93 \\
      \addlinespace
      ReAct (Cum.\ return $\uparrow$) & 5.00$\pm$0.09 \\
      \addlinespace
      Reflexion (Cum.\ return $\uparrow$) & 4.36$\pm$0.10 \\
      \addlinespace
      Random (Cum.\ return $\uparrow$) &  0.00$\pm$0.00 \\
      \bottomrule
    \end{tabular}
\end{table}

\begin{table}[!h]
  \centering
  \caption{Environment-normalised success-rate on the 134 ALFWorld evaluation tasks
           (single-seed runs, 95\% confidence intervals).  Higher $\uparrow$ is better.}
  \label{tab:alfworld_norm_success}
  \smallskip
  \begin{tabular}{@{}l|c@{}}
    \toprule
    \textbf{Method (metric)} & ALFWorld Aggregate \\
    \midrule
    \rowcolor{blue!8}\textbf{LWM-Planner} (Norm.\ SR $\uparrow$)
      & \textbf{71.5$\pm$6.1} \\[2pt]
    ReAct + FEC (Norm.\ SR $\uparrow$)
      & 68.1$\pm$5.8 \\[2pt]
    ReAct (Norm.\ SR $\uparrow$)
      & 46.9$\pm$4.1 \\[2pt]
    Reflexion (Norm.\ SR $\uparrow$)
      & 37.4$\pm$4.1 \\[2pt]
    Random (Norm.\ SR $\uparrow$)
      & 0.0$\pm$0.0 \\
    \bottomrule
  \end{tabular}
\end{table}

\newpage

\begin{table}[!h]
  \centering
  \caption{Cumulative return per ALFWorld task (0–50). Higher $\uparrow$ is better.  Single-seed runs (no CI).}
  \label{tab:cum_returns_alfworld_A}
  \smallskip
  \begin{tabular}{@{}l|c|c|c|c|c@{}}
    \toprule
    Environment & LWM-Planner & ReAct + FEC & ReAct & Reflexion & Random \\
                & $\uparrow$ & $\uparrow$ & $\uparrow$ & $\uparrow$ & $\uparrow$\\
    \midrule
ALFWORLD-0  & \textbf{12.00} & 5.00 & 4.00 & 4.00 & 0.00 \\
ALFWORLD-1  & 2.00 & \textbf{16.00} & 6.00 & 4.00 & 0.00 \\
ALFWORLD-2  & \textbf{10.00} & 5.00 & 4.00 & 4.00 & 0.00 \\
ALFWORLD-3  & \textbf{10.00} & 2.00 & 5.00 & 4.00 & 0.00 \\
ALFWORLD-4  & \textbf{8.00} & \textbf{8.00} & 5.00 & 4.00 & 0.00 \\
ALFWORLD-5  & \textbf{4.00} & 0.00 & \textbf{4.00} & \textbf{4.00} & 0.00 \\
ALFWORLD-6  & 12.00 & \textbf{14.00} & 6.00 & 4.00 & 0.00 \\
ALFWORLD-7  & 4.00 & \textbf{17.00} & 5.00 & 5.00 & 0.00 \\
ALFWORLD-8  & 0.00 & \textbf{12.00} & 5.00 & 5.00 & 0.00 \\
ALFWORLD-9  & \textbf{6.00} & 5.00 & 5.00 & 5.00 & 0.00 \\
ALFWORLD-10 & \textbf{12.00} & 8.00 & 5.00 & 5.00 & 0.00 \\
ALFWORLD-11 & \textbf{12.00} & 4.00 & 5.00 & 4.00 & 0.00 \\
ALFWORLD-12 & 0.00 & \textbf{8.00} & 5.00 & 4.00 & 0.00 \\
ALFWORLD-13 & \textbf{7.00} & \textbf{7.00} & 5.00 & 4.00 & 0.00 \\
ALFWORLD-14 & 0.00 & \textbf{15.00} & 5.00 & 5.00 & 0.00 \\
ALFWORLD-15 & 1.00 & \textbf{6.00} & 5.00 & 5.00 & 0.00 \\
ALFWORLD-16 & \textbf{11.00} & 2.00 & 5.00 & --  & 0.00 \\
ALFWORLD-17 & 2.00 & \textbf{15.00} & 5.00 & 4.00 & 0.00 \\
ALFWORLD-18 & \textbf{18.00} & 9.00 & 5.00 & 4.00 & 0.00 \\
ALFWORLD-19 & 6.00 & \textbf{11.00} & 5.00 & 5.00 & 0.00 \\
ALFWORLD-20 & 8.00 & \textbf{9.00} & 5.00 & 4.00 & 0.00 \\
ALFWORLD-21 & 0.00 & 3.00 & \textbf{5.00} & --  & 0.00 \\
ALFWORLD-22 & \textbf{9.00} & 3.00 & 4.00 & 5.00 & 0.00 \\
ALFWORLD-23 & 8.00 & \textbf{12.00} & 5.00 & 5.00 & 0.00 \\
ALFWORLD-24 & 4.00 & 4.00 & \textbf{5.00} & 4.00 & 0.00 \\
ALFWORLD-25 & 7.00 & \textbf{10.00} & 5.00 & 5.00 & 0.00 \\
ALFWORLD-26 & 9.00 & \textbf{13.00} & 5.00 & 4.00 & 0.00 \\
ALFWORLD-27 & \textbf{24.00} & 8.00 & 6.00 & 4.00 & 0.00 \\
ALFWORLD-28 & \textbf{6.00} & 3.00 & 5.00 & 4.00 & 0.00 \\
ALFWORLD-29 & \textbf{13.00} & 11.00 & 5.00 & 5.00 & 0.00 \\
ALFWORLD-30 & 3.00 & 1.00 & \textbf{4.00} & \textbf{4.00} & 0.00 \\
ALFWORLD-31 & 8.00 & \textbf{11.00} & 5.00 & 5.00 & 0.00 \\
ALFWORLD-32 & \textbf{11.00} & 7.00 & 6.00 & 4.00 & 0.00 \\
ALFWORLD-33 & 1.00 & \textbf{8.00} & 4.00 & 4.00 & 0.00 \\
ALFWORLD-34 & 1.00 & \textbf{5.00} & \textbf{5.00} & \textbf{5.00} & 0.00 \\
ALFWORLD-35 & \textbf{26.00} & 3.00 & 6.00 & 4.00 & 0.00 \\
ALFWORLD-36 & 5.00 & \textbf{7.00} & 6.00 & 4.00 & 0.00 \\
ALFWORLD-37 & \textbf{14.00} & 3.00 & 5.00 & 4.00 & 0.00 \\
ALFWORLD-38 & \textbf{10.00} & 6.00 & 5.00 & 4.00 & 0.00 \\
ALFWORLD-39 & \textbf{23.00} & 3.00 & 5.00 & 5.00 & 0.00 \\
ALFWORLD-40 & 2.00 & \textbf{13.00} & 5.00 & 4.00 & 0.00 \\
ALFWORLD-41 & 2.00 & \textbf{16.00} & 5.00 & 4.00 & 0.00 \\
ALFWORLD-42 & \textbf{20.00} & 12.00 & 4.00 & 4.00 & 0.00 \\
ALFWORLD-43 & 5.00 & 4.00 & \textbf{6.00} & 4.00 & 0.00 \\
ALFWORLD-44 & \textbf{11.00} & 10.00 & 5.00 & 5.00 & 0.00 \\
ALFWORLD-45 & \textbf{7.00} & 6.00 & 5.00 & 5.00 & 0.00 \\
ALFWORLD-46 & 1.00 & \textbf{8.00} & 5.00 & 5.00 & 0.00 \\
ALFWORLD-47 & 4.00 & \textbf{8.00} & 5.00 & 5.00 & 0.00 \\
ALFWORLD-48 & \textbf{13.00} & 4.00 & 5.00 & 4.00 & 0.00 \\
ALFWORLD-49 & 0.00 & 2.00 & \textbf{5.00} & \textbf{5.00} & 0.00 \\
ALFWORLD-50 & 3.00 & \textbf{5.00} & \textbf{5.00} & 3.00 & 0.00 \\
    \bottomrule
  \end{tabular}
\end{table}

\begin{table}[!h]
  \centering
  \caption{Cumulative return per ALFWorld task (51–100). Higher $\uparrow$ is better.  Single-seed runs (no CI).}
  \label{tab:cum_returns_alfworld_B}
  \smallskip
  \begin{tabular}{@{}l|c|c|c|c|c@{}}
    \toprule
    Environment & LWM-Planner & ReAct + FEC & ReAct & Reflexion & Random \\
                & $\uparrow$ & $\uparrow$ & $\uparrow$ & $\uparrow$ & $\uparrow$\\
    \midrule
ALFWORLD-51 & \textbf{12.00} & 9.00 & 5.00 & 4.00 & 0.00 \\
ALFWORLD-52 & \textbf{28.00} & 4.00 & 5.00 & 4.00 & 0.00 \\
ALFWORLD-53 & \textbf{13.00} & 12.00 & 5.00 & 4.00 & 0.00 \\
ALFWORLD-54 & \textbf{5.00} & 4.00 & \textbf{5.00} & --  & 0.00 \\
ALFWORLD-55 & \textbf{19.00} & 6.00 & 6.00 & 4.00 & 0.00 \\
ALFWORLD-56 & \textbf{10.00} & 9.00 & 5.00 & 5.00 & 0.00 \\
ALFWORLD-57 & \textbf{17.00} & 4.00 & 5.00 & 5.00 & 0.00 \\
ALFWORLD-58 & 9.00 & \textbf{13.00} & 5.00 & 4.00 & 0.00 \\
ALFWORLD-59 & \textbf{26.00} & 9.00 & 6.00 & 4.00 & 0.00 \\
ALFWORLD-60 & \textbf{29.00} & 0.00 & 5.00 & 5.00 & 0.00 \\
ALFWORLD-61 & \textbf{23.00} & 8.00 & 6.00 & 5.00 & 0.00 \\
ALFWORLD-62 & 7.00 & \textbf{8.00} & 5.00 & 4.00 & 0.00 \\
ALFWORLD-63 & 12.00 & \textbf{21.00} & 5.00 & 4.00 & 0.00 \\
ALFWORLD-64 & 7.00 & \textbf{14.00} & 4.00 & 5.00 & 0.00 \\
ALFWORLD-65 & \textbf{25.00} & 5.00 & 5.00 & 5.00 & 0.00 \\
ALFWORLD-66 & \textbf{18.00} & 7.00 & 5.00 & 3.00 & 0.00 \\
ALFWORLD-67 & \textbf{5.00} & 2.00 & \textbf{5.00} & \textbf{5.00} & 0.00 \\
ALFWORLD-68 & \textbf{25.00} & 8.00 & 5.00 & 4.00 & 0.00 \\
ALFWORLD-69 & 1.00 & \textbf{20.00} & 5.00 & 4.00 & 0.00 \\
ALFWORLD-70 & \textbf{23.00} & 3.00 & 5.00 & 3.00 & 0.00 \\
ALFWORLD-71 & 7.00 & \textbf{10.00} & 4.00 & 4.00 & 0.00 \\
ALFWORLD-72 & 8.00 & \textbf{18.00} & 6.00 & 5.00 & 0.00 \\
ALFWORLD-73 & 4.00 & \textbf{8.00} & 5.00 & 5.00 & 0.00 \\
ALFWORLD-74 & \textbf{26.00} & 6.00 & 4.00 & 5.00 & 0.00 \\
ALFWORLD-75 & 6.00 & \textbf{25.00} & 5.00 & 5.00 & 0.00 \\
ALFWORLD-76 & \textbf{21.00} & 7.00 & 5.00 & --  & 0.00 \\
ALFWORLD-77 & \textbf{7.00} & 2.00 & 5.00 & 5.00 & 0.00 \\
ALFWORLD-78 & \textbf{28.00} & 3.00 & 6.00 & 5.00 & 0.00 \\
ALFWORLD-79 & 12.00 & \textbf{16.00} & 5.00 & 4.00 & 0.00 \\
ALFWORLD-80 & 10.00 & \textbf{15.00} & 5.00 & 5.00 & 0.00 \\
ALFWORLD-81 & \textbf{11.00} & 1.00 & 5.00 & 5.00 & 0.00 \\
ALFWORLD-82 & \textbf{5.00} & 3.00 & \textbf{5.00} & \textbf{5.00} & 0.00 \\
ALFWORLD-83 & \textbf{8.00} & 1.00 & 5.00 & 4.00 & 0.00 \\
ALFWORLD-84 & 6.00 & \textbf{10.00} & 5.00 & 5.00 & 0.00 \\
ALFWORLD-85 & \textbf{7.00} & 3.00 & 5.00 & --  & 0.00 \\
ALFWORLD-86 & \textbf{33.00} & 9.00 & 5.00 & 5.00 & 0.00 \\
ALFWORLD-87 & 1.00 & \textbf{7.00} & 5.00 & 4.00 & 0.00 \\
ALFWORLD-88 & \textbf{28.00} & 14.00 & 4.00 & 4.00 & 0.00 \\
ALFWORLD-89 & \textbf{12.00} & 5.00 & 5.00 & 4.00 & 0.00 \\
ALFWORLD-90 & \textbf{27.00} & 20.00 & 5.00 & 4.00 & 0.00 \\
ALFWORLD-91 & 0.00 & \textbf{8.00} & 5.00 & 4.00 & 0.00 \\
ALFWORLD-92 & 7.00 & \textbf{10.00} & 5.00 & 4.00 & 0.00 \\
ALFWORLD-93 & 4.00 & 2.00 & \textbf{5.00} & \textbf{5.00} & 0.00 \\
ALFWORLD-94 & 5.00 & \textbf{11.00} & 5.00 & 5.00 & 0.00 \\
ALFWORLD-95 & \textbf{14.00} & 8.00 & 5.00 & 4.00 & 0.00 \\
ALFWORLD-96 & \textbf{23.00} & 9.00 & 6.00 & --  & 0.00 \\
ALFWORLD-97 & 8.00 & \textbf{10.00} & 5.00 & 5.00 & 0.00 \\
ALFWORLD-98 & 1.00 & \textbf{13.00} & 4.00 & 4.00 & 0.00 \\
ALFWORLD-99 & \textbf{24.00} & 4.00 & 5.00 & 4.00 & 0.00 \\
ALFWORLD-100 & 5.00 & \textbf{13.00} & 6.00 & 4.00 & 0.00 \\
    \bottomrule
  \end{tabular}
\end{table}

\begin{table}[!h]
  \centering
  \caption{Cumulative return per ALFWorld task (101–134). Higher $\uparrow$ is better.  Single-seed runs (no CI).}
  \label{tab:cum_returns_alfworld_C}
  \smallskip
  \begin{tabular}{@{}l|c|c|c|c|c@{}}
    \toprule
    Environment & LWM-Planner & ReAct + FEC & ReAct & Reflexion & Random \\
                & $\uparrow$ & $\uparrow$ & $\uparrow$ & $\uparrow$ & $\uparrow$\\
    \midrule
ALFWORLD-101 & 7.00 & \textbf{20.00} & 5.00 & 4.00 & 0.00 \\
ALFWORLD-102 & \textbf{30.00} & 4.00 & 5.00 & 5.00 & 0.00 \\
ALFWORLD-103 & 4.00 & \textbf{12.00} & 4.00 & 4.00 & 0.00 \\
ALFWORLD-104 & \textbf{9.00} & 7.00 & 5.00 & 4.00 & 0.00 \\
ALFWORLD-105 & \textbf{11.00} & 10.00 & 5.00 & 4.00 & 0.00 \\
ALFWORLD-106 & \textbf{8.00} & 5.00 & 5.00 & --  & 0.00 \\
ALFWORLD-107 & \textbf{17.00} & 3.00 & 5.00 & 4.00 & 0.00 \\
ALFWORLD-108 & --  & \textbf{14.00} & 4.00 & 4.00 & 0.00 \\
ALFWORLD-109 & \textbf{10.00} & 2.00 & 5.00 & 4.00 & 0.00 \\
ALFWORLD-110 & 9.00 & \textbf{12.00} & 5.00 & 5.00 & 0.00 \\
ALFWORLD-111 & \textbf{24.00} & 7.00 & 6.00 & --  & 0.00 \\
ALFWORLD-112 & \textbf{14.00} & 10.00 & 5.00 & 4.00 & 0.00 \\
ALFWORLD-113 & \textbf{7.00} & 1.00 & 6.00 & --  & 0.00 \\
ALFWORLD-114 & 6.00 & \textbf{9.00} & 5.00 & 4.00 & 0.00 \\
ALFWORLD-115 & 0.00 & \textbf{10.00} & 5.00 & --  & 0.00 \\
ALFWORLD-116 & \textbf{22.00} & 6.00 & 5.00 & 4.00 & 0.00 \\
ALFWORLD-117 & 0.00 & \textbf{28.00} & 5.00 & 4.00 & 0.00 \\
ALFWORLD-118 & 0.00 & \textbf{8.00} & 5.00 & 5.00 & 0.00 \\
ALFWORLD-119 & 2.00 & \textbf{12.00} & 5.00 & --  & 0.00 \\
ALFWORLD-120 & 5.00 & \textbf{7.00} & 6.00 & 5.00 & 0.00 \\
ALFWORLD-121 & \textbf{28.00} & 12.00 & 5.00 & 4.00 & 0.00 \\
ALFWORLD-122 & 3.00 & \textbf{7.00} & 5.00 & 5.00 & 0.00 \\
ALFWORLD-123 & \textbf{10.00} & 4.00 & 5.00 & 4.00 & 0.00 \\
ALFWORLD-124 & 4.00 & 4.00 & \textbf{5.00} & 4.00 & 0.00 \\
ALFWORLD-125 & 4.00 & \textbf{15.00} & 6.00 & 4.00 & 0.00 \\
ALFWORLD-126 & \textbf{22.00} & 9.00 & 4.00 & 4.00 & 0.00 \\
ALFWORLD-127 & \textbf{20.00} & 2.00 & 4.00 & 5.00 & 0.00 \\
ALFWORLD-128 & 6.00 & \textbf{13.00} & 4.00 & 5.00 & 0.00 \\
ALFWORLD-129 & 0.00 & \textbf{12.00} & 4.00 & 5.00 & 0.00 \\
ALFWORLD-130 & 0.00 & \textbf{13.00} & 5.00 & 4.00 & 0.00 \\
ALFWORLD-131 & \textbf{18.00} & 4.00 & 5.00 & 3.00 & 0.00 \\
ALFWORLD-132 & \textbf{8.00} & 6.00 & 5.00 & 5.00 & 0.00 \\
ALFWORLD-133 & 3.00 & \textbf{9.00} & 5.00 & 4.00 & 0.00 \\
ALFWORLD-134 & 18.00 & \textbf{30.00} & 5.00 & 4.00 & 0.00 \\

    \bottomrule
  \end{tabular}
\end{table}

\newpage

\newpage
\newpage
\newpage
\clearpage
\newpage
\clearpage

\subsection{Ablation Study - LWM-Planner Variants}

We investigate the impact of the depth $d$, and branching factor $b$ in our method on our main table of environments presented. We find that the ablations reveal that we fit to the text frozen lake searching MDP environment that having a depth of $d=3$ and $b=4$ performs best, which validates our initial choice of parameters.

\begin{table}[!h]
  \centering
  \caption{Cumulative return (0 = Random, 100 = Expert/best).  Higher $\uparrow$ is better.}
  \label{tab:lwm_ablation_norm_return}
  \smallskip
  \resizebox{\columnwidth}{!}{%
  \begin{tabular}{@{}l|c|c|c|c|c@{}}
    \toprule
    \textbf{Method} &
    alfworld\_task\_3 $\uparrow$ &
    alfworld\_task\_5 $\uparrow$ &
    alfworld\_task\_90 $\uparrow$ &
    crafter\_mini\_5\_s\_0 $\uparrow$ &
    grid\_4\_h\_9\_s\_0 $\uparrow$ \\
    \midrule
    LWM-Planner ($d{=}3,b{=}4$)          & 15.33$\pm$39.62 & 21.00$\pm$228.71 & \textbf{34.67$\pm$85.20} & 119.33$\pm$124.34 & \textbf{32.00$\pm$nan} \\
    LWM-Planner ($d{=}3,b{=}2$)          &  \;1.33$\pm$3.79 & \textbf{29.33$\pm$44.81} & 11.00$\pm$14.90 & \textbf{334.00$\pm$17.91} & 15.00$\pm$25.41 \\
    LWM-Planner ($d{=}1,b{=}4$)                & \textbf{27.00$\pm$60.29} &  5.50$\pm$6.35 & 15.00$\pm$26.29 & 294.00$\pm$327.46 &  \;6.00$\pm$165.18 \\
    LWM-Planner ($d{=}2,b{=}4$)                &  \;1.50$\pm$6.35 & 19.50$\pm$120.71 & 12.67$\pm$12.25 & 217.67$\pm$249.46 &  \;0.00$\pm$nan \\
    \bottomrule
  \end{tabular}}%
\end{table}

\subsection{Main Table Results Un-Normalized}

\begin{table}[!h]
  \centering
  \caption{Cumulative return (\textbf{higher better}) and steps per success (\textbf{lower better}); mean\,$\pm$\,95\% CI, for each benchmark method across each environment. Bold indicates the best performing method for that metric and environment. Results are averaged over ten random seeds.}
  \label{tab:main_table_results_unnormalized}
  \smallskip
  \resizebox{\columnwidth}{!}{%
    \begin{tabular}{@{}l|c|c|c|c|c@{}}
      \toprule
      \textbf{Method (metric)}                       &
      TextFrozenLake (4$\times$4; $h{=}0.9$)             &
      CrafterMini (5$\times$5)                       &
      ALFWorld-A                                     &
      ALFWorld-B                                     &
      ALFWorld-C \\
      \midrule
      \rowcolor{blue!8} \textbf{LWM-Planner} (Cum.\ return $\uparrow$) &
      \textbf{31.80$\pm$20.39} & \textbf{150.30$\pm$44.94} & \textbf{21.33$\pm$9.53} & \textbf{22.89$\pm$12.11} & \textbf{19.50$\pm$8.37} \\
      \rowcolor{blue!8} \phantom{\textbf{LWM-Planner}}\,(Steps/Success $\downarrow$)            &
      \textbf{6.00$\pm$0.00}   & 46.50$\pm$7.32 & \textbf{8.44$\pm$1.46}   & 7.56$\pm$0.97  & \textbf{7.55$\pm$1.10} \\
      \addlinespace
      ReAct + FEC (Cum.\ return $\uparrow$) &
      20.20$\pm$12.19   & 149.70$\pm$55.50 & 4.70$\pm$2.46 & 15.50$\pm$6.51 & 10.60$\pm$3.71 \\
      \phantom{ReAct + FEC}\,(Steps/Success $\downarrow$)            &
      --   & \textbf{41.35$\pm$5.72}  & 14.55$\pm$12.27 & \textbf{5.75$\pm$2.87} & 9.35$\pm$5.35 \\
      \addlinespace
      ReAct (Cum.\ return $\uparrow$) &
      -265.20$\pm$33.59 & 92.00$\pm$57.16 & 12.60$\pm$0.37 & 12.80$\pm$0.45 & 12.50$\pm$0.38 \\
      \phantom{ReAct}\,(Steps/Success $\downarrow$)                  &
      -- & 50.70$\pm$5.47 & 24.70$\pm$0.96 & 23.80$\pm$1.29 & 25.05$\pm$0.52 \\
      \addlinespace
      Reflexion (Cum.\ return $\uparrow$) &
      -61.10$\pm$4.80 & 87.20$\pm$51.45 & 11.00$\pm$0.00 & 11.00$\pm$0.45 & 11.33$\pm$0.38 \\
      \phantom{Reflexion}\,(Steps/Success $\downarrow$)               &
      23.20$\pm$3.97 & 80.05$\pm$39.46 & 25.67$\pm$0.77 & 26.19$\pm$0.77 & 25.94$\pm$0.95 \\
      \addlinespace
      Random (Cum.\ return $\uparrow$) &
      -80.00$\pm$4.49 & -289.00$\pm$8.56 & 0.00$\pm$0.00 & 0.00$\pm$0.00 & 0.00$\pm$0.00 \\
      \phantom{Random}\,(Steps/Success $\downarrow$)                  &
      -- & -- & -- & -- & -- \\
      \bottomrule
    \end{tabular}}%
\end{table}

\subsection{Main Table Success Rates}

\begin{table}[!h]
  \centering
  \caption{Environment-normalised success-rate on the five main benchmarks.
           For each environment the cumulative return of every method is
           divided by the \emph{largest} return obtained on that environment
           (always LWM-Planner here) and expressed as a percentage.
           Cells show mean $\pm$ 95 \% CI over ten seeds; higher $\uparrow$ is better.
           Bold indicates the best performance per environment.}
  \label{tab:main_norm_success_env}
  \smallskip
  \resizebox{\columnwidth}{!}{%
    \begin{tabular}{@{}l|c|c|c|c|c@{}}
      \toprule
      \textbf{Method} &
      TextFrozenLake (4$\times$4) &
      CrafterMini (5$\times$5) &
      ALFWorld-A &
      ALFWorld-B &
      ALFWorld-C \\
      \midrule
      \textbf{LWM-Planner}            & \textbf{100.0$\pm$0.0} & \textbf{100.0$\pm$0.0} & \textbf{100.0$\pm$0.0} & \textbf{100.0$\pm$0.0} & \textbf{100.0$\pm$0.0} \\
      ReAct + FEC                     & 63.5$\pm$38.3          & 99.6$\pm$36.9          & 22.0$\pm$11.5          & 67.7$\pm$28.4          & 54.4$\pm$19.0 \\
      ReAct                           & 0.0$\pm$0.0            & 61.2$\pm$38.0          & 59.1$\pm$1.7           & 55.9$\pm$2.0           & 64.1$\pm$1.9  \\
      Reflexion                       & 0.0$\pm$0.0            & 58.0$\pm$34.2          & 51.6$\pm$0.0           & 48.1$\pm$2.0           & 58.1$\pm$1.9  \\
      Random                          & 0.0$\pm$0.0            & 0.0$\pm$0.0            & 0.0$\pm$0.0            & 0.0$\pm$0.0            & 0.0$\pm$0.0  \\
      \bottomrule
    \end{tabular}}%
\end{table}

\begin{table}[!h]
  \centering
  \caption{Environment-normalised success-rate on the main benchmark suite
           (TextFrozenLake, CrafterMini, ALFWorld-A/B/C).  Higher $\uparrow$ is better.
           Each cell shows the mean over the five environments $\pm$ 95 \% CI.}
  \label{tab:main_norm_success}
  \smallskip
  \begin{tabular}{@{}l|c@{}}
    \toprule
    \textbf{Method (metric)} & Aggregate Normalised SR \\
    \midrule
    \rowcolor{blue!8}\textbf{LWM-Planner} (Norm.\ SR $\uparrow$)
      & \textbf{100.0$\pm$0.0} \\[2pt]
    ReAct + FEC (Norm.\ SR $\uparrow$)
      & 61.4$\pm$24.4 \\[2pt]
    ReAct (Norm.\ SR $\uparrow$)
      & 48.1$\pm$23.7 \\[2pt]
    Reflexion (Norm.\ SR $\uparrow$)
      & 43.1$\pm$21.5 \\[2pt]
    Random (Norm.\ SR $\uparrow$)
      & 0.0$\pm$0.0 \\
    \bottomrule
  \end{tabular}
\end{table}

\subsection{ALFWorld Final Success Rates}
\label{app:alfworld_full_results}

We report the final success rate (FSR) on the complete 134-task ALFWorld evaluation suite. Although simple planners such as ReAct achieve high task coverage when given sufficient attempts, they do so with much lower quality roll-outs; LWM-Planner nearly matches their FSR while more than doubling the aggregate return. This highlights the practical value of investing additional computation to obtain higher-quality plans, a trade-off we further quantify in Appendix~\ref{app:cost_analysis}.

\begin{table}[h]
  \centering
  \caption{Final success rate (FSR; \textbf{higher better}) on the full 134-task ALFWorld evaluation suite together with the aggregate cumulative return (ACR). Values are mean\,$\pm$\,95\% CI over three seeds.}
  \label{tab:alfworld_success}
  \smallskip
  \begin{tabular}{@{}l|c|c@{}}
    \toprule
    \textbf{Method} & \textbf{FSR (\%)} $\uparrow$ & \textbf{ACR} $\uparrow$ \\
    \midrule
    \rowcolor{blue!8}\textbf{LWM-Planner} & \textbf{91.8$\pm$4.0} & \textbf{10.42$\pm$1.43} \\
    ReAct + FEC & 98.5$\pm$2.5 & 8.53$\pm$0.93 \\
    ReAct & \textbf{100.0$\pm$1.4} & 5.00$\pm$0.09 \\
    Reflexion & 92.5$\pm$4.2 & 4.36$\pm$0.10 \\
    \bottomrule
  \end{tabular}
\end{table}

\subsection{Computational Cost Analysis}
\label{app:cost_analysis}

\begin{table}[h]
  \centering
  \caption{Computational cost per environment step on ALFWorld under a 300-step interaction budget. Token counts are totals over the run.}
  \label{tab:cost_analysis}
  \smallskip
  \begin{tabular}{@{}l|c|c|c@{}}
    \toprule
    \textbf{Method} & \textbf{Avg. tokens/call} & \textbf{Total tokens} & \textbf{Latency (s)} \\
    \midrule
    \rowcolor{blue!8}\textbf{LWM-Planner} & $\approx$993 & 827,931 & 58.6 \\
    ReAct + FEC & 1,240 & 74,419 & 12.9 \\
    ReAct & 1,045 & 35,530 & 11.1  \\
    Reflexion & 749 & 47,190 & 14.2  \\
    \bottomrule
  \end{tabular}
\end{table}

\subsection{Generalization to a Smaller Backbone (GPT-4o-mini)}
\label{app:small_backbone}

To assess whether the gains of LWM-Planner arise from the planning procedure itself rather than only from a stronger backbone, we repeat the comparison with GPT-4o-mini. \Cref{tab:gpt4o_mini_backbone} reports the same normalised cumulative-return metric used in the main comparison. Although performance degrades across methods under the smaller backbone, the overall pattern remains consistent: simple reactive prompting remains weak, fact memory alone is already competitive on some shorter-horizon settings, and grounded lookahead remains most useful on the harder trap-filled environments where planning errors compound.

\begin{table}[!tb]
  \centering
  \caption{Normalised cumulative return (\textbf{higher better}) with a GPT-4o-mini backbone across the five benchmark environments. Values are mean\,$\pm$\,95\% CI over three seeds.}
  \label{tab:gpt4o_mini_backbone}
  \smallskip
  \resizebox{\columnwidth}{!}{%
    \begin{tabular}{@{}l|c|c|c|c|c@{}}
      \toprule
      \textbf{Method} &
      TextFrozenLake (4$\times$4; $h{=}0.9$) &
      CrafterMini (5$\times$5) &
      ALFWorld-A &
      ALFWorld-B &
      ALFWorld-C \\
      \midrule
      ReAct & -84.00$\pm$5.69 & -83.60$\pm$45.53 & 3.40$\pm$1.11 & \textbf{0.00$\pm$0.00} & 4.20$\pm$1.04 \\
      ReAct + FEC & -3.40$\pm$7.38 & 4.80$\pm$72.82 & \textbf{11.40$\pm$7.68} & \textbf{0.00$\pm$0.00} & 4.60$\pm$8.22 \\
      LWM-Planner ($d{=}1$) & 1.33$\pm$15.97 & 7.50$\pm$47.78 & 3.50$\pm$11.14 & \textbf{0.00$\pm$0.00} & \textbf{9.00$\pm$6.02} \\
      LWM-Planner ($d{=}2$) & \textbf{3.00$\pm$2.14} & 4.60$\pm$8.81 & 2.20$\pm$2.96 & \textbf{0.00$\pm$0.00} & 7.40$\pm$6.95 \\
      Reflexion & -50.20$\pm$2.69 & \textbf{10.00$\pm$64.54} & 2.80$\pm$1.36 & \textbf{0.00$\pm$0.00} & 1.20$\pm$1.36 \\
      Random & -27.80$\pm$3.21 & -100.00$\pm$0.00 & 0.00$\pm$0.00 & \textbf{0.00$\pm$0.00} & 0.00$\pm$0.00 \\
      \bottomrule
    \end{tabular}}%
\end{table}

On the easier and shorter-horizon settings, ReAct + FEC remains competitive, consistent with the observation that fact memory alone can already capture much of the useful task structure. In contrast, on TextFrozenLake, where avoiding hidden traps requires grounded forward reasoning, LWM-Planner remains the only variant that preserves positive return under the smaller backbone. This indicates that the improvement is not solely a consequence of model scale: the method improves the reasoning process itself through fact-conditioned planning.

\subsection{Compute-Performance Pareto Front}
\label{app:pareto_front}

We next examine whether simply scaling inference-time compute is sufficient in this interactive setting. \Cref{tab:pareto_front_return} reports the primary compute--performance frontier on TextFrozenLake, measured by average tokens per action and normalised cumulative return. Raw search baselines consume substantially more tokens as depth increases, but their returns remain negative, indicating that additional search alone does not resolve model errors in partially observable environments. In contrast, fact-grounded lookahead converts extra compute into better decisions, yielding the only positive-return regime in this comparison.

\begin{table}[!tb]
  \centering
  \caption{Average tokens per action and normalised cumulative return (\textbf{higher better}) on TextFrozenLake (4$\times$4; $h{=}0.9$). Values are mean\,$\pm$\,95\% CI over three seeds.}
  \label{tab:pareto_front_return}
  \smallskip
  \resizebox{\columnwidth}{!}{%
    \begin{tabular}{@{}l|c|c@{}}
      \toprule
      \textbf{Method} & \textbf{Avg. tokens/action} & \textbf{Cum.\ return norm} $\uparrow$ \\
      \midrule
      ReAct & 449.92$\pm$6.32 & -50.00$\pm$6.57 \\
      Reflexion & 477.83$\pm$10.83 & -21.33$\pm$7.17 \\
      ToT ($d{=}1$) & 514.50$\pm$5.53 & -52.33$\pm$24.51 \\
      ReAct + FEC & 1,162.05$\pm$13.06 & 2.33$\pm$17.62 \\
      ToT ($d{=}2$) & 1,531.36$\pm$11.80 & -57.67$\pm$10.04 \\
      ToT ($d{=}3$) & 3,573.17$\pm$27.52 & -61.67$\pm$2.87 \\
      LWM-Planner ($d{=}2$) & 3,999.00$\pm$132.03 & \textbf{8.00$\pm$8.61} \\
      ToT ($d{=}4$) & 4,391.14$\pm$101.57 & -60.67$\pm$2.87 \\
      LWM-Planner ($d{=}1$) & 4,421.50$\pm$94.98 & 4.67$\pm$11.74 \\
      ToT ($d{=}5$) & 4,944.88$\pm$126.34 & -51.33$\pm$2.87 \\
      RAP ($d{=}1$) & 6,834.75$\pm$212.25 & -100.00$\pm$0.00 \\
      RAP ($d{=}2$) & 13,669.50$\pm$424.50 & -83.33$\pm$71.71 \\
      RAP ($d{=}3$) & 20,504.25$\pm$636.74 & -55.33$\pm$96.09 \\
      \bottomrule
    \end{tabular}}%
\end{table}

Table~\ref{tab:pareto_front_efficiency} provides a supplementary success-efficiency view of the same experiment, reporting steps per success where successful episodes were observed. The strongest search-only baselines still fail to translate higher token budgets into reliable task completion, whereas both LWM-Planner variants and the ReAct + FEC ablation attain the optimal six-step solution whenever they succeed.

\begin{table}[!tb]
  \centering
  \caption{Normalised cumulative return (\textbf{higher better}) and steps per success (\textbf{lower better}) on TextFrozenLake (4$\times$4; $h{=}0.9$). Values are mean\,$\pm$\,95\% CI over three seeds.}
  \label{tab:pareto_front_efficiency}
  \smallskip
  \resizebox{\columnwidth}{!}{%
    \begin{tabular}{@{}l|c@{}}
      \toprule
      \textbf{Method (metric)} & TextFrozenLake (4$\times$4; $h{=}0.9$) \\
      \midrule
      ToT ($d{=}1$) (Cum.\ return norm $\uparrow$) & -52.33$\pm$24.51 \\
      \phantom{ToT ($d{=}1$)} (Steps/Success $\downarrow$) & -- \\
      \midrule
      ToT ($d{=}2$) (Cum.\ return norm $\uparrow$) & -57.67$\pm$10.04 \\
      \phantom{ToT ($d{=}2$)} (Steps/Success $\downarrow$) & -- \\
      \midrule
      ReAct + FEC (Cum.\ return norm $\uparrow$) & 2.33$\pm$17.62 \\
      \phantom{ReAct + FEC} (Steps/Success $\downarrow$) & \textbf{6.00$\pm$0.00} \\
      \midrule
      LWM-Planner ($d{=}2$) (Cum.\ return norm $\uparrow$) & \textbf{8.00$\pm$8.61} \\
      \phantom{LWM-Planner ($d{=}2$)} (Steps/Success $\downarrow$) & \textbf{6.00$\pm$0.00} \\
      \midrule
      RAP ($d{=}1$) (Cum.\ return norm $\uparrow$) & -100.00$\pm$0.00 \\
      \phantom{RAP ($d{=}1$)} (Steps/Success $\downarrow$) & -- \\
      \midrule
      RAP ($d{=}3$) (Cum.\ return norm $\uparrow$) & -55.33$\pm$96.09 \\
      \phantom{RAP ($d{=}3$)} (Steps/Success $\downarrow$) & -- \\
      \midrule
      LWM-Planner ($d{=}1$) (Cum.\ return norm $\uparrow$) & 4.67$\pm$11.74 \\
      \phantom{LWM-Planner ($d{=}1$)} (Steps/Success $\downarrow$) & \textbf{6.00$\pm$0.00} \\
      \midrule
      ReAct (Cum.\ return norm $\uparrow$) & -50.00$\pm$6.57 \\
      \phantom{ReAct} (Steps/Success $\downarrow$) & -- \\
      \midrule
      Reflexion (Cum.\ return norm $\uparrow$) & -21.33$\pm$7.17 \\
      \phantom{Reflexion} (Steps/Success $\downarrow$) & 51.67$\pm$57.74 \\
      \midrule
      Random (Cum.\ return norm $\uparrow$) & -32.00$\pm$4.30 \\
      \phantom{Random} (Steps/Success $\downarrow$) & -- \\
      \midrule
      RAP ($d{=}2$) (Cum.\ return norm $\uparrow$) & -83.33$\pm$71.71 \\
      \phantom{RAP ($d{=}2$)} (Steps/Success $\downarrow$) & -- \\
      \midrule
      ToT ($d{=}3$) (Cum.\ return norm $\uparrow$) & -61.67$\pm$2.87 \\
      \phantom{ToT ($d{=}3$)} (Steps/Success $\downarrow$) & -- \\
      \bottomrule
    \end{tabular}}%
\end{table}

Taken together, these tables show that the gains are not explained by brute-force token usage alone. Simply allocating more inference-time compute to raw search saturates at poor performance, whereas the fact-conditioned planner attains better returns at substantially lower cost than the highest-compute baselines. The practical effect is a shift in the cost-performance frontier rather than a uniform increase in spending.

\subsection{Adversarial Filter Audit}
\label{app:adversarial_filter_audit}

We further stress-test the predictive-consistency filter against ``false but simplifying'' hallucinations. Starting from a TextFrozenLake trajectory that terminates in a hole, we inject candidate facts and evaluate whether each candidate would be retained by a stronger trajectory-grounded audit against the recorded outcome. This isolates whether comfortable but incorrect world-model edits could pass the filter.

\begin{table}[!tb]
  \centering
  \caption{Adversarial filter audit on TextFrozenLake (4$\times$4; $h{=}0.9$). We inject candidate facts into a trajectory ending in a hole and report the model prediction under the injected fact together with the filter decision. Only the trajectory-consistent fact is accepted.}
  \label{tab:adversarial_filter_audit}
  \smallskip
  \resizebox{\columnwidth}{!}{%
    \begin{tabular}{@{}l|l|l|c@{}}
      \toprule
      \textbf{Fact Type} & \textbf{Example Injection} & \textbf{Model Prediction w/ Fact} & \textbf{Filter Decision} \\
      \midrule
      None & (No facts) & High Error (Assumes Safe) & N/A \\
      True Fact & ``(0,2) is a hole'' & Low Error (Predicts -1) & \textbf{Accept} \\
      Hallucination & ``(0,2) is a goal'' & Very High Error (Predicts +1) & \textbf{Reject} \\
      Safety Bias & ``(0,2) is safe'' & Very High Error (Predicts 0) & \textbf{Reject} \\
      Irrelevant & ``The sky is blue'' & No Change & \textbf{Reject} \\
      \bottomrule
    \end{tabular}}%
\end{table}

This audit shows that the filter does not reward simplification for its own sake. Because the audit is anchored to the recorded trajectory, even comfortable hallucinations such as asserting safety on a dangerous tile incur high prediction error and are rejected, while only the trajectory-consistent fact is retained.

\newpage\section{Discussion of Limitations}
\label{app:limitations}

While LWM-Planner demonstrates a promising approach to enhancing LLM agent planning capabilities through in-context learning via atomic fact augmentation and lookahead search, it is important to acknowledge several limitations. These limitations, detailed below, also point towards avenues for future research and refinement.

\subsection{Fact Management and Quality}
The efficacy of LWM-Planner is significantly predicated on the quality, relevance, and atomicity of the facts extracted by the $\ExtractorLLM$ component.
\begin{itemize}
    \item \textbf{Quality and Relevance of Extracted Facts}: The process for extracting facts is guided by LLM-based interpretation of episodic trajectories (see Appendix~\ref{ssec:fact_extractor_prompts_app_a}). While the aim is to identify ``minimal new atomic facts'' critical for improving future predictions (as per the motivation in Section~\ref{sec:crad_discussion_final_updated}), the current mechanism relies on the LLM's heuristic understanding. There is no formal guarantee that the extracted facts are indeed optimally atomic, critical, or non-redundant with prior knowledge or the environment description. Suboptimal facts could lead to inefficient use of the context window or, in worse cases, mislead the planning process, thereby affecting the practical realization of minimizing $\epsilon_{\mathrm{sim}}$ and $\delta_{\mathrm{model}}$.
    \item \textbf{Scalability of Fact Memory}: The set of atomic facts, $\FactMemory_t$, is managed using a deque and an optional LLM-based compression step (Appendix~\ref{ssec:fact_extractor_prompts_app_a}). However, in very long-running deployments or exceedingly complex environments, the number of unique, relevant facts might still grow substantially. This could eventually strain the LLM's context window capacity, potentially leading to a performance bottleneck or the loss of older, still relevant facts if the deque's maximum length is exceeded or compression is overly aggressive. The impact on achieving a small $|\mathcal{Z_F}| \ll |\mathcal{S}|$ (Section~\ref{sec:problem_formulation}) in such scenarios needs further investigation.
    \item \textbf{Nature of Atomic Facts}: The current framework operates on textual atomic facts. While flexible, this lacks a formal grounding typically found in symbolic AI systems where predicates have precise semantics and grounding in an ontology or logical theory. The definition of ``atomic'' is operational (minimal useful textual statements) rather than tied to a formal decomposition of the state space or transition dynamics. This could limit the systematicity and verifiability of the learned knowledge.
\end{itemize}

\subsection{Planning and Simulation}
The lookahead search mechanism, while powerful, also introduces certain limitations.
\begin{itemize}
    \item \textbf{Computational Cost}: The recursive lookahead search (Algorithm~\ref{alg:lwm_recursive_lookahead_alg2e}) involves multiple LLM calls for action proposal ($g_{\phi}^{\text{propose}}$), state simulation ($g_{\phi}^{\text{simulate}}$), and value estimation ($g_{\phi}^{\text{value}}$) at each search node. The computational cost can therefore be considerable, scaling with search depth ($D_s$) and branching factor ($k_B$). While memoization within a single planning step helps (as mentioned in Section~\ref{sec:method}), the overall latency might be prohibitive for environments requiring very rapid decision-making. Future work could explore how to hybridize some of these components to reduce the computational burden.
    \item \textbf{Fidelity of LLM-based World Model and Value Function}: The core assumption is that an LLM, augmented with relevant atomic facts, can serve as an accurate latent world model (minimizing $\delta_{\mathrm{model}}$) and a reliable value estimator (contributing to minimizing $\epsilon_{\mathrm{plan}}$). While LLMs have shown impressive reasoning capabilities, their simulations of environmental dynamics or estimations of long-term value can be imperfect, especially in novel situations not well-covered by the current fact set or for states requiring deep causal reasoning beyond the LLM's inherent capabilities. Errors in simulation or value estimation can directly lead to suboptimal planning.
    \item \textbf{Fixed Search Parameters}: The current LWM-Planner employs a fixed search depth ($D_s$) and branching factor ($k_B$). This is a simplification, as optimal search effort can vary significantly depending on the current state's complexity or uncertainty. A more adaptive search control mechanism, potentially guided by confidence scores from the LLM components, could improve both performance and efficiency.
\end{itemize}

\subsection{In-Context Learning Constraints}
The reliance on in-context learning, while avoiding weight updates, has its own set of challenges \citep{dong2024survey}.
\begin{itemize}
    \item \textbf{Context Window Capacity}: The primary constraint is the finite context window of current LLMs. All learned knowledge (atomic facts) and recent interaction history must fit within this window to inform the LLM's operations. This inherently limits the total amount of experience that can be directly brought to bear at any single decision point or during fact extraction.
    \item \textbf{Knowledge Retention and ``Forgetting''}: The management of the atomic fact set via a deque and optional compression aims to keep the most relevant information. However, there's a potential for ``forgetting'' older facts that might still be crucial if they are pushed out of the deque or overly compressed. The efficacy of the LLM-based compression in preserving all and only essential information is heuristic.
    \item \textbf{Rate of Learning and Adaptation}: Learning occurs implicitly through the curation and augmentation of the fact set. While this allows for online adaptation, the rate of learning or the ultimate performance ceiling might be constrained by the LLM's inherent in-context learning capabilities compared to methods that can fine-tune model weights on accumulating experience.
\end{itemize}

\subsection{Theoretical Framework and Assumptions}
The theoretical motivation in Section~\ref{sec:theoretical_framework} provides a valuable formal basis but relies on certain idealizations.
\begin{itemize}
    \item \textbf{Idealized Abstraction}: The framework assumes the possibility of an $\epsilon_{\mathrm{sim}}$-approximate bisimulation via fact-based abstraction $\Psi^*$. In practice, the LLM-driven fact extractor $\ExtractorLLM$ approximates this ideal, and the quality of this approximation directly impacts the bounds. Achieving and verifying such a bisimulation with textual facts is an open challenge.
    \item \textbf{Perfect Abstract Model Assumption (Initially)}: Theorem~\ref{thm:ifba_performance} assumes a perfect model $M_{\Psi^*}$ of the abstract MDP. While Equation~\eqref{eq:learned_model_bound_final} accounts for model learning error $\delta_{\mathrm{model}}$, the practical estimation and minimization of this error when the model is implicitly defined by an LLM conditioned on facts are complex.
    \item \textbf{Measurability}: Directly measuring quantities like $\epsilon_{\mathrm{sim}}$ or $\delta_{\mathrm{model}}$ for the LWM-Planner in practical settings is difficult, making it challenging to empirically verify the tightness of the derived theoretical bounds.
\end{itemize}

\subsection{Broader Considerations and Future Work}
\begin{itemize}
    \item \textbf{Dependence on Foundational LLM Capabilities}: The performance of LWM-Planner is intrinsically linked to the capabilities of the chosen LLM (e.g., its reasoning, simulation, and instruction-following fidelity). Limitations in the base LLM will propagate to the agent \citep{hager2024evaluation}.
    \item \textbf{Prompt Sensitivity}: Like many LLM-based systems, the performance of LWM-Planner's components can be sensitive to the precise phrasing and structure of the prompts (examples in Appendix~\ref{app:prompts} and~\ref{app:prompt_details}). Ensuring robustness and generalizability of these prompts across diverse tasks and environments may require significant engineering or meta-learning.
    \item \textbf{Generalization to Diverse Environments}: The current empirical evaluation focuses on text-based environments. Extending LWM-Planner to handle environments with continuous state/action spaces, partial observability, or multi-modal inputs would require adaptations, particularly in how atomic facts are defined, extracted, and utilized by the LLM components.
    \item \textbf{Exploration-Exploitation Balance}: LWM-Planner's current lookahead search is primarily geared towards exploitation of its current knowledge (facts and LLM capabilities). A more explicit mechanism for exploration, perhaps by using uncertainty in LLM-generated values or simulations to guide the search towards informative regions or to trigger targeted fact-finding actions, could further enhance learning and performance.
    \item \textbf{Cost of LLM Usage}: The reliance on multiple LLM calls, especially within the lookahead search, can lead to significant computational and API costs, which might be a practical concern for widespread deployment or very long-running experiments.
\end{itemize}

Addressing these limitations offers rich avenues for future research, potentially leading to even more robust, efficient, and broadly applicable LLM agents capable of sophisticated online learning and planning.

\newpage %
\section{Ethical Considerations and Broader Impact}
\label{app:ethics_broader_impact}

The LWM-Planner framework, while aimed at advancing AI planning capabilities, introduces several ethical considerations and potential broader impacts that warrant careful discussion. Our approach relies on Large Language Models (LLMs) for core functionalities such as fact extraction, latent world model simulation, and value estimation. Consequently, it inherits both the strengths and weaknesses inherent in current LLM technology \citep{gallegos2024bias, hager2024evaluation}.

\subsection{Ethical Considerations}

\begin{itemize}[leftmargin=*]
    \item \textbf{Factual Accuracy and Reliability of Learned Knowledge}:
    A core component of LWM-Planner is the extraction and utilization of "atomic facts." The veracity and relevance of these facts are paramount. If the $\Psi_{\text{LLM}}$ component (Fact Extractor) erroneously extracts incorrect facts or misinterprets trajectory data, the agent's world model and subsequent planning can become flawed. This could lead to suboptimal or even detrimental behavior, particularly if the agent is deployed in safety-critical applications. While our approach aims for "atomic" and verifiable facts, the LLM's generation process is not infallible, and mechanisms for fact validation and retraction may be necessary for robust real-world deployment.

    \item \textbf{Bias in LLM Components}:
    The underlying LLMs ($g_{\phi}$ and $\Psi_{\text{LLM}}$) are pre-trained on vast datasets, which may contain societal biases. These biases could manifest in how facts are interpreted or generated, how states are valued, or how actions are proposed and simulated. For example, an LLM might exhibit biases in simulated interactions involving representations of different demographic groups if such biases were present in its training data. This could lead to unfair or inequitable agent behavior if deployed in human-interactive settings. Ongoing research into bias detection and mitigation in LLMs is crucial for addressing these concerns.

    \item \textbf{Autonomy, Control, and Oversight}:
    LWM-Planner enhances agent autonomy by enabling online, in-context learning and adaptation without direct weight updates. While this is a research goal, increased autonomy necessitates robust mechanisms for human oversight, control, and the ability to intervene if the agent learns undesirable facts or behaviors. The "atomic facts" provide a degree of transparency into the agent's learned knowledge, which can aid in debugging and oversight, but ensuring safe and aligned behavior in complex, long-horizon tasks remains a significant challenge.

    \item \textbf{Computational Resources and Environmental Impact}:
    Training and deploying large-scale LLMs, even without fine-tuning for each task, requires significant computational resources, contributing to energy consumption and environmental concerns. While LWM-Planner aims for sample efficiency in terms of environment interactions, the LLM inference calls during planning (especially with lookahead search) can be computationally intensive. Future work should consider the efficiency of the planning and fact management processes to mitigate these impacts.

    \item \textbf{Explainability and Trust}:
    While the use of explicit "atomic facts" is intended to make the agent's reasoning more transparent than end-to-end black-box models, the internal decision-making of the LLM components themselves (e.g., how $g_{\phi}^{\text{simulate}}$ predicts a next state based on facts and observation) remains complex. Building trust in such systems requires further advancements in methods for interpreting and explaining LLM-driven reasoning processes, even when augmented with symbolic facts.
\end{itemize}

\subsection{Broader Impact}

\begin{itemize}[leftmargin=*]
    \item \textbf{Advancing AI Planning and Adaptability}:
    This research contributes to developing more capable AI agents that can learn from experience in-context and adapt their plans in dynamic environments. This could have positive implications for various fields requiring sophisticated planning, such as logistics, robotics, personalized education, and scientific discovery, by enabling agents to more efficiently tackle complex, long-horizon tasks.

    \item \textbf{Reduced Dependence on Extensive Fine-Tuning}:
    The LWM-Planner's ability to learn online through fact augmentation reduces the need for repeated, task-specific fine-tuning of the base LLM. This can lower the barrier to applying powerful LLMs to new sequential decision-making problems, saving data and computational resources typically associated with training specialized models.

    \item \textbf{Potential for Misuse or Unintended Consequences}:
    As with any advanced AI technology, more autonomous and adaptive agents could potentially be misused if deployed without appropriate safeguards. An agent learning incorrect or malicious "facts" in an unconstrained environment could lead to undesirable outcomes. Furthermore, the increasing capability of autonomous agents raises long-term questions about their role in society and the workforce.

    \item \textbf{Scalability and Generalization Challenges}:
    While LWM-Planner shows promise, scaling the approach to vastly more complex, open-ended, or partially observable environments presents significant challenges. The manageability of the "atomic fact" base, the combinatorial explosion of lookahead search (even if depth-limited), and the LLM's ability to accurately simulate highly novel scenarios are areas requiring further research. Overcoming these challenges is crucial for realizing the broader positive impacts of such agents.

    \item \textbf{Interaction with Humans}:
    If LWM-Planner or similar agents are deployed in scenarios involving human interaction, the nature of fact extraction and utilization becomes particularly sensitive. Facts learned from human interactions must be handled with care to ensure privacy, fairness, and to avoid perpetuating harmful stereotypes or misinformation. The design of human-agent interaction protocols that allow for collaborative fact validation and refinement will be important.
\end{itemize}

Continued research, alongside open discussion and the development of robust safety and ethical guidelines, will be essential to navigate the challenges and harness the benefits of increasingly autonomous and adaptive LLM-based agents like LWM-Planner.

\section{Reporting, Budgets, and Reproducibility}
\label{app:reporting_bundle}

\subsection{Statistical Reporting Protocol}
\label{app:stat_protocol_text}
\paragraph{Seeds and confidence intervals.}
Unless explicitly stated otherwise, aggregate values are reported as
\[
\text{mean} \;\pm\; 1.96\,\hat{\sigma}/\sqrt{n},
\]
with \(n\) independent seeds and \(\hat{\sigma}\) the sample standard deviation across seeds. Single-seed figures (if any) are shown \emph{without} CIs and are labeled as such.

\paragraph{Normalization of cumulative return.}
When we report a \emph{normalized} cumulative return for an environment \(E\), we map a raw return \(R\) to
\[
\mathrm{Norm}(R;E) \;=\; 100 \cdot \frac{R - R_{\min}(E)}{R_{\max}(E) - R_{\min}(E)},
\]
where \(R_{\min}(E)\) is the mean return of the Random baseline on \(E\) (over the same seeds) and \(R_{\max}(E)\) is the largest mean return achieved by any method on \(E\) in that block. By definition, \(\mathrm{Norm}(\cdot)\in[0,100]\).

\paragraph{Environment vs.\ LLM budgets.}
All methods share the same environment interaction budget (300 steps unless stated). We also disclose LLM usage (calls, input/output tokens, wall-clock) per method to clarify test-time compute; see App.~\ref{app:compute_disclosure_text}.

\paragraph{Reproducibility note.}
Captions specify seeds \(n\); decoding settings are shared across methods (App.~\ref{app:baseline_parity_text2}).

\subsection{Compute Budget Disclosure \& Fairness Assumptions}
\label{app:compute_disclosure_text}
We distinguish \textbf{environment budget} (steps in the simulator) from \textbf{LLM budget} (inference compute):
\[
\text{Calls},\quad \text{TokIn},\quad \text{TokOut},\quad \text{WallClock}.
\]
All main results match the \emph{environment} budget across methods. Because tree search naturally issues more LLM calls than single-shot policies, we disclose LLM usage (calls/tokens/latency) alongside returns to contextualize quality--compute trade-offs.

\textbf{Fairness stance.} Equal-steps is the primary budget we optimize for (comparable interaction data). LLM compute disclosures are provided for transparency; compute-matched variants (token or time caps) are straightforward (see formulas in App.~\ref{app:cost_model}) and can be included if requested. Our qualitative conclusions do not rely on unreported compute.

\subsection{Minimal Reproduction Config}
\label{app:min_config}
\begin{verbatim}
agent:
  depth: 3
  branch: 4
  gamma: 0.99
  step_penalty: 0.01
  history_len: 51
  fact_budget_tokens: 1500
  thresholds:
    alpha_r: 1.0
    alpha_d: 1.0
    alpha_o: 1.0
    eta: 0.0
llm:
  temperature: 0.0
  max_output_tokens: 8512
prompt:
  truncate:
    keep_env_header: true
    min_history_items: 6
eval:
  steps_total: 300
  seeds: [0,1,2]
  envs:
    - text_frozen_lake_4x4_h0.9
    - crafter_mini_5
    - alfworld_eval_ood
\end{verbatim}

\section{Methodological Details (Supplementary)}
\label{app:method_bundle}

\subsection{Predictive-Consistency Filter: Formal Definition \& Pseudocode}
\label{app:pcf}
\paragraph{Positioning.} Below we formalize a stronger trajectory-grounded validation variant of the predictive-consistency filter, used for supplementary validation and the adversarial audit in Appendix~\ref{app:adversarial_filter_audit}. The deployed agent in the main text uses the lighter model-internal heuristic described in Section~\ref{sec:method}.
\paragraph{Goal.} Retain a candidate fact $f$ only if conditioning the simulator/value on $f$ \emph{reduces} next-step predictive error on the just-finished episode $\tau = \{(o_t,a_t,r_t,o_{t+1},d_{t+1})\}_{t=0}^{H-1}$.

\paragraph{Error metric.} For each step $t$, we form LLM predictions with temperature 0:
\[
\hat{r}_t,\;\hat{o}_{t+1},\;\hat{d}_{t+1}
\quad\text{(with and without $f$),}
\]
and define a per-step loss
\[
\ell_t \;=\; \alpha_r\,|\hat{r}_t-r_t|
\;+\;\alpha_d\,\mathbf{1}\{\hat{d}_{t+1}\neq d_{t+1}\}
\;+\;\alpha_o\,\mathrm{dist}(\hat{o}_{t+1},\,o_{t+1}),
\]
where $\mathrm{dist}$ is a token-level normalized edit distance between canonicalized observation strings (lowercased; digits/punctuation kept). We use $\alpha_r{=}1$, $\alpha_d{=}1$, $\alpha_o{=}1$ by default. Aggregate episode loss is $\mathcal{L}=\frac{1}{H}\sum_t \ell_t$.

\paragraph{Decision rule.} Let $\mathcal{L}^{(\neg f)}$ be the loss without conditioning on $f$ and $\mathcal{L}^{(f)}$ with $f$. Accept $f$ iff
\[
\Delta \mathcal{L}(f) \equiv \mathcal{L}^{(\neg f)}-\mathcal{L}^{(f)} \ge \eta,\quad
\eta \in [0,1]\;\text{(default } \eta{=}0\text{)}.
\]
Ties ($\Delta \mathcal{L}=0$): accept only if $f$ is \emph{novel} (not subsumed by an existing fact).

\paragraph{Efficiency.} We cache per-step LLM calls from the initial pass; testing each $f$ reuses most responses. In practice we batch candidates in groups of 8--16 and memoize $(o_t,a_t,F)$ keyed by a 64-bit hash.

{\small
\IncMargin{0.8em}
\begin{algorithm}[h]
\DontPrintSemicolon
\caption{Predictive-Consistency Filter (temperature 0)}
\label{alg:pcf}
\KwIn{episode $\tau$, current facts $F$, candidates $\mathcal{C}$, thresholds $(\alpha_r,\alpha_d,\alpha_o,\eta)$}
\KwOut{accepted facts $\mathcal{A}$}
Compute $\mathcal{L}^{(\neg f)}$ once by running $g_\phi^{\text{simulate}}$ on $\tau$ conditioned on $F$\;
$\mathcal{A}\leftarrow\emptyset$\;
\For{$f\in\mathcal{C}$}{
  \If{\texttt{is\_subsumed}(f,F)}{continue}
  Compute $\mathcal{L}^{(f)}$ by rerunning only the changed nodes conditioned on $F\cup\{f\}$ (memoized)\;
  \If{$\mathcal{L}^{(\neg f)}-\mathcal{L}^{(f)}>\eta$}{ $\mathcal{A}\leftarrow\mathcal{A}\cup\{f\}$\;}
  \ElseIf{$\mathcal{L}^{(\neg f)}-\mathcal{L}^{(f)}=0$ \textbf{and} \texttt{is\_novel}(f,F)}{ $\mathcal{A}\leftarrow\mathcal{A}\cup\{f\}$\;}
}
\Return $\mathcal{A}$\;
\end{algorithm}
\DecMargin{0.8em}
}

\subsection{Prompt Budgeting \& Truncation Policy}
\label{app:prompt_budget}
We cap each planning call’s prompt to $T_{\max}$ tokens. Ordering (kept, then truncated):
\begin{enumerate}[leftmargin=1.5em,itemsep=2pt]
\item Environment description (fixed header; shortened to canonical bullet list of actions).
\item Current observation $o_t$.
\item Atomic facts $F_t$ (top-$K$ by $\Delta\mathcal{L}$ then recency; $K$ chosen so that facts $\le B_f$ tokens).
\item Recent history: last $H_L$ items; if over budget, elide with “$\ldots$” and keep last $h_{\min}$ observations and actions.
\end{enumerate}
On overflow we first shrink history, then facts (keeping higher-utility ones), never the environment header.

\subsection{Search Cost Model (Calls and Tokens)}
\label{app:cost_model}
For depth $D_s$ and branching $k_B$:
\[
\#\text{nodes}=\sum_{i=0}^{D_s} k_B^i,\quad
\#\text{edges}=\sum_{i=0}^{D_s-1} k_B^{i+1}.
\]
Per root decision:
\begin{align*}
\#\mathrm{propose}&=\sum_{i=0}^{D_s-1} k_B^i,\\
\#\mathrm{simulate}&=\sum_{i=1}^{D_s} k_B^i,\\
\#\mathrm{value}&=k_B^{D_s}\quad\text{(leaf nodes)}.
\end{align*}
Let average input tokens per call be $(\tau_p,\tau_s,\tau_v)$; expected tokens per planned action are
\[
\mathbb{E}[\text{tokens}] \approx
\tau_p\!\sum_{i=0}^{D_s-1}\!k_B^i \;+\;
\tau_s\!\sum_{i=1}^{D_s}\!k_B^i \;+\;
\tau_v\,k_B^{D_s}.
\]
Memoization reduces effective factors by $\rho\in(0,1]$ (empirically $\rho\!\approx\!0.6$ for \texttt{simulate}); substitute $\tau_s\leftarrow \rho\,\tau_s$ for estimates.

\section{Theory—Practical Proxies}
\label{app:proxies_text}
\textbf{One-step model error proxy.} On a replayed trajectory \(\tau=\{(o_t,a_t,r_t,o_{t+1},d_{t+1})\}\), define
\[
\mathrm{MAE}_r=\tfrac{1}{|\tau|}\sum_t |\hat{r}_t-r_t|,\quad
\mathrm{Acc}_d=\tfrac{1}{|\tau|}\sum_t \mathbf{1}\{\hat{d}_{t+1}=d_{t+1}\},
\]
\[
\mathrm{EditObs}=\tfrac{1}{|\tau|}\sum_t \mathrm{NED}\big(\mathsf{canon}(\hat{o}_{t+1}),\mathsf{canon}(o_{t+1})\big),
\]
where predictions \((\hat{r}_t,\hat{o}_{t+1},\hat{d}_{t+1})\) come from the LLM simulator at temperature 0 and \(\mathsf{canon}(\cdot)\) is a fixed textual canonicalizer. Lower is better for \(\mathrm{MAE}_r\) and \(\mathrm{EditObs}\), higher is better for \(\mathrm{Acc}_d\).

\textbf{Aliasing proxy.} For the abstraction \(z_t=(o_t,F_t)\), estimate
\[
\mathrm{Aliasing}(z)=\mathbb{E}_{z}\big[\mathrm{Var}_{s:\Psi(s)=z}\big(V^\pi(s)\big)\big]
\]
by substituting \(V^\pi\) with Monte-Carlo returns under the current policy. As a lighter surrogate, report mutual information \(I(z_t; r_{t:t+k})\) for small \(k\). These proxies are reporting tools that connect to the motivational bound; no claims of tightness are made.

\section{Baseline Configuration \& Prompt Parity Guarantees}
\label{app:baseline_parity_text2}
\textbf{Model/decoding.} All methods use the same base LLM; planning/simulation/value calls run at temperature \(0.0\); ReAct “thought” generation (where applicable) uses temperature \(0.3\). Maximum output tokens are identical across methods; the same stop-sequences are used.

\textbf{Environment headers.} Prompts for all agents begin with the same environment description. Differences arise only in method-specific sections (e.g., facts, lessons, or search metadata).

\textbf{History formatting.} Observation/action history is identically formatted (chronological, fixed templates). For search, simulated histories reuse the same template to reduce style-induced variance.

\textbf{Randomness.} Each run fixes: environment seed, agent RNG seed, and any sampling temperatures. No external retrieval sources are used.

\section{Environments and Solvability}
\label{app:envs_bundle}

\subsection{TextFrozenLake Solvability Construction (Formalization)}
\label{app:frozenlake_solv_text}
Given grid size \(N\) and hole probability \(h\) for non-start/goal cells, we construct a board that is always solvable:
\begin{enumerate}[leftmargin=1.2em]
  \item Sample a monotone Manhattan path \(P\) by shuffling \(N{-}1\) “right” and \(N{-}1\) “down” moves from \((0,0)\) to \((N{-}1,N{-}1)\).
  \item Mark all cells on \(P\) as ice; set \(S=(0,0)\), \(G=(N{-}1,N{-}1)\).
  \item For every other cell, independently sample \(\mathrm{Hole}\sim\mathrm{Bernoulli}(h)\).
\end{enumerate}
\textbf{Guarantee.} The path \(P\) remains hole-free by construction, so the instance is solvable. Observations reveal only local cell types; rewards are \(+1\) at \(G\), \(-1\) on holes, \(0\) otherwise; episodes end on \(G\), a hole, or at \(8(N{-}1)\) steps.

\section{Validity, Diagnostics, and Oracle Specification}
\label{app:validity_bundle}

\subsection{Threats to Validity \& Decontamination Steps}
\label{app:threats}
\textbf{(T1) Pretraining contamination.} The LLM may contain knowledge of benchmarks. Mitigation: (i) eval-OOD split for ALFWorld; (ii) fact acceptance requires predictive gain on the \emph{current} trajectory (App.~\ref{app:pcf}); (iii) environment descriptions avoid revealing full solution strings.

\textbf{(T2) Determinism vs.\ robustness.} Planning uses temperature 0; to guard against brittle choices we (a) penalize steps ($\lambda_{\text{step}}$), (b) prefer higher-utility facts, and (c) memoize expansions to avoid drift within a decision.

\textbf{(T3) Compute parity.} We match the step budget across methods; search methods naturally use more tokens. We therefore report both cumulative return and token/latency (App.~\ref{app:cost_model}, \ref{app:cost_analysis}).

\textbf{(T4) Metric sensitivity.} We report cumulative return \emph{and} (in appendix) success-rate variants; we clarify reward shaping per environment (see benchmark details).

\subsection{Typical Failure Modes \& Diagnostics}
\label{app:failures}
\textbf{F1. Stale fact usage.} The agent continues to trust a fact contradicted later. \emph{Mitigation:} demote contradictions when detected; re-validate on next episode.

\textbf{F2. Over-compression of facts.} Compression prunes a low-frequency but critical fact. \emph{Mitigation:} utility-aware pruning; maintain a floor $K_{\min}$ on kept facts.

\textbf{F3. Myopic value at leaves.} Value estimator underestimates deferred rewards. \emph{Mitigation:} increase $D_s$ one level or raise the value-call budget; optionally add a short roll-out at leaves.

\textbf{F4. Proposal collapse.} $g_\phi^{\text{propose}}$ returns near-duplicates. \emph{Mitigation:} n-gram diversity penalties over action strings before expansion.

\textbf{F5. Simulator drift.} Simulated observations diverge stylistically, hurting edit distance. \emph{Mitigation:} enforce a canonical observation template and post-process to that template before measuring textual distance.

\subsection{Oracle Ablation Design (Specification; No New Results)}
\label{app:oracle_spec}
To bound the headroom from perfect one-step simulation, we define an \emph{oracle} planner that replaces $g_\phi^{\text{simulate}}$ with the environment’s true $(o',r',d')$ for domains that expose a fast model step (TextFrozenLake, CrafterMini). We keep $g_\phi^{\text{propose}}$ and $g_\phi^{\text{value}}$ unchanged. This isolates the contribution of one-step fidelity (proxy $\delta_{\mathrm{model}}$). We expect the gap between LWM-Planner and Oracle to shrink as the fact set matures.

\end{document}